\documentclass[a4paper, 11pt,  numbers=noenddot, bibtotoc, abstracton]{scrartcl}

\usepackage{setspace}
\usepackage[ngerman,english]{babel}
\usepackage[T1]{fontenc}
\usepackage[utf8]{inputenc}

\usepackage[margin=1in]{geometry}
\usepackage{graphicx} 
\usepackage{floatflt} 
\usepackage{amsmath}  
\usepackage{amssymb}  
\usepackage{amsthm}   
\usepackage{paralist} 
\usepackage{color}	
\usepackage{scrlayer-scrpage}
\usepackage{bbm} 
\usepackage{tikz} 
\usepackage{tikz-cd} 
\usepackage{array}
\usepackage{pgfplots} 
\usepackage[section]{placeins} 
\usepackage{subcaption}
\usepackage{caption}
\usepackage[official]{eurosym}
\usepackage{multirow}
\usepackage{float}
\usepackage{verbatim}
\usepgfplotslibrary{units}
\usepackage{units}
\usepackage{mathtools}
\usetikzlibrary{arrows,shapes,positioning,matrix}
\usepackage{smartdiagram}
\usepackage{booktabs}
\usepackage[b]{esvect}
\usepackage{algorithm}
\usepackage{algorithmic}
\usepackage{authblk}
\usepackage[colorlinks=false, pdfborder={0 0 0}]{hyperref}
\usepackage{tikz-network}
\usetikzlibrary{positioning}
\usepackage[flushleft]{threeparttable}
\usepackage{pifont}
\usepackage{setspace}
\usepackage{bm}
\usepackage{mathrsfs}

\definecolor{lightOrange}{rgb}{0.996, 0.8,0.361}
\definecolor{lightRed}{rgb}{0.941, 0.2313,0.1254}
\definecolor{darkGray}{rgb}{0.498, 0.498,0.498}

\usepackage[babel,german=quotes]{csquotes}
\usepackage[style=authoryear, date=year, doi=false,isbn=false, dashed=false, firstinits=true,maxbibnames=99,uniquelist=false,maxcitenames=3,backend=bibtex]{biblatex}
\bibliography{Literature}
\renewbibmacro{in:}{}
\renewbibmacro*{volume+number+eid}{%
  \printfield{volume}
  \addnbthinspace
  \printfield{number}%
  \printfield{eid}}
\DeclareFieldFormat[article]{number}{\mkbibparens{#1}}

\setkomafont{sectioning}{\normalfont\normalcolor\bfseries}

\newtheoremstyle{1}
	{\topsep}
	{\topsep}
	{\itshape}
	{}
	{\bfseries}
	{.}
	{ }
	{\thmname{#1}\thmnumber{ #2}\thmnote{ (#3)}}

\newtheoremstyle{2}
	{\topsep}
	{\topsep}
	{\normalfont}
	{}
	{\bfseries}
	{.}
	{ }
	{\thmname{#1}\thmnumber{ #2}\thmnote{ (#3)}}

\theoremstyle{1}
\newtheorem{theorem}{Theorem}[section]
\newtheorem{lemma}[theorem]{Lemma}
\newtheorem{proposition}[theorem]{Proposition}

\newtheorem{corollary}[theorem]{Corollary}

\newtheorem{remark}[theorem]{Remark}

 \numberwithin{equation}{section}
 
\makeatletter
\newcases{mycases}{~}{%
  \hfil$\m@th\displaystyle{##}$}{$\m@th\displaystyle{##}$\hfil}{\lbrace}{.}
\makeatother

\makeatletter
\AtBeginDocument{%
  \expandafter\renewcommand\expandafter\subsection\expandafter{%
    \expandafter\@fb@secFB\subsection
  }%
}
\makeatother

\newcommand{\bbr}{\mathbb{R}}  
\newcommand{\bbn}{\mathbb{N}}

\newcommand{\bbz}{\mathbb{Z}}

\newcommand{\xcal}{\mathcal{X}}

\newcommand{\be}{\begin{equation}}
\newcommand{\ee}{\end{equation}}
\newcommand{\bew}{\begin{equation*}}
\newcommand{\eew}{\end{equation*}}

\newcommand{\ba}{\begin{array}{ll}}
\newcommand{\bal}{\begin{array}{ll}}
\newcommand{\ea}{\end{array}}

\newcommand{\cI}{{\mathscr{I}}}

\newcommand{\cO}{\mathcal{O}}

\makeatletter
\def\@fnsymbol#1{\ensuremath{\ifcase#1\or \mathrm{a}\or \mathrm{b}\or
   \mathrm{c}\or \mathparagraph\or \|\or **\or \dagger\dagger
   \or \ddagger\ddagger \else\@ctrerr\fi}}
    \makeatother

\allowdisplaybreaks

\begin{document}

\title{Stochastic Cell Transmission Models of Traffic Networks }
\author{Zachary Feinstein\thanks{Stevens Institute of Technology, School of Business, Hoboken, NJ 07030, USA, email: \href{mailto:zfeinste@stevens.edu}{zfeinste@stevens.edu}}\qquad Marcel Kleiber\thanks{Leibniz Universität Hannover, House of Insurance \& Institute of Actuarial and Financial Mathematics, Welfengarten 1,
30167 Hannover, Germany, email: \href{mailto:marcel.kleiber@insurance.uni-hannover.de}{marcel.kleiber@insurance.uni-hannover.de}} \qquad Stefan Weber\thanks{Leibniz Universität Hannover, House of Insurance \& Institute of Actuarial and Financial Mathematics, Welfengarten 1,
30167 Hannover, Germany, email: \href{mailto:stefan.weber@insurance.uni-hannover.de}{stefan.weber@insurance.uni-hannover.de}}}

\date{\today}
\maketitle

\begin{abstract}
We introduce a rigorous framework for stochastic cell transmission models for general traffic networks. The performance of traffic systems is evaluated based on preference functionals and acceptable designs. The numerical implementation combines simulation, Gaussian process regression, and a stochastic exploration procedure. The approach is illustrated in two case studies.
\end{abstract}


\section{Introduction}

Cell transmission models enable the quantification of the motion of traffic participants on a high level of aggregation. This provides computational advantages in comparison to microscopic traffic models that capture the motion of traffic participants in great detail. This gain in computational efficiency is sometimes disadvantageously associated with lower granularity, which complicates the representation of complex traffic modules and interactions of traffic participants. In this paper, we propose a rigorous framework for cell transmission models that incorporates three important features: a) The cells are identified with the nodes of a graph. We introduce a precise notation for the directions of the traffic participants within each cell. This allows the construction of cell transmission models for general traffic networks. b) Within each cell, road users traveling in one direction interact with road users traveling in other directions. Sending and receiving functions can capture these interactions of traffic flow and density with oncoming traffic flows and densities. c) Traffic volumes and conditions may vary randomly. Our general framework allows the inclusion of probabilistic phenomena. 

The proposed models enable the evaluation of traffic systems under a wide range of conditions. They can also be used for traffic planning by testing the effects of changes in design parameters. Comparisons can be made not only for deterministic systems, but also in the face of randomness and risk. We use preference functionals and their level sets for the normative classification and categorization of transportation systems. This approach is also closely related to the construction of measures of systemic risk. In concrete applications, we specify random benchmark flows and compute the collection of  parameters associated with traffic systems that are weakly preferred. We call these sets \emph{acceptable designs}. They accurately combine the descriptive, possibly random cell transmission model and the normative evaluation framework. 

Although less granular than microscopic traffic models, the extended flexibility compared to classical cell transmission models and the inclusion of randomness increase the computational complexity. In particular, stochastic simulation of the system under different conditions for multiple design parameters is costly. To address this problem, we employ a powerful machine learning technique, Gaussian Process Regression (GPR). GPR allows an interpolation of system performance between simulated points while providing at the same time measures of uncertainty. Our innovation is the adaptive selection of points to refine the estimation of the acceptable designs. For this purpose, we use the GPR estimation of the boundary of this level set and GPR variance estimates at candidate points from the previous iteration. We also provide error bounds.

The capabilities of our algorithms in the context of generalized cell transmission models are illustrated in two case studies. We study two traffic networks, one with two signalized intersections and another one with variable capacities of highways and speed limits. Acceptable designs are identified and interpreted. From an algorithmic point of view, we compare the squared exponential kernel to  Mat{\'e}rn kernels. 

Our main contributions are the following:
\begin{enumerate}
\item We provide a rigorous framework for cell transmission models in general traffic networks. Traffic participants traveling in different directions interact with each other locally. Traffic volumes and conditions can vary stochastically.
\item To classify and categorize traffic systems, we propose the notion of acceptable design inspired by preference functionals and systemic risk measures.
\item The numerical estimation of acceptable designs combines Monte Carlo simulation, Gaussian process regression, and a stochastic exploration procedure in the parameter space. The performance of this algorithm is demonstrated through case studies.
\end{enumerate}

\subsection{Structure of the Paper}

The paper is organized as follows: Section~\ref{sec:literature} reviews the related literature on cell transmission models, systemic risk measures, and Gaussian process regression. Section~\ref{sec:ctm} presents our general framework for cell transmission models of traffic networks. Section~\ref{sec:acd} describes the objective of the machine learning estimation problem:  sets of acceptable designs. Our algorithm is discussed in Section~\ref{sec:learning}. It is applied in numerical case studies in Section~\ref{sec:case-studies}. Questions for further research are presented in Section~\ref{sec:concl}. The appendix contains proofs and auxiliary material.

\subsection{Literature}\label{sec:literature}

We develop a general and rigorous formulation of cell transmission models in a network environment. Considering general transmission and receiving functions and general directions of travel, we can consider many related contributions in the literature as special cases. Inspired by the theory of systemic risk measurements, we introduce as a diagnostic instrument the notion of acceptable designs of a traffic system. This constitutes a normative instrument of traffic planning, which enables the evaluation and control of traffic systems. To identify the acceptable designs of traffic systems, we develop an active learning method based on Gaussian process regression, a powerful machine learning technique. In the following, we review the relevant literature and compare it with our innovations.

\paragraph{Cell Transmission Models.} The classical cell transmission model (CTM) was developed in the seminal work of \textcite{Daganzo1994} and \textcite{Daganzo1995}. It is a deterministic macroscopic traffic model that captures the evolution of traffic flows and densities in discrete time. \textcite{Daganzo1994} explains how this model can be viewed as a discrete approximation to the LWR model (\textcite{Lighthill1955} and \textcite{Richards1956}); accordingly, he originally introduces CTM to study homogeneous traffic flows on highways. 

Since then, CTM has been revisited countless times in the literature. Important research questions range from estimating traffic densities (\textcite{Munoz2003}) to establishing variable speed limits on highways (\textcite{Hadiuzzaman2013})). The introduction of randomness increases the informativeness and allows a representation of more complex phenomena. \textcite{Sumalee2011} develop a version with stochastic demand and supply constraints. \textcite{Jin2019} study highway dynamics under random capacity-reducing incidents modeled by an exogenous Markov chain. 

However, the popularity of CTM is also due to the fact that it can be used to represent urban traffic. For example, \textcite{Long2008} examines the formation and dissipation  of congestion in urban networks. Other papers discuss traffic lights and their optimization  (e.g., \textcite{Pohlmann2010}, \textcite{Xie2013}, \textcite{Srivastava2015}). We refer to \textcite{Adacher2018} for a detailed review on CTM, especially for urban traffic.

Another strand of literature generalizes CTM for different traffic types. Different traffic users with their different driving characteristics can share the available space. \textcite{Tuerprasert2010} and \textcite{Tiaprasert2017} partition a cell, \textcite{Levin2016} add partial densities. Buses can be introduced as moving bottlenecks that reduce capacity (\textcite{Liu2015}, \textcite{Tang2022}).

In this paper, we explain how cells can represent different types of roads, including highways, roundabouts, signalized intersections. Our setup allows for randomness as a general modeling paradigm. We present a precise formulation of the direction of travel that allows  detailed modeling of the interaction of competing traffic flows. Some related conceptional issues are also discussed by  \textcite{Tampere2011}. We note that conflicting flows are common in pedestrian flow modeling (see, e.g., \textcite{Floetteroed2015}); \textcite{Moustaid2021} can be viewed as a special case of our model. We also briefly indicate how our model can be generalized for multiple interacting traffic types.

\paragraph{Systemic Risk Measures.} The axiomatic theory on the quantification of risk  dates back to the seminal paper of \textcite{Artzner1999}.  Past contributions have focused primarily on the quantification of financial risk: The central construction is based on a notion of acceptability, i.e., a set of (financial) positions with an acceptable risk. A monetary risk measure quantifies the risk of a financial position in monetary units: It is the minimum amount of cash that must be added to a position to make it acceptable. We refer to \textcite{Foellmer2015} for an overview.

This construction can be generalized to quantify the risk of a system of interacting entities. Systemic risk measures as a precise mathematical notion are introduced by \textcite{Feinstein2017} and \textcite{Biagini2018}. The theory has proven useful not only for quantifying risk in financial networks (e.g., \textcite{Weber2017}); \textcite{Cassidy2016} applies it to measuring the risk of power outages in transmission networks. \textcite{Salomon2020} use it to control the resilience of technical systems.

Borrowing from these measures of systemic risk, we introduce the concept of acceptability to assess the efficiency of traffic systems. The results, typically measures of efficiency such as traffic flow, are normatively categorized into acceptable and unacceptable outcomes. The set of acceptable designs of a traffic system then refers to those design parameters (e.g., noise parameters, traffic light configurations, initial densities) that lead to acceptable outcomes. As a conceptual difference, we recognize that systemic risk measures are introduced for financial systems whose risk decreases with the amount of available capital. There is no such a priori monotonic dependence of traffic flows on underlying system parameters.

\paragraph{Gaussian Process Regression.} Mathematically, the acceptable designs of a traffic system form a set of real vectors defined in terms of a level set of a function. To estimate this set, we estimate the underlying function. To address the computational cost in the context of stochastic simulation, we develop an active learning approach.

We apply Gaussian Process Regression (GPR), also called Kriging, as a Bayesian inference method to estimate a metamodel from isolated noisy data. The method assumes a Gaussian process as a prior distribution over functions and is updated with observed data to produce an estimate. The popularity of this method is due to its probabilistic foundation, which also allows an evaluation of the uncertainty of the estimate. We refer to \textcite{Rasmussen2005} as a standard reference and \textcite{Kanagawa2018} and \textcite{Swiler2020} for insightful surveys.

\textcite{Ankenman2010} examines GPR as a tool for metamodeling in the context of simulation. Similarly, \textcite{binois2018practical} discusses practical aspects. The focus is on approximating the entire underlying function rather than just a particular level set. Most closely related to our active learning framework are \textcite{gotovos2013active} and \textcite{lyu2021evaluating}, which also develop iterative procedures for approximating level sets with GPR. In contrast to previous works, we construct a random search algorithm to determine where the function values will be estimated next. This circumvents the need to rely on complicated optimization methods. In addition, we introduce a general sandwich principle to impose upper bounds on the approximation error of set estimation algorithms. Error bounds on the estimated function, as extensively discussed in \textcite{Srinivas2012} and \textcite{Lederer2019}, result in bounds on the approximation error of level sets.

\section{Cell Transmission Models for Traffic Networks}\label{sec:ctm}

\begin{figure}[h]
\centering
\includegraphics[scale=0.7]{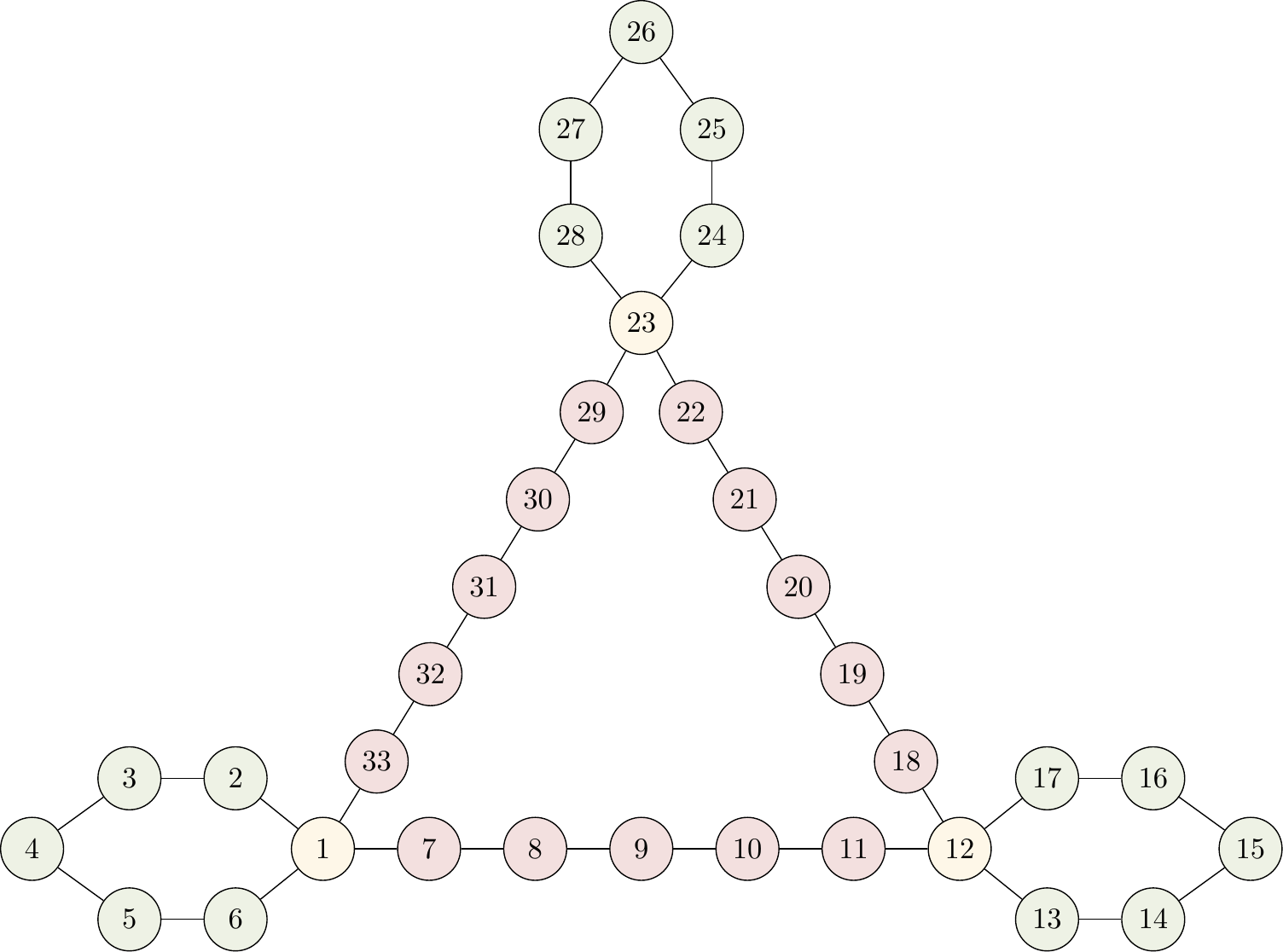}
\caption{Highway Network.}
\label{fig:network2}
\end{figure}
\subsection{A Motivating Example}

Cell transmission models capture the dynamic evolution of traffic densities and flows in traffic networks. Figure~\ref{fig:network2} shows a stylized network with 28 cells or nodes. The nodes can be of different types, for example, the red and green nodes in Figure~\ref{fig:network2} could be sections of highways or roads, and the yellow nodes could be intersections or roundabouts. The type of each node determines how traffic participants interact within the node. The total traffic volume in the system may vary (possibly randomly) due to traffic participants entering or leaving the system. For example, the green areas in Figure~\ref{fig:network2} could be sources and sinks of the traffic network.

The model determines how much traffic is transmitted from one node to the next. In generalized cell transmission models on graphs, traffic in the nodes may move in various directions. This is especially true for nodes that are connected to multiple other nodes, such as nodes 1, 12, and 23. It is thus important to indicate different directions or routes by notation when modeling traffic dynamics. A rigorous framework is described in the next sections. 

\subsection{General Framework}

To model general traffic networks, we consider a set of vertices or nodes $V$ denoting all existing traffic spaces; these are connected by a collection of edges $E\subseteq V\times V$.  The corresponding graph $G=(V,E)$ is the traffic network under consideration. Vehicles, bicycles, and pedestrians move through the graph, and traffic flows interact with each other at the nodes. Traffic flows and corresponding traffic densities in a node $v$ are distinguished by the preceding location $u\in \cI(v)$ and the subsequent destination $w \in \cO(v)$; here, the sets $\cI (v) = \{v': (v',v) \in E\}$ and $\cO (v) = \{v': (v,v') \in E\}$ collect the nodes which can reach $v$ and which can be reached from $v$, respectively.

\paragraph{Traffic Dynamics.} Traffic is modeled in discrete time, enumerated as $t=0,1,2, \dots$. For simplicity, we consider only one type of traffic participants, although the formalism can easily be extended to multiple types. For each node $v$ and each route $(u, v, w)$ through $v$ with $u\in \cI (v)$, $w\in \cO(v)$, $\rho_{(u, v ,w)}(t)$ is the traffic density of agents at time $t$ traveling through $v$ along $(u, v, w)$. The flow into $v$ of traffic participants traveling along $(u, v, w)$ through $v$ during the time interval $(t, t+1)$ is denoted by $q^\mathrm{in}_{(u,v,w)}(t+1)$; the corresponding flow out of $v$ is $q^\mathrm{out}_{(u,v, w)}(t+1)$. We further assume that sources and sinks exist in the system, and denote by $q^\mathrm{net}_{(u, v, w)}(t+1)$ the (possibly random) net flow of traffic participants entering or leaving the system in $v$ on the route $(u,v,w)$ during the time interval $(t, t+1)$. Nodes represent traffic spaces (such as roads, intersections, roundabouts, or shared spaces) and can be of different sizes, labeled $l_v$ for $v\in V$. Densities are updated iteratively for each $v\in V$, $u\in \cI(v)$ and $w\in \cO(v)$ at each point in time $t+1$:
\begin{equation}\label{eq:dynamics}
\rho_{(u,v,w)}(t+1)=\rho_{(u,v,w)}(t)+\frac{1}{l_v}\left(q^\mathrm{in}_{(u,v,w)}(t+1)-q^\mathrm{out}_{(u,v,w)}(t+1)+q^\mathrm{net}_{(u,v,w)}(t+1)\right)
\end{equation}

To model a specific traffic network, its initial conditions and the dynamic behavior of sources and sinks must be specified exogenously. Traffic flows are constrained on routes $(u,v,w)$ by the characteristics of the corresponding traffic space, the density of agents on that route, and other traffic participants traveling through $v$; this is modeled by general supply and demand constraints. In addition, we assume that turning fractions are conserved, cf.~\textcite{Tampere2011}; this captures the principle of first-in-first-out for incoming traffic. The solution of the traffic model according to \eqref{eq:dynamics} is defined as the solution of a global optimization problem that maximizes traffic flow.  This solution captures a perfect cooperation among traffic participants to achieve the objective and provides a benchmark solution; inspired by \textcite{Tampere2011}, we introduce interaction rules in Section~\ref{sec:SCIR} that realistically capture non-cooperative behavior. Adequate traffic dynamics is the solution of an optimization problem under these additional constraints.\footnote{The cooperative benchmark does not guarantee the greatest flow over a longer time horizon because myopic optimization focuses on a single time period. When focusing on a longer time horizon, constrained optimization solutions may be superior in some cases, although they are suboptimal for each time window given the same state at the beginning of that period.} 

\paragraph{Sending and Receiving Functions.} We begin by discussing general supply and demand constraints. Traffic densities constrain both inflow and outflow; this includes both the density of traffic on a given route $(u,v,w)$ and the counter-densities of routes that pass through traffic space $v$. The demand constraint is formalized by the sending function $S$, which captures existing traffic participants that would leave a traffic area in the next step if they could continue without any constraints in the subsequent modules (\emph{``demand for space by traffic participants'', ``flow that may be sent''}). The supply constraint is formalized by the receiving function $R$ and captures the maximum amount of traffic that can be absorbed from preceding traffic modules (\emph{``supply of space for traffic participants'', ``flow that may be received''}). The sending function $S$ and the receiving function $R$ are of the following form:\footnote{Another function with the same domain and range is the fundamental diagram that characterizes the stationary or long-run traffic flow on a given route $(u,v,w)$, if both demand from preceding traffic modules and supply of subsequent modules are unrestricted. It can be computed as the minimum of sending function $S_{(u,v,w)}$ and receiving function $R_{(u,v,w)}$.}\small
\begin{eqnarray*} 
S_{(u,v,w)} : &\quad 
\left\{
\begin{array}{ccc}
\bbr_+^{ \cI (v)  \times \cO (v)  } & \longrightarrow & \bbr_+\\
(\rho_{(u', v, w')})_{u'\in \cI (v), w'\in \cO(v) } & \mapsto& q_{(u, v, w)}
\end{array}
\right.   \\
R_{(u,v,w)} : &\quad 
\left\{
\begin{array}{ccc}
\bbr_+^{ \cI (v)  \times \cO (v)  } & \longrightarrow & \bbr_+\\
(\rho_{(u', v, w')})_{u'\in \cI (v), w'\in \cO(v)} & \mapsto& q_{(u, v, w)}
\end{array}
\right. 
\end{eqnarray*}\normalsize
They bound inflow and outflow, i.e.,\small
\begin{eqnarray*}
q^\mathrm{in}_{(u,v,w)}(t+1) & \leq &   R_{(u,v,w)} \left((\rho_{(u', v, w')} (t))_{u'\in \cI (v), w'\in \cO(v) } \right), \\ 
q^\mathrm{out}_{(u, v, w)}(t+1) & \leq &   S_{(u,v,w)}  \left((\rho_{(u', v, w')} (t))_{u'\in \cI (v), w'\in \cO(v)} \right).
\end{eqnarray*}\normalsize
For the sending function that bounds the outflow, we will always require that 
\begin{equation}\label{eq:s-pos}
S_{(u,v,w)} (\rho_{(u', v, w')} (t))_{u'\in \cI (v), w'\in \cO(v) }  \leq  \rho_{(u, v, w)}(t),
\end{equation}
 i.e., the outflow during the time interval $(t, t+1)$ cannot be greater than the occupation of $(u,v,w)$ 
with traffic participants.

Sending and receiving functions do not have to be constant in time, but can vary periodically, randomly or depending on circumstances. This can be modeled by a dependency on additional state variables besides the dependency on the traffic densities of the agents on the paths through a node.

\paragraph{Cooperative Driving Benchmark Model.} To be able to specify the cooperative driving benchmark model, we denote the fraction of traffic participants on $(u,v,w)$ turning to $y\in \cO(w)$ by $f_{(u,v,w) \to y}(t+1) \geq 0$ with $\sum_{y\in \cO(w)} f_{(u,v,w) \to y}(t+1) =1$. This leads to the following identity:
\begin{equation}\label{eq:in_out}
q^\mathrm{in}_{(u,v,w)}(t+1) =  \sum_{x\in \cI(u)} f_{(x,u, v) \to w}(t+1) \cdot q_{(x,u, v)}^\mathrm{out}(t+1)
\end{equation}

The flows $q^\mathrm{out}_{(u, v, w)}(t+1)$ are the solutions of the following myopic global optimization problem:
\begin{align}\label{eq:globopt}
\mathrm{argmax}\quad &\sum_{v\in V} \sum_{u\in \cI (v), \; w\in \cO(v)} q^\mathrm{out}_{(u, v, w)}(t+1)
\end{align}
s.t. for all  $v\in V$, $u\in \cI (v)$, $w\in \cO(v)$:\small
\begin{itemize}
\item $q^\mathrm{out}_{(u, v, w)}(t+1) \geq 0$,
\item 
$q^\mathrm{in}_{(u,v,w)}(t+1) =  \sum_{x\in \cI(u)} f_{(x,u, v) \to w}(t+1) \cdot q_{(x,u, v)}^\mathrm{out}(t+1)$,
\item $
q^\mathrm{in}_{(u,v,w)}(t+1) \leq R_{(u,v,w)} \left((\rho_{(u', v, w')} (t))_{u'\in \cI (v), w'\in \cO(v) } \right)
$,
\item $q^\mathrm{out}_{(u, v, w)}(t+1) \leq S_{(u,v,w)}  \left((\rho_{(u', v, w')} (t))_{u'\in \cI (v), w'\in \cO(v) } \right)$. 
\end{itemize}\normalsize

\begin{remark}\label{rem:optdecoupled}
The global problem can be split into decoupled local problems.  Since \small
$$
\sum_{v\in V} \sum_{u\in \cI (v), \; w\in \cO(v)} q^\mathrm{out}_{(u, v, w)}(t+1)=\sum_{v\in V}\sum_{u\in\cI(v)}\sum_{x\in\cI(u)}q^\mathrm{out}_{(x, u, v)}(t+1),
$$\normalsize
we can instead solve the decoupled problems
\begin{equation}\label{eq:optdecoupled}
\mathrm{argmax}\quad  \sum_{x\in\cI(u)} q_{(x,u, v)}^\mathrm{out}(t+1)
\end{equation} for all $v\in V$, $u\in \cI (v)$ under the corresponding constraints.\footnote{The constraints require that for all $x\in\cI(u)$, $w\in\cO(v)$:
$q^\mathrm{out}_{(x, u, v)}(t+1) \geq 0$,
$q^\mathrm{in}_{(u,v,w)}(t+1) =  \sum_{x\in \cI(u)} f_{(x,u, v) \to w}(t+1) \cdot q_{(x,u, v)}^\mathrm{out}(t+1)$, $q^\mathrm{in}_{(u,v,w)}(t+1) \leq R_{(u,v,w)} \left((\rho_{(u', v, w')} (t))_{u'\in \cI (v), w'\in \cO(v)} \right)
$, $q^\mathrm{out}_{(x, u, v)}(t+1) \leq S_{(x, u, v)}  \left((\rho_{(x', u, v')} (t))_{x'\in \cI (u), v'\in \cO(u)} \right)$.}
\end{remark}

\subsection{Examples of Traffic Nodes}\label{sec:ex_tn}

Our general framework allows us to capture various traffic modules, including the cell transmission model of \textcite{Daganzo1994}, roads, multidirectional pedestrian areas as in \textcite{Moustaid2021}, unsignalized and signalized intersections, roundabouts, and many other types of traffic spaces. The model can also be extended to multiple types of traffic participants  by introducing additional constraints that link these types together.

To describe a few motivating examples, consider a node labeled $\#$ and assume for simplicity that $\cI (\#) = \cO(\#)$ are adjacent nodes in a plane. Its elements are enumerated counterclockwise by $0, 1, \dots, n - 1$ with $n := \mbox{card} \left( \cI (\#) \right)$. A convenient approach will be to identify $\cI (\#)$ with the additive group $\bbz_n$, i.e., to equip $0, 1, \dots, n - 1$ with the operation $+$ modulo $n$.

\paragraph{Highways.}

Lanes of highways are separated, thus interaction between different directions is not present. Setting $\cI (\#) = \bbz_2$, a simple linear model for sending and receiving functions is, with $u \in \bbz_2$,\small
\begin{align*}
S_{(u,\#,u+1)}\left( \left(\rho_{(u',\#,w')} \right)_{u'\in\cI(\#),w'\in\cO(\#)}\right)&=\min\left\{s^\mathrm{max}_\#,a\rho_{(u,\#,u+1)} \right\},\\
R_{(u,\#,u+1)}\left( \left(\rho_{(u',\#,w')} \right)_{u'\in\cI(\#),w'\in\cO(\#)}\right)&=\max\left(b\left(\frac{\rho^\mathrm{max}_\#}{2}-c\rho_{(u,\#,u+1)} \right),0\right),
\end{align*}\normalsize
where $\rho^\mathrm{max}_\#>0$ is the maximum density, $s^\mathrm{max}_\#>0$ is the maximum flow, $0<a\leq 1$ is the free-flow speed, $0<b\leq 1$ is the congestion wave speed, and $c>0$ is an interaction parameter.

\paragraph{Bidirectional Linear Interfaces.}

A generalization of the unidirectional situation is adequate for pedestrians with bidirectional interacting flows. We introduce an additional interaction parameter $d>0$ that reflects the impact of traffic travelling in opposite direction and obtain a simple linear model for sending and receiving functions; their specification is, with $u \in \bbz_2$,\small
\begin{align*}
S_{(u,\#,u+1)}\left( \left(\rho_{(u',\#,w')} \right)_{u'\in\cI(\#),w'\in\cO(\#)}\right)&=\min\left\{s^\mathrm{max}_\#,a\rho_{(u,\#,u+1)} \right\},\\
R_{(u,\#,u+1)}\left( \left(\rho_{(u',\#,w')} \right)_{u'\in\cI(\#),w'\in\cO(\#)}\right)&=\max\left(b\bigg(\rho^\mathrm{max}_\#-c\rho_{(u,\#,u+1)} -d\rho_{(u+1,\#,u)} \bigg),0\right).
\end{align*}\normalsize
The models may, of course, include nonlinear relationships if the data or expert knowledge suggest other functional forms. An example can be found in \textcite{Moustaid2021}, where the sending function is nonlinear if both the density and the counter-density are subcritical. Their model can be easily transferred to our notation.

\paragraph{Pedestrian Square.}
Bidirectional linear interfaces can be generalized canonically to multiple directions. A fully symmetric geometry with $n$ entries/exits and can be captured by a node $\#$ with $\cI (\#) = \bbz_n$. Assuming pedestrians do not return to the same entry, constant turning rates $f_{(x,u,\#)\to w}(t)\equiv 1/(n-1)$ for all $u\in \mathbb{Z}_n$, $w\in\mathbb{Z}_n\setminus \{u\}$ are a simple choice. For $u\neq w$, a simple model is \small
\begin{align*}
S_{(u,\#,w)}\left( \left(\rho_{(u',\#,w')} \right)_{u'\in\cI(\#),w'\in\cO(\#)}\right)&=\min\left(s^\mathrm{max}_\#,a\rho_{(u,\#,w)}\right),\\
R_{(u,\#,w)}\left( \left(\rho_{(u',\#,w')} \right)_{u'\in\cI(\#),w'\in\cO(\#)}\right)&=\max\Bigg(b\bigg(\rho^\mathrm{max}_\#-c\rho_{(u,\#,w)}-d\sum_{\substack{u'\in\cI(\#)\setminus\{u\},\\ w'\in\cO(\#)\setminus\{w\}}} \rho_{(u',\#,w')} \bigg),0\Bigg).
\end{align*}\normalsize

\paragraph{Roundabouts.}
In roundabouts the interaction of the various participants is also determined by the overlap of their paths. This leads to somewhat more complicated sending and receiving functions, but the main ideas are similar to those outlined above. We provide a description of the details in Appendix~\ref{sec:a-ex}, including a bidirectional traffic circle for pedestrians and an extension of the framework to traffic models with multiple populations.

\paragraph{Intersections.}
Just as with roundabouts, complex interactions of traffic participants can also be modeled at intersections. As examples, we consider two cases: a highly simplified model of an unsignalized junction and a rather complex model of an intersection with a traffic light.

A possible \emph{simplified model} of an intersection adjusts the pedestrian space. Since cars move slower on crowded intersections, we include an exponential damping factor between $0$ and $1$ in the definition of the sending function that depends on a parameter $\zeta >0$. For  $u\in\cI(v)$ and $u\neq w\in \cO(v)$ we set\small
\begin{align*}
S_{(u,\#,w)}\left( \left(\rho_{(u',\#,w')} \right)_{u'\in\cI(\#),w'\in\cO(\#)}\right)&=\min\left\{s^\mathrm{max}_\#,a \rho_{(u,\#,w)} \exp\left(- \zeta \cdot \sum_{u'\in\cI(\#),w'\in\cO(\#)} \rho_{(u',\#,w') }\right)\right\},\\
R_{(u,\#,w)}\left( \left(\rho_{(u',\#,w')} \right)_{u'\in\cI(\#),w'\in\cO(\#)}\right)&=\max\left\{b\left(\rho^\mathrm{max}_\#-c\sum_{u'\in\cI(\#),w'\in\cO(\#)} \rho_{(u',\#,w')} \right),0\right\}.
\end{align*}\normalsize

At the other end of the spectrum are models incorporating more details that allow cell transmission models (albeit still much simpler than microscopic models) to reproduce complex interaction patterns. As an example, consider a \emph{signalized intersection} with $\cI(\#)=\cO(\#) = \bbz_4$. Consider $u \in\cI(\#)$ and right-hand traffic. The path $(u,v, u+1)$ corresponds to a right turn, while the paths $(u,v, u+2)$ and $(u,v, u+3)$ represent going straight and a left turn, respectively. As additional state variables, we consider for each path the current signal (state $0$ -- red,  or state $1$ -- green) and the time it has been in this state. To capture the traffic lights, we assume that both the sending functions and receiving functions depend on these states. More specifically, $LA_{(u,\#,u+i)}$, $i=1,2,3$, adjusts the free-flows speed of traffic. In signal state $LS_{(u,\#,w)} = 0$, it is set to $0$, but in state $LS_{(u,\#,w)} = 1$ the adjustment $LA_{(u,\#,u+i)}$ depends on the time the signal has been in this state. The traffic light does not immediately show green in state $0$ but with a delay, so that $LA_{(u,\#,u+i)}$ is initially $0$ and increases with time due to the acceleration of traffic until the maximal free-flow speed is reached. When turning right and driving straight ahead, the adjusted sending functions are \small
\begin{align*}
S_{(u,\#,u+1)}\left( \left(\rho_{(u',\#,w')} \right)_{u'\in\cI(\#),w'\in\cO(\#)},LA_{(u,\#,u+1)}\right)&=\min\left\{s^\mathrm{max}_\#,LA _{(u,\#,u+1)} a \rho_{(u,\#,u+1)} \right\},\\
S_{(u,\#,u+2)}\left( \left(\rho_{(u',\#,w')} \right)_{u'\in\cI(\#),w'\in\cO(\#)},LA_{(u,\#,u+2)}\right)&=\min\left\{s^\mathrm{max}_\#, LA_{(u,\#,u+2)} a \rho_{(u,\#,u+2)} \right\}.
\end{align*} \normalsize
When turning left, the sending function may be decreased due to oncoming traffic. This can be modeled by an exponential term for a parameter $\zeta >0$, for example:\small
\begin{align*}
&S_{(u,\#,u+3)}\left( \left(\rho_{(u',\#,w')} \right)_{u'\in\cI(\#),w'\in\cO(v)},LA_{(u,\#,u+3)}\right) \; = \\ &
  \quad\quad\quad\quad\quad\quad\quad \min\Bigg\{s^\mathrm{max}_\#, LA_{(u,\#,u+3)} a \rho_{(u,\#,u+3)}\cdot
\exp\bigg( - \zeta \cdot \left(\rho_{(u+2,\#,u)} +\rho_{(u+2,\#,u+3)} \right) \bigg) \Bigg\}
\end{align*}\normalsize
The traffic node $\#$ of the signalized intersection also includes the areas in front of the traffic lights. Their capacity is limited by $\rho_{u,\#}^\mathrm{max}$, and this is reflected by the receiving functions: \small
\begin{align*}
R_{(u,v,w)}\left( \left(\rho_{(u',v,w')}\right)_{u'\in\cI(v),w'\in\cO(v)}\right)&=\max\left\{ b \left(\rho_{u,\#}^\mathrm{max}-\sum_{w'\in\cO(\#)} \rho_{(u,\#,w')} \right),0\right\}.
\end{align*}\normalsize

\subsection{Interaction Rules}\label{sec:SCIR}

The global optimization problem \eqref{eq:globopt} captures perfect cooperation to achieve maximum (albeit myopic) traffic flow in the traffic system. This is unrealistic for models of real traffic, since individual participants optimize only their own utility, which may conflict with the goals of others. As explained in \textcite{Tampere2011}, further constraints -- \emph{interaction rules} -- can mimic the local behavior of agents. We explicitly specify three formal approaches.

\paragraph{Demand Proportional Flows.}
Sending functions model the demand of road users for movement. An interaction rule could specify that realized flows are proportional to demand. To be more specific, we focus on the decoupled problems \eqref{eq:optdecoupled} and assume that there exists a constant $\lambda_{(u,v)}(t+1)\in[0,1]$, independent of $x$, such that
\begin{equation}\label{eq:dpf}
q^\mathrm{out}_{(x, u, v)}(t+1) =\lambda_{(u,v)}(t+1) S_{(x, u, v), k}  \left((\rho_{(x', u, v')} (t))_{x'\in \cI (u), v'\in \cO(u)}\right) 
\end{equation}
Maximizing flow in \eqref{eq:optdecoupled} is now considerably simplified and equivalent to solving for all $v\in V$ and $u\in\cI(v)$ the problems $\mathrm{argmax}_{\lambda_{(u,v)}\in[0,1]}\quad  \lambda_{(u,v)}$ under the constraints given in Remark \ref{rem:optdecoupled}. These problems possess the explicit solutions \small
\begin{equation*}
\lambda_{(u,v)} (t+1)=\min\left\{1,\min_{w\in\cO(v)}\left\{\frac{R_{(u,v,w)} \left((\rho_{(u', v, w')} (t))_{u'\in \cI (v), w'\in \cO(v)}\right)}{\sum_{x\in \cI(u)} f_{(x,u, v) \to w}(t+1) \cdot S_{(x, u, v)}  \left((\rho_{(x', u, v')} (t))_{x'\in \cI (u), v'\in \cO(u)}\right)}\right\}\right\},
\end{equation*} \normalsize
yielding the flows\small
\begin{align*}
q^\mathrm{out}_{(x, u, v)}&(t+1) \quad =\\ & \min\Bigg\{S_{(x, u, v)}  \left((\rho_{(x', u, v')} (t))_{x'\in \cI (u), v'\in \cO(u)}\right),\min_{w\in\cO(v)}\left\{\frac{R_{(u,v,w)} \left((\rho_{(u', v, w')} (t))_{u'\in \cI (v), w'\in \cO(v)}\right)}{\sum_{x\in \cI(u)} f_{(x,u, v) \to w}(t+1) }\right\}\Bigg\}.
\end{align*}\normalsize

\paragraph{Capacity Proportional Flows.} Instead of assuming that realized flows are proportional to demand for the same factor, different directions could have different capacities that determine the proportionality factors. Letting $\sum_{x\in \cI(u)} d_{(x,u,v)}=1$ with $d_{(x,u,v)}\geq 0$, $x\in \cI(u)$, one may assume that there exists a constant $\lambda_{(u,v)}(t+1)\in[0,1]$, independent of $x$, such that
\begin{equation}\label{eq:dcpf}
q^\mathrm{out}_{(x, u, v)}(t+1) = \min\left(\lambda_{(u,v)} d_{(x,u,v)},1\right) S_{(x, u, v)}  \left((\rho_{(x', u, v')} (t))_{x'\in \cI (u), v'\in \cO(u)} \right)
\end{equation}
The factor $\min\left(\lambda_{(u,v)} d_{(x,u,v)},1\right)$ ensures that realized flows cannot become larger than the sending flows. As in the case of demand proportional flows, this leads again to several one-dimensional optimization problems that can be easily solved explicitly:\small
\begin{multline*}
\lambda_{(u,v)}(t+1) =  \min_{w\in\cO(v)}\Bigg\{  \inf\bigg\{\lambda\geq 0\colon  
\sum_{x\in\cI(u)} f_{(x,u,v)}(t+1) \min\left(\lambda d_{(x,u,v)},1\right)\cdot \\  S_{(x, u, v)}  \left((\rho_{(x', u, v')} (t))_{x'\in \cI (u), v'\in \cO(u) }\right)  = R_{(u,v,w)} \left((\rho_{(u', v, w')} (t))_{u'\in \cI (v), w'\in \cO(v) } \right)\bigg\}\Bigg\}.
\end{multline*}\normalsize
The interior minimization is simply the solution of an equation in the single variable $\lambda$. The right hand side is continuous and increasing in $\lambda$. Due to its piecewise linearity, this problem can be solved by a finite number of iterations.

\paragraph{Priority Rules.} 
Priority rules vary from country to country. A common example is that on an intersection traffic from the right often has priority. This can be implemented as an interaction rule in our model. Again, for $v\in V$ and $u\in \cI(v)$, as in problem \eqref{eq:optdecoupled}, we consider the decoupled problems of local optimization of the flow over the next time period. To capture priority rules, we assume that based on the current traffic state, a fixed enumeration of $\cI(u)=\{x_{u,1},\dots,x_{u,I_u}\}$ is chosen. This order of incoming nodes is fixed for a certain period of time during which local traffic flows are computed hierarchically. Traffic originating from the nodes listed before the others has priority. The duration of the regime must be chosen according to the real situation being modeled and the real time to which each time period in the model corresponds. Formally, we sequentially solve for $i=1,\dots,I_u$ the problems $\mathrm{argmax} \; q_{(x_{u,i},u, v)}^\mathrm{out}(t+1)$ for the corresponding constraints\footnote{The constraints are $q^\mathrm{out}_{(x_{u,i}, u, v)}(t+1) \geq 0$, $q^\mathrm{in}_{(u,v,w)}(t+1) =  \sum_{j=1}^i f_{(x_{u,j},u, v) \to w}(t+1) \cdot q_{(x_{u,j},u, v)}^\mathrm{out}(t+1)$, $q^\mathrm{in}_{(u,v,w)}(t+1) \leq R_{(u,v,w)} \left((\rho_{(u', v, w')} (t))_{u'\in \cI (v), w'\in \cO(v) }\right)$  for all $w \in \cO(v)$ and  $q^\mathrm{out}_{(x_{u,i}, u, v)}(t+1) \leq S_{(x_{u,i}, u, v)}  \left((\rho_{(x', u, v')} (t))_{x'\in \cI (u), v'\in \cO(u)} \right)$.} and obtain the solution \small
\begin{align*}
q^\mathrm{out}_{(x_{u,1}, u, v)}(t+1)=&\min\bigg(S_{(x_{u,1}, u, v)}  \left((\rho_{(x', u, v')} (t))_{x'\in \cI (u), v'\in \cO(u)} \right),\frac{R_{(u,v,w)} \left((\rho_{(u', v, w')} (t))_{u'\in \cI (v), w'\in \cO(v)} \right)}{f_{(x_{u,1},u, v) \to w}(t+1)}\bigg),\\
q^\mathrm{out}_{(x_{u,i+1}, u, v)}(t+1)=&\min\bigg(S_{(x_{u,i+1}, u, v)}  \left((\rho_{(x', u, v')} (t))_{x'\in \cI (u), v'\in \cO(u) } \right),\\
& \frac{R_{(u,v,w)} \left((\rho_{(u', v, w')} (t))_{u'\in \cI (v), w'\in \cO(v)} \right) - \sum_{j=1}^i f_{(x_{u,j},u, v) \to w}(t+1) \cdot q_{(x_{u,j},u, v)}^\mathrm{out}(t+1)}{f_{(x_{u,i+1},u, v) \to w}(t+1)} \bigg), \\
& \quad \quad\quad\quad\quad \mbox{for all} \; i=1,\dots,I_u .
\end{align*} \normalsize

\section{Acceptable Configurations and Designs}\label{sec:acd}

\subsection{The Question}

We will be interested in evaluating the performance of traffic systems under various conditions. Specifically, the influence of measures to regulate traffic should be understood. To this end, we consider a collection of cell transmission models enumerated by a vector $ k \in \mathbb{D} \subseteq \mathbb{R}^r$ for some dimension $r \in \mathbb{N}$. The set $\mathbb{D}$ is assumed to be bounded. The components of $k$ specify the characteristics of a traffic system, such as total traffic volume, traffic control parameters, the magnitude of random variations of various variables, and weather conditions. For fixed $k\in \mathbb{D} $, the (possibly random) time evolution of the corresponding cell transmission model is described following the approach from Section~\ref{sec:ctm}. We will call $k$ a \emph{design parameter}. For each design parameter $k\in \mathbb{D}$, the traffic system can be simulated and the corresponding random variables of interest can be calculated. We assume that we wish to evaluate a random variable $Q_k$ and compare the results across $k \in \mathbb{D}$. The random variable $Q_k$ could model the total network traffic flow or traffic flow per traffic volume over a given time horizon, for example.

\subsection{Preference Functionals and Acceptable Designs}\label{sec:pfad}

By $\xcal$ we denote a normed space of random variables such as $L^p$, $p\in [1,\infty]$, and assume that $Q_k\in \xcal$, $k \in \mathbb{D}$. We evaluate the performance of the traffic system by a preference functional $U:  \xcal \to \bbr$. Typically, $U$ is increasing on $\xcal$, i.e., if $Q\leq Q'$ almost surely, then $U(Q) \leq U(Q')$. A special case is expected utility with $U(Q) = \mathbb{E}(u(Q))$ for an increasing function $u: \mathbb{R} \to \mathbb{R}$. In the case studies in Section~\ref{sec:case-studies}, we will study
\begin{enumerate}
\item \textbf{Expectation:} $u(x)=x$,
\item \textbf{Polynomial Utility:} $u(x) = -|x-c_p|^\alpha\mathbbm{1}\{x\leq c_p\},\quad c_p\in\mathbb{R},~\alpha\geq 1$,
\item \textbf{Expectile Utility:} $u(x) = \alpha(x-c_e)_+ - (1-\alpha) (x-c_e)_- ,\quad c_e\in\mathbb{R},~\alpha\leq 1/2$ where $x_+=\max(x,0)$ and $x_-=\max(-x,0)$,
\item \textbf{Square Root Utility:} $u(x) = \sqrt{x}$.
\end{enumerate}

For the practical evaluation of all traffic systems enumerated by design parameters $k\in \mathbb{D} $, categorization by performance is a convenient methodology. This is related to the level sets of the utility functionals. For a utility functional $U$ and a fixed level $\gamma \in \bbr$, the set of \emph{acceptable designs} is 
\begin{equation*}
\mathcal{D}=\mathcal{D}_{U,\gamma}=\left\{ k\in\mathbb{D}\colon U(Q_k) \geq \gamma\right\}.
\end{equation*}
This characterizes the traffic systems with utility of at least $\gamma$. In applications, the level $\gamma$ is often chosen as the utility of a benchmark distribution, i.e., $\gamma:= U(Q)$ for a random variable $Q$ with the benchmark distribution.

The acceptable designs are closely related to systemic risk measures as introduced in \textcite{Feinstein2017}. In the special case that $U$  and $k\mapsto Q_k$ are increasing and $k$ parametrizes the design parameters in terms of incremental monetary costs, the systemic risk measure
$$R ((Q_m)_{m\in \mathbb{R}^r};k)  = \left\{ m\in\mathbb{D}\colon U(Q_{k+m}) \geq \gamma\right\} =  \left\{ m\in\mathbb{D}\colon k+m   \in  \mathcal{R} \right\} = \mathcal{R} - k $$ is the collection of vectors of additional investments required for the various features of the traffic system to achieve the acceptable design. 

\section{Learning the Acceptable Design}\label{sec:learning}

We are interested in characterizing acceptable designs $\mathcal{D}=\{k\in\mathbb{D}\colon \mathbb{E}(u(Q_k))\geq \gamma\}$ of traffic systems. The challenge is that $Q_k$ can only be simulated for finitely many $k\in \mathbb{D}$ and needs to be interpolated in between these points. The selection of points for the simulation is also an important issue.
In this section, we propose a machine learning algorithm for the accelerated estimation of this set. 

We approach the problem by approximating the function\footnote{To control the approximation error, we assume that $\mathbb{D}$ is bounded. This implies that $\mathcal{D}\subseteq \mathbb{D}$ is bounded.}
\begin{equation*}
\mu\colon\mathbb{D}\to\mathbb{R},\quad k\mapsto\mathbb{E}(u(Q_k)),
\end{equation*} 
based on simulated data $(k,\hat{\mu}_k)_{k\in\mathbb{D}}$.\footnote{We use the index notation for the data $\hat{\mu}_k$ at isolated points $k$ and retain the notation $\mu(\cdot)$ when referring to a function defined on $\mathbb{D}$.} 
The simulated data are generated by Monte Carlo simulation of $Q_k$ for selected points $k\in \mathbb{D}$. We use an iterative learning algorithm that successively selects finite sets $\mathbb{D}^i$ of points, $i=1,2, \dots$.  These sets form an increasing sequence $\mathbb{D}^0\subseteq\mathbb{D}^1\subseteq\mathbb{D}^2 \subseteq\dots$, and new points $\mathbb{D}^i\setminus\mathbb{D}^{i-1}$ are strategically selected. The corresponding values $\mu(k)$ at $k\in \mathbb{D}^i\setminus\mathbb{D}^{i-1}$ are estimated with increasing accuracy. To extend the $(k,\hat{\mu}_k)_{k\in\mathbb{D}^i}$ to the entire design space $\mathbb{D}$, we use Gaussian process regression (GPR), a Bayesian inference method. The corresponding estimator of $\mu$ is denoted by $m^i\colon\mathbb{D}\to\mathbb{R}$. A key feature of GPR is that it not only produces an estimator of $\mu$, but also captures the corresponding uncertainty.

The acceptable designs $\mathcal{D}$ are estimated by the plug-in estimators $\hat{\mathcal{D}}^i = \{k \in \mathbb{D} \colon m^i(k) \geq \gamma\}$, $i\in \bbn$. Gaussian process regression has been used previously for estimating level sets (\textcite{lyu2021evaluating}). The main innovation of this section is to develop a framework for active learning using sequential statistics in conjunction with GPR. For this purpose, we use a heteroscedastic version of Gaussian process regression.

\subsection{Monte Carlo Estimation of Function Values} 

We first describe the Monte Carlo estimation of the simulated data, i.e., the estimation of the function value $\mu(k)=\mathbb{E}(u(Q_k))$ given a fixed design parameter $k\in\mathbb{D}$ in a fixed iteration $i\in\mathbb{N}$. Denote by $(\tau^i)^2>0$ a selected \emph{target variance} of estimates $(k,\hat{\mu}_k)_{k\in\mathbb{D}^i}$. Increasing precision is obtained by letting $\tau^i$ be decreasing for $i \geq 1$; however, $\tau^0$ is associated to the initialization of the algorithms and therefore typically chosen smaller than $\tau^1$. We combine a heuristic from the central limit theorem with sequential statistics to determine a stopping criterion. 

If $\hat{Q}_k^1,\hat{Q}_k^2,\dots$ is a sequence of i.i.d. samples of $Q_k$, the central limit theorem implies that, for large $n\in\mathbb{N}$, the distribution of the sample mean $
\hat{\mu}_k^n = \frac{1}{n}\sum_{j = 1}^n u(\hat{Q}_k^j)$
is roughly 
$\mathcal{N}\left(\mu(k),\frac{(\hat{\sigma}^n_k)^2}{n}\right)$
where the correct variance is replaced by the sample variance $(\hat{\sigma}^n_k)^2$ of our estimates of expected utility. 

We stop sampling when the sample noise $(\hat{\sigma}^n_k)^2/n$ falls below the desired target noise $(\tau^i)^2$.  Additionally, we impose a minimal number of simulations $n_\mathrm{min}$ and a maximal number of simulations $n_\mathrm{max} \gg n_\mathrm{min}$ so that a known upper bound on computational resources can be guaranteed -- but this does not have to be required. To be precise, we end our simulations of $Q_k$ when
\begin{equation*}
n = \min\left\{\min\left\{n_\mathrm{min} \leq \bar{n}\colon \frac{(\hat{\sigma}_k^{\bar{n}})^2}{\bar{n}} \leq (\tau^i)^2\right\} , n_\mathrm{max}\right\}.
\end{equation*}
Finally, we set $\hat{\mu}_k:=\hat{\mu}_k^n$ and (with a slight abuse of notation) $\tau_k^2:=\frac{(\hat{\sigma}_k^{{n}})^2}{{n}}$.

\subsection{Gaussian Process Regression}\label{sec:learning-estimate-gpr}

We give a brief overview of Gaussian process regression\footnote{We refer to \textcite{Rasmussen2005} as a standard reference for more details}, adapted to our framework. A \emph{Gaussian process} defines a probability law over functions from $\mathbb{D}$ to $\mathbb{R}$ such that the finite-dimensional marginal distributions are normal. The distribution is fully characterized by its mean function $m\colon\mathbb{D}\to\mathbb{R}$ and covariance function (also called kernel) $c\colon \mathbb{D}\times\mathbb{D}\to[0,\infty)$. The corresponding law or process is denoted by $\mathcal{GP}(m,c)$. For any finite set\footnote{Note that we retain the index $i$ even though it is not technically required for this section.} $\mathbb{D}^i\subseteq\mathbb{D}$, we know that $M(\mathbb{D}^i)\sim\mathcal{N}\left(m(\mathbb{D}^i),\Sigma(\mathbb{D}^i,\mathbb{D}^i)\right)$ with $M(\mathbb{D}^i)=(M(l))_{l\in\mathbb{D}^i}$, $m(\mathbb{D}^i)=(m(l))_{l\in\mathbb{D}^i}$, and $\Sigma(\mathbb{D}^i,\mathbb{D}^i)=(c(l,l'))_{l,l'\in \mathbb{D}^i}$. The kernel $c$ is also associated to  properties of the process such as the existence of regular versions. 

\emph{Gaussian process regression} is a probabilistic procedure for inference on an unknown function from possibly noisy data on the values of the function at some points. GPR is a Bayesian approach, i.e., the function is assumed to be an unknown sample from a prior distribution. The prior is taken to be the law of a Gaussian process. The posterior distribution is then calculated based on the data.

There are numerous approaches for choosing the prior. In this paper, we will use the standard choice $m\equiv 0$ as the prior mean and consider two commonly used covariance functions\footnote{\textcite{Rasmussen2005} offer an extensive discussion on covariance functions.}, namely the squared exponential function and the Matérn kernel. Both are defined in terms of hyperparameters $\sigma_c^2>0$ (referred to as the \emph{signal variance}) and $l>0$ (\emph{characteristic length scale}). The squared exponential function is given by
\begin{equation*}
c_\mathrm{SE}(k,k')=\sigma_c^2\cdot \exp\left(-\frac{\|k-k'\|^2}{2l^2}\right),\quad k,k'\in\mathbb{D},
\end{equation*}
and, for $\nu>0$, the Matérn kernel is defined as
\begin{equation*}
c_\mathrm{M}(k,k';\nu)=\sigma_c^2\cdot \frac{2^{1-\nu}}{\Gamma(\nu)}\left( \frac{\sqrt{2\nu}\|k-k'\|}{l}\right)^\nu K_\nu\left(\frac{\sqrt{2\nu}\|k-k'\|}{l}  \right),\quad k,k'\in\mathbb{D},
\end{equation*}
with $K_\nu$ being a modified Bessel function. The squared exponential function has mean square derivatives, yielding smooth behavior of the sample paths. The smoothness properties of the Matérn kernel are controlled by $\nu$: Small choices lead to rougher sample paths, while, in the limit $\nu\to\infty$, the Matérn kernel equals the squared exponential function and can, thus, be understood as a generalization. For more details, we refer to \textcite{Rasmussen2005}. In our case studies, we will consider the common choice $\nu=3/2$.\footnote{For $\nu=3/2$, the Matérn kernel simplifies into $c_\mathrm{M}(k,k';3/2)=\left(1+\frac{\sqrt{3}\|k-k'\|}{l}\right)\exp\left(-\frac{\sqrt{3}\|k-k'\|}{l}\right)$.} The Matérn kernel offers advantages when higher curvature is required at some points to reduce the amplitude of oscillations in the approximation.

To estimate the unknown function from data, the following Bayesian approach will be applied: The unknown function $\mu$ is interpreted as a realization of a Gaussian process $M\sim\mathcal{GP}(m,c)$ with the given prior distribution. For any finite set $\mathbb{D}_*\subseteq \mathbb{D}\setminus \mathbb{D}^i$, $M$ evaluated in $(\mathbb{D}^i,\mathbb{D}_*)$ follows a joint normal distribution. This still holds true, if we introduce independent noise $\varepsilon_k\sim\mathcal{N}(\mu(k),\tau_k^2)$ with $\tau_k\geq 0$ and set $\hat{M}_k=M(k)+\varepsilon_k$, $k\in \mathbb{D}^i$. The joint distribution with noise is
\begin{equation}\label{eq:jointNormal}
\begin{bmatrix}
\hat{M}(\mathbb{D}^i)\\
M(\mathbb{D}_*)
\end{bmatrix}\sim\mathcal{N}\left(\begin{bmatrix}
m(\mathbb{D}^i)\\
m(\mathbb{D}_*)
\end{bmatrix},\begin{bmatrix}
\Sigma(\mathbb{D}^i,\mathbb{D}^i)+\operatorname{diag}\left(\tau_1^2,\dots,\tau_{|\mathbb{D}^i|}^2\right) & \Sigma(\mathbb{D}^i,\mathbb{D}_*)\\
\Sigma(\mathbb{D}_*,\mathbb{D}^i) & \Sigma(\mathbb{D}_*,\mathbb{D}_*)
\end{bmatrix}\right)
\end{equation}
with $\hat{M}(\mathbb{D}^i)=(\hat{M}_l)_{l\in\mathbb{D}^i}=(M(l)+\varepsilon_l)_{l\in\mathbb{D}^i}$, $M(\mathbb{D}_*)=(M(l))_{l\in\mathbb{D}_*}$ and $\Sigma(\mathbb{D}^i,\mathbb{D}^i)=(c(l,l'))_{l,l'\in \mathbb{D}^i}$, $\Sigma(\mathbb{D}^i,\mathbb{D}_*)=(c(l,l'))_{l\in\mathbb{D}^i,l'\in \mathbb{D}_*}$, $\Sigma(\mathbb{D}_*,\mathbb{D}^i)=\Sigma(\mathbb{D}^i,\mathbb{D}_*)^\top$, $\Sigma(\mathbb{D}_*,\mathbb{D}_*)=(c(l,l'))_{l,l'\in \mathbb{D}_*}$.

We assume that \emph{noisy observations} are obtained on $\mathbb{D}^i$. Then a Bayesian update for the Gaussian process is computed. A Bayesian estimator of $\mu$ is given by the posterior mean function, and its uncertainty can be captured by the posterior variance. The \emph{posterior distribution}, i.e., the distribution of $(M\mid \hat{M}(\mathbb{D}^i)=\hat{\mu}(\mathbb{D}^i))$, is characterized by the following theorem. 

\begin{theorem}[Gaussian Process Regression]\label{thm:GPR}
Suppose that the Gaussian process $M\sim\mathcal{GP}(m,c)$ has a mean function $m\colon\mathbb{D}\to\mathbb{R}$ and a covariance function $c\colon \mathbb{D}\times \mathbb{D}\to [0,\infty)$. We assume that $\hat{M}_k=M(k)+\varepsilon_k$,  $k\in\mathbb{D}^i$, where the random variables $\varepsilon_k\sim\mathcal{N}(0,\tau^2_k)$, $k\in\mathbb{D}^i$, are jointly independent and also independent of $M$.  For given noisy observations $\hat{\mu}(\mathbb{D}^i)$ of $\hat{M}$ on $\mathbb{D}^i$, the conditional law of the process is $(M\mid \hat{M}(\mathbb{D}^i)=\hat{\mu}(\mathbb{D}^i))\sim \mathcal{GP}(m^i,c^i)$ where
\begin{align*}
m^i(k)&=m(k)+ \Sigma(\mathbb{D}^i,k)^\top \left( \Sigma(\mathbb{D}^i,\mathbb{D}^i) + \operatorname{diag}\left(\tau_1^2,\dots,\tau_{|\mathbb{D}^i|}^2\right)\right)^{-1}(\hat{\mu}(\mathbb{D}^i)-m(\mathbb{D}^i)),\quad k\in\mathbb{D},\\
c^i(k,k')&=c(k,k')- \Sigma(\mathbb{D}^i,k)^\top \left( \Sigma(\mathbb{D}^i,\mathbb{D}^i) + \operatorname{diag}\left(\tau_1^2,\dots,\tau_{|\mathbb{D}^i|}^2\right)\right)^{-1}\Sigma(\mathbb{D}^i,k') ,\quad k,k'\in\mathbb{D}.
\end{align*}
with $\hat{\mu}(\mathbb{D}^i)=(\hat{\mu}_l)_{l\in\mathbb{D}^i}$, $m(\mathbb{D}^i)=(m(l))_{l\in\mathbb{D}^i}$, $\Sigma(\mathbb{D}^i,k)=(c(l,k))_{l\in\mathbb{D}^i}$, and $\Sigma(\mathbb{D}^i,\mathbb{D}^i)=(c(l,l'))_{l,l'\in \mathbb{D}^i}$. The standard deviations are given by $\sigma^i(k) = \sqrt{c^i (k,k)} $.
\end{theorem}
\begin{proof}
See, for example, \textcite{Goldberg1997}.
\end{proof}

The previous theorem requires that the hyperparameters are determined in advance. The prior distribution of the Gaussian distribution and the sampling distribution  of the noisy observation specify the distribution of the observed data given the hyperparameters. To select the hyperparameters $\theta$, we maximize the \emph{marginal likelihood}\footnote{Maximizing the marginal likelihood is also referred to as empirical Bayes, evidence approximation, or type-II maximum likelihood (cf. \textcite{Rasmussen2005} or \textcite{Schulz2018}). This is a non-convex problem and typical methods may converge to local maxima. However, numerical experiments indicate that prediction based on the squared exponential function is robust with respect to the estimation of its hyperparameters (see \textcite{Chen2018}). Alternatively, one could place a hyperprior on $\theta$ at the expense of tractability. We refer to \textcite{Lalchand2019} for more details.} of the observed data $\hat{\mu}(\mathbb{D}^i)$ using
\begin{equation*}
\hat{M}(\mathbb{D}^i)\mid \theta\sim\mathcal{N}\left(m(\mathbb{D}^i),\Sigma(\mathbb{D}^i,\mathbb{D}^i)+\operatorname{diag}\left(\tau_1^2,\dots,\tau_{|\mathbb{D}^i|}^2\right)\right),
\end{equation*}
cf. equation \eqref{eq:jointNormal}. 

\subsection{Active Learning Framework}\label{sec:leaning-alg}

We propose an algorithm that combines GPR and the successive acquisition of new sets of points $\mathbb{D}^0\subseteq\mathbb{D}^1\subseteq\mathbb{D}^2 \subseteq\dots$. The goal is to approximate the set of acceptable designs $\mathcal{D}$. The general structure is described in Algorithm~\ref{alg:SEP}, which we will discuss in more detail in the following subsections. We would like to emphasize that each part of this algorithm is related to active learning approaches with stochastic kriging that have been used before (see, for example. \textcite{binois2018practical,lyu2021evaluating,gotovos2013active}); however, as far as we are aware, our particular combination of stochastic search and GPR is novel. Our estimation procedure consists of two major steps:
\begin{itemize}
\item \textbf{Phase 1: Initialize.} First, we create an initial set of points $\mathbb{D}^0$ and estimate values $\hat{\mu}_k$ with a target noise $(\tau^0)^2$ for each $k\in\mathbb{D}^0$. The hyperparameters of the GPR are estimated by maximizing the marginal likelihood.\footnote{The detailed procedure is described in Algorithm~\ref{alg:PEstInitial}.} We then compute our first GPR\footnote{The details are provided in Algorithm~ \ref{alg:PEstGPR}.} estimate $m^0\colon\mathbb{D}\to\mathbb{R}$ and $\sigma^0\colon\mathbb{D}\to[0,\infty)$ according to Theorem \ref{thm:GPR}.

\item \textbf{Phase 2: Loop.} Second, we repeat an active learning procedure to improve our GPR estimate.  We use an acquisition function to randomly construct $\mathbb{D}^i\setminus \mathbb{D}^{i-1}$. Function values at the new arguments are estimated\footnote{The GPR from each iteration yields a prior $\mathcal{N}(m^{i-1}(k),(\sigma^{i-1}(k))^2)$ which could be exploited in the subsequent iteration for sampling. This idea of Bayesian sampling is briefly described in Appendix \ref{sec:BayesianSampling}.} and used to improve the estimates of the posterior mean and variance of GPR (see again Algorithm \ref{alg:PEstGPR}). 

The algorithm either terminates after a maximum number of iterations, or an upper bound for the approximation error can be specified to control the quality of the approximation.\footnote{Such an upper bound is discussed in more detail in Section \ref{sec:learning-estimate-error}.}
\end{itemize}

The set estimate $\hat{\mathcal{D}}^i$ is a function of the point estimate of the mean $m^i$. Additionally, the measure of uncertainty $\sigma^i$ will be used to construct $\mathbb{D}^i\setminus\mathbb{D}^{i-1}$, i.e., selecting points to sample. It also provides information on the approximation errors associated with this procedure.

\begin{algorithm}[!htbp]
\caption{Active Learning Framework.}
\label{alg:SEP}
\begin{algorithmic}
\STATE{\textbf{Phase 1: Initialize}
    \begin{enumerate}
    \item Sample an initial data set $\mathbb{D}^0$ of $n_\mathrm{initial}$ points uniformly in $\mathbb{D}$
    \item Estimate values at points $k\in\mathbb{D}^0$ such that the variance is bounded by the target variance $(\tau^0)^2$
    \item Estimate the hyperparameters for GPR (which will be fixed hereafter)
    \item Compute the posterior mean and variances according to Theorem \ref{thm:GPR}
    \end{enumerate}}
\STATE{\textbf{Phase 2: Loop}
    \begin{enumerate}
    \item Build the acquisition function from the last posterior estimates. The acquisition encodes tradeoffs between the distance from the estimated boundary and the posterior uncertainty, when selecting additional points
    \item Sample $n_\mathrm{loop}$ new points according to the acquisition function via rejection sampling. This set is $\mathbb{D}^i \setminus \mathbb{D}^{i-1}$ 
    \item Simulate values at points $k \in \mathbb{D}^i\setminus \mathbb{D}^{i-1}$ such that the variance is bounded by the target variance $(\tau^i)^2$
    \item Use all simulated values (i.e., the simulations for $\mathbb{D}^i$) to compute the new posterior mean and variance by Bayesian updating of the original prior according to Theorem \ref{thm:GPR}
    \item Determine whether to stop 
    \end{enumerate}}
\end{algorithmic}
\end{algorithm}

\paragraph{Selecting Points to Sample.} The goal of selecting new points $\mathbb{D}^i\setminus\mathbb{D}^{i-1}$ after iteration $i-1$ is to improve the estimate of the superlevel set $\mathcal{D} = \{k \in \mathbb{D}\colon \mathbb{E}(u(Q_k)) \geq \gamma\}$. To examine $\mathbb{D}$, we first randomly select new points. Sampled points $k\in \mathbb{D}$ will only be used for further improvements in step $i$ in the GPR, if for some constant $c_1>1$ the inequality $c_1\cdot\tau^i  <  \sigma^{i-1} (k)$ holds, i.e., if Monte Carlo sampling can substantially decrease the uncertainty at that point. We use $c_1=5$ in our implementations of the algorithm.

Moreover, points that are close to the previously estimated boundary $m^{i-1}(k) = \gamma$ are preferred. This can be encoded by an \emph{acquisition function} $\mathcal{I}^i\colon \mathbb{D} \to [0,\infty)$ that seeks to capture the informative potential of estimating $\mathbb{E}(u(Q_k))$ at a new design parameter $k\in\mathbb{D}$. We use $\mathcal{I}^{i}(k) \colon= \Phi\left(- c^{i}_2\cdot |m^{i-1}(k) - \gamma|\right)$, $i=1,2,\dots$, where $\Phi$ is the standard normal CDF and $c^{i}_2>0$ is an increasing sequence.\footnote{Alternatively, one could use $\mathcal{I}^{i}(k) \colon= \Phi\left(- c_2\cdot \frac{|m^{i-1}(k) - \gamma|}{\sigma^{i-1}(k)}\right)$, $i=1,2,\dots$, with constant $c_2>0$. It is large when $k$ is close to the estimated boundary and when the posterior uncertainty is large. This acquisition function is comparable to that used in \textcite{lyu2021evaluating}; in that work, the acquisition function is optimized using a genetic algorithm to determine a single next point. Here, one would directly implement a stochastic sampling routine to generate $n_\text{loop} \geq 1$ new points to construct $\mathbb{D}^{i}\setminus\mathbb{D}^{i-1}$.}
Up to its normalizing constant, we treat $\mathcal{I}^{i}$ as a density in order to simulate new points in $\mathbb{D}$. 

In order to sample from this acquisition function, we employ rejection sampling (see, e.g., \textcite{Glasserman2003} for more details on the method) to sample from $\mathcal{I}^{i}$. We note that $\mathcal{I}^{i}(k)\leq 1/2$ for all $k\in\mathbb{D}$ by construction. Therefore by sampling points uniformly in $\mathbb{D}$ (i.e., with density $1/\mathrm{vol}(\mathbb{D})$), we can accept the point $k$ with probability $2\mathcal{I}^{i}(k)$ to recover samples from our acquisition function $\mathcal{I}^{i}$.\footnote{The upper bound on the likelihood ratio for the rejection sampling can be given by $\sup_{k \in \mathbb{D}} \mathcal{I}^{i}(k) / (1/\mathrm{vol}(\mathbb{D})) \leq \mathrm{vol}(\mathbb{D})/2$.} This algorithm is modified by first checking the inequality $c_1\cdot\tau^i  <  \sigma^{i-1} (k)$. The procedure is provided by Algorithm~\ref{alg:PEst}.\footnote{The implementation of Algorithm~\ref{alg:PEst} includes a maximal number of trials for the while loop. If no point is found that satisfies $c_1\cdot\tau^i  <  \sigma^{i-1} (k)$, the algorithm continues with the next iteration $i+1$.
Moreover, after the estimation of a function value, we check if its sample noise is small enough so that including the new data point is beneficial. For $c_3>0$ (we use $c_3=2$ in our simulations), we discard a new data point $(k,\hat{\mu}_k)$ if $\hat{\sigma}_k / \sqrt{n} \geq c_3\cdot \tau^i$.}

\begin{algorithm}[!htp]
\caption{Rejection Sampling.}
\label{alg:PEst}
\begin{algorithmic}
\FOR{$j=1,2,\dots,n_\mathrm{loop}$}
\STATE{Set $\mathrm{flag}=\mathrm{true}$.}
\WHILE{$\mathrm{flag}=\mathrm{true}$}
\STATE{Sample $\hat{U}_j\sim\mathrm{Unif}(\mathbb{D})$.}
\IF{$c_1\cdot\tau^i<\sigma^{i-1}(k)$}
\STATE{Sample $p\sim\mathrm{Unif}(0,1)$.}
\IF{$p<2\mathcal{I}^{i}(\hat{U}_j)$}
\STATE{Set $\mathrm{flag}=\mathrm{false}$.}
\ENDIF
\ENDIF
\ENDWHILE
\ENDFOR
\end{algorithmic}
\end{algorithm}

\subsection{Sandwich Principle and Bounds on the Approximation Error}\label{sec:learning-estimate-error}

In this section, we evaluate the approximation error of the estimate $\hat{\mathcal{D}}^i$ of the set $\mathcal{D}$. We construct inner and outer approximations $\hat{\mathcal{D}}^i_-$ and $\hat{\mathcal{D}}^i_+$ of $\mathcal{D}$ that \emph{sandwich} the true set. We find these inner and outer approximations by constructing lower and upper approximations $m^i_-,m^i_+\colon\mathbb{D}\to\mathbb{R}$ of $\mu\colon\mathbb{D}\to\mathbb{R}$, in the context of Gaussian process regression. We study approximation errors under the assumptions that are described in Theorem~\ref{thm:GPR}. The metric we use is given by the commonly studied Nikodym metric\footnote{There are other distance metrics between sets. We refer, for example, to \textcite{Cuevas2009} or \textcite{Brunel2018} for an overview. Another common metric is the Hausdorff metric, which has a more visual character. Here, the distance between two sets $A_1,A_2\subseteq \mathbb{D}$ is defined by $d_H\left(A_1,A_2\right)=\inf\{\varepsilon>0\mid A_1\subseteq \bigcup_{k\in A_2} B(k,\varepsilon)~,A_2\subseteq \bigcup_{k\in A_1}B(k,\varepsilon)\}$ where $B(k,\varepsilon)$ is the open ball of radius $\varepsilon$ centered in $k$.}
\begin{equation*}
d_N\left(\hat{\mathcal{D}}^i,\mathcal{D}\right)=\mathrm{vol}\left(\hat{\mathcal{D}}^i\Delta \mathcal{D}\right),
\end{equation*}\normalsize
where $\hat{\mathcal{D}}^i\Delta \mathcal{D}=\left(\hat{\mathcal{D}}^i\setminus \mathcal{D}\right)\cup \left(\mathcal{D}\setminus \hat{\mathcal{D}}^i\right)$ is the symmetric difference between $\hat{\mathcal{D}}^i$ and $\mathcal{D}$, and the volume $\mathrm{vol}(\cdot)$ refers to the $l$-dimensional Lebesgue measure. Given inner and outer approximations $\hat{\mathcal{D}}^i_-$ and $\hat{\mathcal{D}}^i_+$ that sandwich $\mathcal{D}$, we can upper bound the true error $d_N\left(\hat{\mathcal{D}}^i,\mathcal{D}\right)$ as follows.

\begin{lemma}[Sandwich Principle and Error Bound]\label{thm:errorBound}
Let $\hat{\mathcal{D}}^i_-$, $\hat{\mathcal{D}}^i$, and $\hat{\mathcal{D}}^i_+$ be estimators of $\mathcal{D}$ such that $\hat{\mathcal{D}}^i_-\subseteq \hat{\mathcal{D}}^i\subseteq \hat{\mathcal{D}}^i_+$ and $P(\hat{\mathcal{D}}^i_-\subseteq \mathcal{D} \subseteq \hat{\mathcal{D}}^i_+) \geq 1-\delta$ for $\delta \in (0,1)$. Then it holds\small
\begin{equation*}
P\left(d_N\left(\hat{\mathcal{D}}, \mathcal{D}\right) \leq \mathrm{vol}\left(\hat{\mathcal{D}}^i_+\setminus \hat{\mathcal{D}}^i_-\right)\right) \; \geq \; 1-\delta.
\end{equation*}\normalsize
\end{lemma}
\begin{proof}
See Section \ref{proof:thm:errorBound}.
\end{proof}

We consider two alternatives to define the lower and upper approximations $m^i_-,m^i_+\colon\mathbb{D}\to\mathbb{R}$ and apply them in the sandwich principle Lemma \ref{thm:errorBound}. These choices are associated to uniform and pointwise bounds, respectively.\footnote{Appendix \ref{sec:computeErrorBounds} discusses how to compute the proposed upper bounds. Appendix \ref{sec:interpret-pointwise} presents additional ideas for robustifying the pointwise bound.}

\paragraph{Uniform Bounds.}  
Functions $m^i_-,m^i_+\colon\mathbb{D}\to\mathbb{R}$ are called \emph{uniform bounds} for some $\delta\in(0,1)$, if \small
\begin{equation}\label{eq:credBand}
P\left(\forall~k\in\mathbb{D}\colon m^i_-(k)\leq M(k) \leq m^i_+(k)\mid \hat{M}(\mathbb{D}^i)=\hat{\mu}(\mathbb{D}^i)\right) \; \geq \;  1-\delta.
\end{equation}\normalsize
The region in between is referred to as a \emph{credible band} for the unknown function $\mu\colon\mathbb{D}\to\mathbb{R}$. 

\textcite{Lederer2019}\footnote{See, in particular, Theorem 3.1 in \textcite{Lederer2019}. These results are extended in \textcite{Lederer2021}.} discuss the derivation of such uniform bounds for Gaussian process regression. They impose Lipschitz conditions on the true function $\mu\colon\mathbb{D}\to\mathbb{R}$ and the covariance function $c\colon\mathbb{D}\times\mathbb{D}\to\mathbb{R}$. The uniform bounds can then be constructed by applying a bound on values at sampled points and bounding the values in between using the Lipschitz assumptions.\footnote{For $s\in\mathbb{R}_+$, the bounds can be defined by
\begin{equation*}
m^i_\pm(k) =m^i(k)\pm \sqrt{\alpha(s)} \sigma^i(k) \pm \beta(s),\quad k\in\mathbb{D},
\end{equation*}
where $\alpha(s)=2\log\left(\frac{M(s,\mathbb{D})}{\delta}\right)$ and $\beta(s)=(L_{m^i}+L_\mu)s+\sqrt{\alpha(s)}\omega_{\sigma^i}(s)$ for Lipschitz constants $L_{m^i},L_\mu\geq 0$, $\omega_{\sigma^i}(\cdot)$ being the modulus of continuity of $\sigma^i$, and $M(s,\mathbb{D})$ the $s$-convering number of $\mathbb{D}$. We refer to \textcite{Lederer2019} for details.} While a Lipschitz constant on $\mu$ is often unknown, a Lipschitz condition is satisfied by commonly used covariance functions.  As an immediate consequence, we can relate such uniform bounds to probabilistic bounds of the set $\mathcal{D}$, so that the sandwich principle can be applied. We set $\mathfrak{D}=\{k\in\mathbb{D}\colon M(k)\geq \gamma\}$.

\begin{corollary}[Credible Band for the Acceptable Design and Error Bound]\label{cor:error}
Define the estimators $\hat{\mathcal{D}}^i=\{k\in\mathbb{D}\colon m^i(k)\geq \gamma\}$, $\hat{\mathcal{D}}^i_-=\{k\in\mathbb{D}\colon m^i_-(k)\geq \gamma\}$, and $\hat{\mathcal{D}}^i_+=\{k\in\mathbb{D}\colon m^i_+(k)\geq \gamma\}$. If  condition \eqref{eq:credBand} is satisfied, then \small
\begin{align*}
P(\hat{\mathcal{D}}^i_-\subseteq \mathfrak{D}\subseteq \hat{\mathcal{D}}^i_+\mid \hat{M}(\mathbb{D}^i)=\hat{\mu}(\mathbb{D}^i))\geq 1-\delta,\\
P(d_N(\mathfrak{D}, \hat{\mathcal{D}}^i)\leq \mathrm{vol}(\hat{\mathcal{D}}^i_+\Delta \hat{\mathcal{D}}^i_-)\mid \hat{M}(\mathbb{D}^i)=\hat{\mu}(\mathbb{D}^i))\geq 1-\delta
\end{align*}\normalsize
where $\mathrm{vol}(\hat{\mathcal{D}}^i_+\Delta \hat{\mathcal{D}}^i_-)=\mathrm{vol}(\hat{\mathcal{D}}^i_+\setminus \hat{\mathcal{D}}^i_-)=\mathrm{vol}\left\{k\in\mathbb{D}\colon m^i_+(k)\geq \gamma > m^i_-(k)\right\}$.
\end{corollary}
\begin{proof}
See Section \ref{proof:cor:error}.
\end{proof}

Controlling the approximation error uniformly with high probability typically results in a large set $\hat{\mathcal{D}}^i_+\setminus \hat{\mathcal{D}}^i_-$ used to locate the boundary of $\mathfrak{D}$, which is not desirable. Moreover, uniform bounds require more prior information on the true function $\mu\colon\mathbb{D}\to\mathbb{R}$, i.e., the Lipschitz constant. These problems do not occur with pointwise bounds.

\paragraph{Pointwise Bounds.} In practice, it is often sufficient to assess the quality of a single candidate design $k$. This issue can be approached as follows.
\begin{lemma}[Pointwise Approximation Error]\label{lem:pointwiseApprox}
Let $\delta\in(0,1)$. It holds \small
\begin{equation*}
\forall~k\in\mathbb{D}\colon P\left(|M(k)-m^i(k)|\leq \Phi^{-1}\left(1-\frac{\delta}{2}\right)\sigma^i(k) \quad \bigg|  \quad \hat{M}(\mathbb{D}^i)=\hat{\mu}(\mathbb{D}^i)\right) \;\geq \; 1-\delta
\end{equation*} \normalsize
where $\Phi^{-1}$ denotes the inverse CDF of the standard normal distribution.
\end{lemma}

\begin{proof}
This is a standard argument, see Section \ref{proof:lem:pointwiseApprox} for a proof.
\end{proof}

This pointwise credible band can be attained with the functions $\bar{m}^i_-(k) := m^i(k) + \Phi^{-1}\left(\frac{\delta}{2}\right)\sigma^i(k)$ and $\bar{m}^i_+(k) := m^i(k) + \Phi^{-1}\left(1 - \frac{\delta}{2}\right)\sigma^i(k)$ for the GPR estimates $m^i,\sigma^i$, i.e.,\small
\begin{equation*}
P\left(\bar{m}^i_-(k)\leq M(k) \leq \bar{m}^i_+(k) \mid \hat{M}(\mathbb{D}^i)=\hat{\mu}(\mathbb{D}^i)\right) \; \geq \; 1-\delta \quad \forall~k \in \mathbb{D}.
\end{equation*}\normalsize

To set up a sandwich principle, we define the inner and outer approximations 
\begin{equation*}
\hat{\mathcal{D}}^i_\pm=\{k\in\mathbb{D}\colon \bar{m}^i_\pm(k)\geq \gamma\}
\end{equation*}\normalsize
and evaluate the Nikodym metric $d_N\left(\hat{\mathcal{D}}^i_-,\hat{\mathcal{D}}^i_+\right)$. For $k\in \hat{\mathcal{D}}^i_+\setminus \hat{\mathcal{D}}^i_-$, we have that $\gamma \in (\bar{m}^i_- (k), \bar{m}^i_+ (k)]$, thus, due to Lemma \ref{lem:pointwiseApprox} and the choice of $\bar{m}^i_-(k)$ and $\bar{m}^i_+(k)$,\small
$$P\left(| M(k) - \gamma| < 2\cdot  \Phi^{-1}\left(1 - \frac{\delta}{2}\right)\sigma^i(k) \quad  \bigg|  \quad\hat{M}(\mathbb{D}^i)=\hat{\mu}(\mathbb{D}^i)\right) \; \geq \;  1-\delta .$$ \normalsize

\section{Case Studies}\label{sec:case-studies}

We study two traffic networks, one with two signalized intersections and another one with variable capacities of highways, speed limits and bottlenecks due to roundabouts. Appendix \ref{app:TrafficSimulation} provides a pseudocode for the implementation of our traffic model for general networks. We investigate acceptable designs based on demand proportional flows. For selected design parameters, we compare these with the cooperative driving benchmark model. 

\subsection{Urban Network}\label{sec:casestudyI}

Traffic signals are essential control elements in modern urban networks. Their main function is to temporarily block certain traffic flows so that competing traffic flows can pass safely. Efficient placement and design are nontrivial problems; issues include the choice of traffic flow to be interrupted and the duration of the interruption. Complex interdependencies arise in networks, for example, when there is more than one traffic light in a network.

\subsubsection{Set-Up}

\begin{figure}[!htpb]
\centering
\includegraphics[scale=0.7]{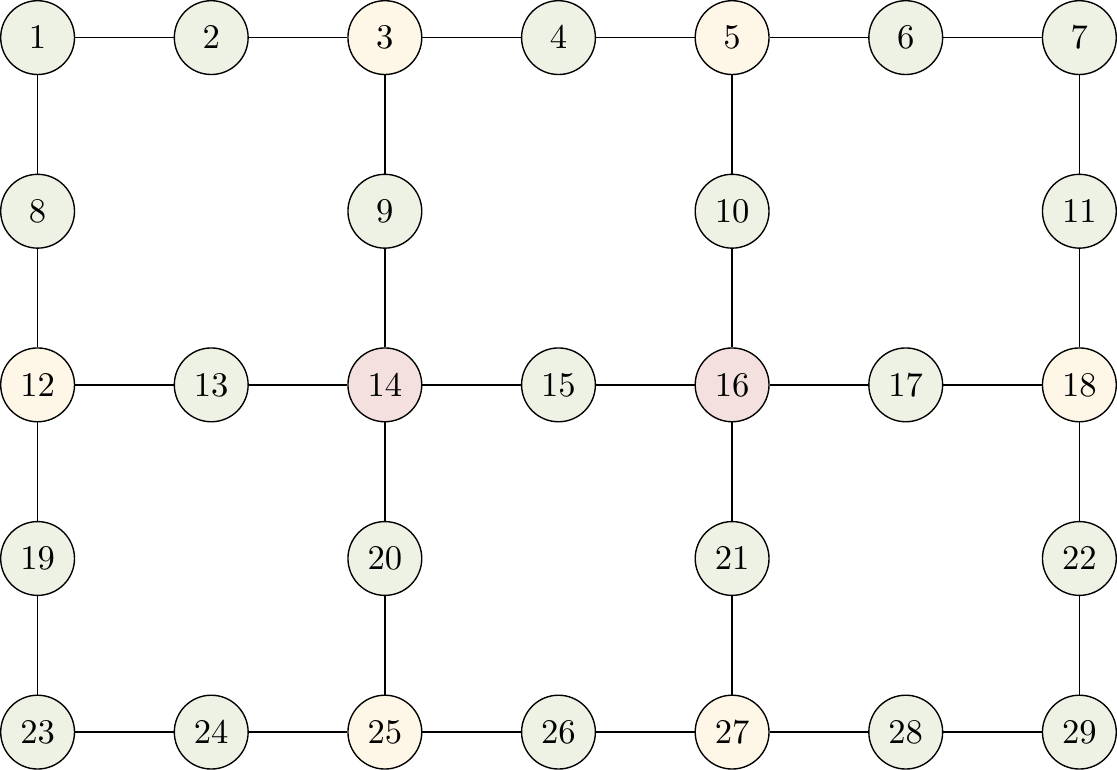}
\caption{Signalized urban network.}
\label{fig:network1}
\end{figure}

\paragraph{Network.} We consider a stylized urban network as depicted in Figure \ref{fig:network1}. It consists of 29 nodes $V=\{1,\dots,29\}$, which are interconnected in a grid-like manner. The network features three types of nodes (cf.~Section \ref{sec:ex_tn}):
\begin{itemize}
\item \emph{Signalized Intersection.}  We consider two signalized intersections $\mathcal{R}=\{14,16\}$, marked in red. Both are connected to four adjacent intersections and allow vehicles to travel horizontally or vertically through the network while blocking vehicles in orthogonal directions. We study the duration of their green time $T^g$ and the displacement of the two green phases $T^s$ of nodes $14$ and $16$ as important design parameters. Details on the implementation of the traffic lights are given in Section \ref{sec:trafficLightImplementation}.
\item \emph{Unsignalized Intersection.} On the periphery are six unsignalized intersections $\mathcal{Y}=\{3,5,12,15,25,27\}$, highlighted in yellow. Each of these intersections connects traffic flows from three adjacent nodes.
\item \emph{Bidirectional Road.} The remaining nodes of the network are simple bidirectional roads $\mathcal{G}=V\setminus(\mathcal{R}\cup\mathcal{Y})$ which are highlighted in green.
\end{itemize}

We assume that the signalized and unsignalized intersections have the same diameter of $\unit[40]{m}$ while the green connecting bidirectional roads are of length $\unit[120]{m}$. Vehicles move at a free-flow speed of $\unit[50]{km/h}$. We assume symmetric and constant turning fractions.\footnote{The turning fractions are $f_{(x,u,v)\to w}(t+1)=1$ for $v\in\mathcal{G}$ $w\in\cO(v),~w\neq u\in\cI(v),~u\neq x\in\cI(u)$ and, in analogy, $f_{(x,u,v)\to w}(t+1)=1/2$ for $v\in\mathcal{Y}$ and $f_{(x,u,v)\to w}(t+1)=1/3$ for $v\in\mathcal{R}$.} The connecting bidirectional roads $\mathcal{G}$ are of length $\unit[120]{m}$. We normalize traffic densities, i.e., we let $l_v=1$ for $v\in \mathcal{R}\cup\mathcal{Y}$ and set $l_v=3$ for $v\in\mathcal{G}$. As an initial configuration, we set $\rho_{(\cdot,v,\cdot)}(0)=5$ for $v\in\mathcal{G}$ and $\rho_{(\cdot,v,\cdot)}(0)=1$ for $v\in\mathcal{Y}\cup\mathcal{R}$. At $\unit[50]{km/h}$, intersections can be passed in $\unit[2.88]{s}$, i.e., the interval $[t,t+1]$ corresponds to $t^\mathrm{real}=\unit[2.88]{s}$. We simulate $T=1,250$ time steps corresponding to $\unit[1]{h}$ of traffic. The remaining parameters of the modules are listed in Table \ref{table:parameterNetwork1}.

\begin{table}[!htbp]
\centering 

\begin{threeparttable}
\caption[]{Parameter choice.} 
\label{table:parameterNetwork1}
	\begin{tabular}{rcccccc}\toprule 
		
&		$s^\mathrm{max}_v$ & $\rho^\mathrm{max}_v$& $a_v$ & $b_v$ & $c_v$ & $\zeta_v$ \\ \toprule 
$v\in\mathcal{R}$&		$5$ & $16$ &  $1$ & $1$ & $1$  & $1/10$\\
$v\in\mathcal{Y}$&		$5$ & $10$ &  $1$ & $1$ & $1$  & $1/10$\\
$v\in\mathcal{G}$&		$5$ & $30$ &  $1$ & $1$ & $1$  & -\\
\bottomrule 
\end{tabular} 

\end{threeparttable} 

\end{table}

\paragraph{Random Environment.} We place the traffic network described above in a random environment by introducing random sources and sinks. Specifically, we implement two stochastic processes to model $q^\mathrm{net}_{(6,7,11)}(t)$ and $q^\mathrm{net}_{(24,23,19)}(t)$ for $t=1,\dots,T$. As building blocks, we take two autoregressive models with a given dependence structure. The details are described in Section \ref{sec:highwayNetworkDetailsA}. The key idea is that both time series models include white noise $\varepsilon_{(6,7,11)}(t+1)$ and $\varepsilon_{(24,23,19)}(t+1)$. For each time step, these are normal random variables centered around $0$ with standard deviations $\sigma_{(6,7,11)}\geq 0$  and  $\sigma_{(24,23,19)}\geq 0$. We assume a particular dependence structure on $\varepsilon_{(6,7,11)}(t+1)$ and $\varepsilon_{(24,23,19)}(t+1)$ modeled by the Frank copula\footnote{A copula is a multivariate distribution function with uniform marginals. It captures dependence among the marginals of a random vector by virtue of Sklar's theorem. We refer to \textcite{Mcneil2015} for more details.}
\begin{equation*}
C^r(u_1,u_2)=-\frac{1}{r}\log\left(1+\frac{(e^{-ru_1}-1)(e^{-ru_2}-1)}{e^{-r}-1}\right),\quad u_1,u_2\in(0,1),
\end{equation*}
which is parametrized by $r\in\mathbb{R}$. The Frank copula interpolates from full countermonotonicity for $r\to -\infty$ (i.e.,  $\varepsilon_{(6,7,11)}(t+1)$ is large when $\varepsilon_{(24,23,19)}(t+1)$ is small and vice versa) to full comonotonicity for $r\to\infty$ (i.e.,  $\varepsilon_{(6,7,11)}(t+1)$ and $\varepsilon_{(24,23,19)}(t+1)$ move in the same direction). In the limit $r\to 0$, $\varepsilon_{(6,7,11)}(t+1)$ and $\varepsilon_{(24,23,19)}(t+1)$ are stochstically independent.

\subsubsection{Acceptable Configurations and Design}\label{sec:caseStudy1accept}

\paragraph{Construction of Acceptance Sets.} To assess the performance of traffic systems, we employ the normative approach of Section \ref{sec:pfad} and compare acceptable designs $\mathcal{D}_{u,\gamma}=\left\{ k\in\mathbb{D}\colon \mathbb{E}(u(Q_k)) \geq \gamma\right\}$ for different utility functions $u$ and levels $\gamma$. Specifically, we consider expectation, polynomial utility, expectile utility and square root utility.\footnote{While the square-root utility is a standard utility function (increasing and concave), polynomial utility and expectile utility, with appropriately chosen constants $c_p,c_e\in\mathbb{R}$, place special emphasis on downside risk. We set $c_p=2\mathbb{E}(Q_{k^*})$ and $c_e=\mathbb{E}(Q_{k^*})$ for a ``good'' design parameter $k^*\in\mathbb{D}$. Polynomial utility evaluates only flows smaller than $c_p$, and expectile utility evaluates random fluctuations around $c_e$ asymmetrically, with a stronger penalty for outcomes below this value.}

The comparison across different utility functions is facilitated by calibrating the thresholds to benchmark flow distributions $\tilde{Q}^A$, $\tilde{Q}^B$, $\tilde{Q}^C$. For all utility functions, we choose the corresponding threshold levels as  $\gamma_u^A=\mathbb{E}(u(\tilde{Q}^A))$, $\gamma_u^B=\mathbb{E}(u(\tilde{Q}^B))$, and $\gamma_u^C=\mathbb{E}(u(\tilde{Q}^C))$.\footnote{Specifically, let $X=X(\beta)\sim \mathrm{Beta}(\beta,\beta)$ be a beta distribution with mean $1/2$ and standard deviation $\sigma(X)=1/\sqrt{8\beta+4}$. We compute $\beta^A$, $\beta^B$ and $\beta^C$ such that $\sigma(X(\beta^A))=0.1$, $\sigma(X(\beta^B))=0.15$ and $\sigma(X(\beta^C))=0.2$. We obtain benchmark flow distributions $\tilde{Q}^A$, $\tilde{Q}^B$, $\tilde{Q}^C$ by setting $\tilde{Q}^A=e^A\cdot 2X(\beta^A)$, $\tilde{Q}^B=e^B\cdot 2X(\beta^B)$, and $\tilde{Q}^C=e^C\cdot 2X(\beta^C)$ for chosen benchmark expectations $e^A>e^B>e^C>0$.  Numerical evaluation yields the corresponding thresholds $\gamma^A=\mathbb{E}(u(\tilde{Q}^A))$, $\gamma^B=\mathbb{E}(u(\tilde{Q}^B))$, and $\gamma^C=\mathbb{E}(u(\tilde{Q}^C))$ for the different utility functions. Beta distributions were chosen as benchmarks because they are simple two-parametric distributions on compact intervals that generalize the uniform distribution; they are uniquely specified by their mean and variance.}

\paragraph{Simulation Set-Up.}

We consider five design parameters that characterize traffic models, $k=\left(r,\sigma_{(6,7,11)},\sigma_{(24,23,19)},T^g,T^s\right)$:
\begin{itemize}
\item \emph{Sources and Sinks.} We vary the dependence structure of the autoregressive models determining the source and sink flows, i.e., the dependence parameter $r$ and the respective standard deviations of the noise terms $\sigma_{(6,7,11)}$ and $\sigma_{(24,23,19)}$.
\item \emph{Signal Control.} We vary the duration of the green phases of the two traffic lights $T^g$ and the displacement of the green phases $T^s$.
\end{itemize} 
Performance is characterized by average network traffic flow defined by
\begin{equation*}
Q=Q_k=\frac{1}{T}\sum_{t=0}^{T-1}\sum_{v\in V}\sum_{u\in\cI(v)}\sum_{x\in\cI(u)} q_{(x,u, v)}^\mathrm{out}(t+1)
\end{equation*}
for $k\in \mathbb{D}: =[-50,50]\times [0,0.025] \times  [0,0.025]\times [0,100]\times[0,100]$. The network flow is simulated\footnote{Since this is a discrete-time model, this property must be respected for each simulation run by $T^g, T^s$. This issue is addressed by independently sampling discrete numbers for each simulation run from $\{\lfloor T^g \rfloor, \lceil T^g \rceil \}$ and $\{\lfloor T^s \rfloor, \lceil T^s \rceil\}$ such that their respective expected value is $T^g$ or $T^s$.} and evaluated by our stochastic search algorithm.\footnote{For the level set estimation, we apply the following computational budget: We consider $8$ iterations of our algorithm, where $n_\mathrm{initial}=150$ points are sampled uniformly in the initial phase and $n_\mathrm{loop}=50$ are sampled in the following $7$ iterations according to the acquisition function. We define target noises as $\{5\%,~10\%,~8\%,~6\%,~5\%,~4\%,~3\%,~2\%\}\cdot (\gamma^C-\gamma^A)$ for the respective utility functions. We consider at least $n_\mathrm{min}=20$ independent simulations and set $n_\mathrm{max}=500,~150,~200,~300,~400,~650,~1200,~3000$.\newline In our case studies, we find that a design parameter $k^*$ for the first network is good if $\mathbb{E}(Q_{k^*})=60$; therefore, we set $e^A=60$, $e^B=55$, $e^C=50$ and use these to calibrate $\gamma^A$, $\gamma^B$, and $\gamma^C$ as described above.}    

\subsubsection{Results}

\paragraph{Impact of Sources and Sinks.} 
Fixing $T^g=10$, $T^s=0$, we vary the parameters $r,\sigma_{(6,7,11)},\sigma_{(24,23,19)}$ under the condition $\sigma_{(6,7,11)}=\sigma_{(24,23,19)}$. The impact of the dependence structure on performance is small. Increasing noise decreases the system performance considerably.\footnote{This is shown in Figure \ref{fig:network1AcceptableDependence} and Figure \ref{fig:network1AcceptableNoise} in the appendix.}

\paragraph{Acceptable Traffic Lights Design.} Second, we fix the parameters of the random environment as $r=2.5$ and $\sigma_{(6,7,11)}=\sigma_{(24,23,19)}=0.01$ and examine the acceptable designs of the two traffic lights. Figure \ref{fig:network1AcceptableLightsMATERN} shows acceptable configurations of green time duration $T^g$ and shift $T^s$. The resulting quantities are nontrivial. 

$\mathbb{E}(u(Q))$, as a function of $T^g$ and $T^s$, is concave in $T^g$ -- first increasing and then decreasing, since too short green times lead to low traffic flow and are unacceptable, and the same applies to values that are too high; moreover, it is periodic in $T^s$. The simulations show that for longer green times $T^g$ (i.e., $T^g\geq 40$), acceptable designs can also be found on a diagonal in the $(T^g, T^s)$ plane. The comparison across different utility functions indicates qualitatively similar behavior. However, different normative assessments of risk are reflected in the different sizes of the domains that are preferred over the three benchmark levels.

\begin{figure}[!htbp]

\begin{minipage}{\textwidth}
\centering

\includegraphics[scale=0.4,trim={4.5cm 8cm 5cm 8cm},clip]{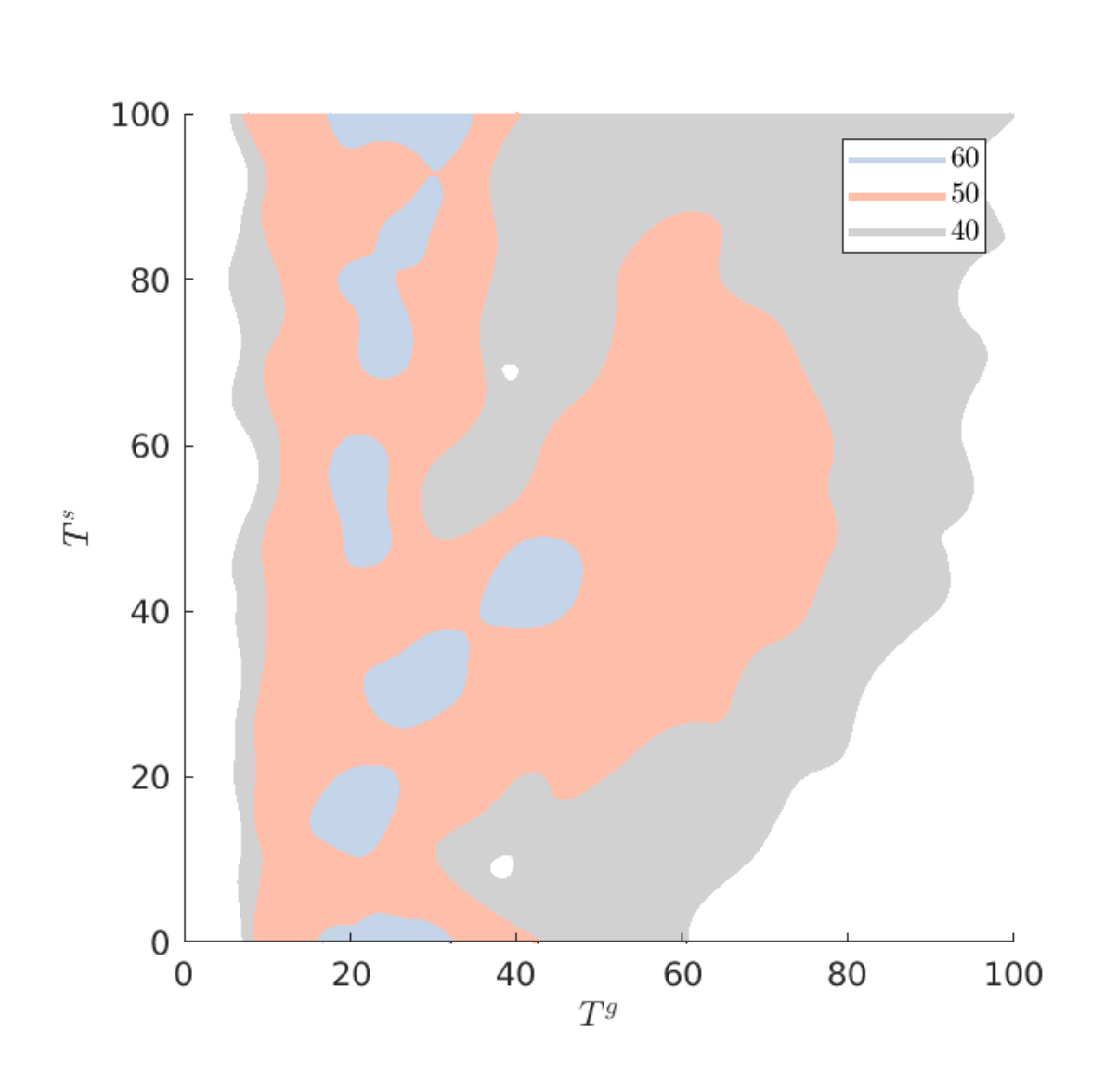}

\subcaption{$u(x)=x$}

\end{minipage}

\begin{minipage}{0.3\textwidth}
\centering

\includegraphics[scale=0.35,trim={4.5cm 8cm 5cm 8cm},clip]{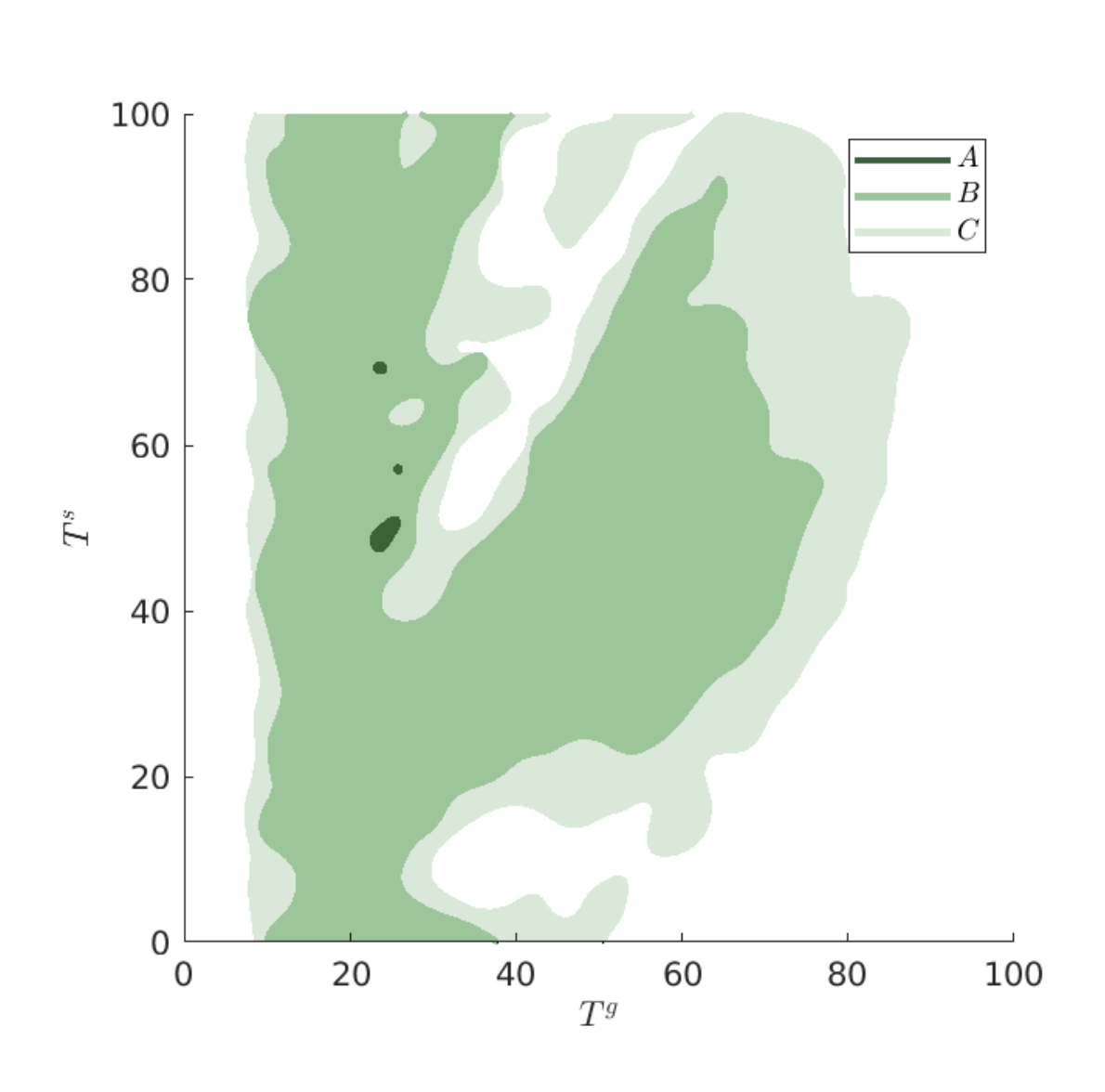}

\subcaption{$u(x)=0.1(x-60)_+ - 0.9(x-60)_-$ }

\end{minipage}\hspace{0.5cm}\begin{minipage}{0.3\textwidth}
\centering

\includegraphics[scale=0.35,trim={4.5cm 8cm 5cm 8cm},clip]{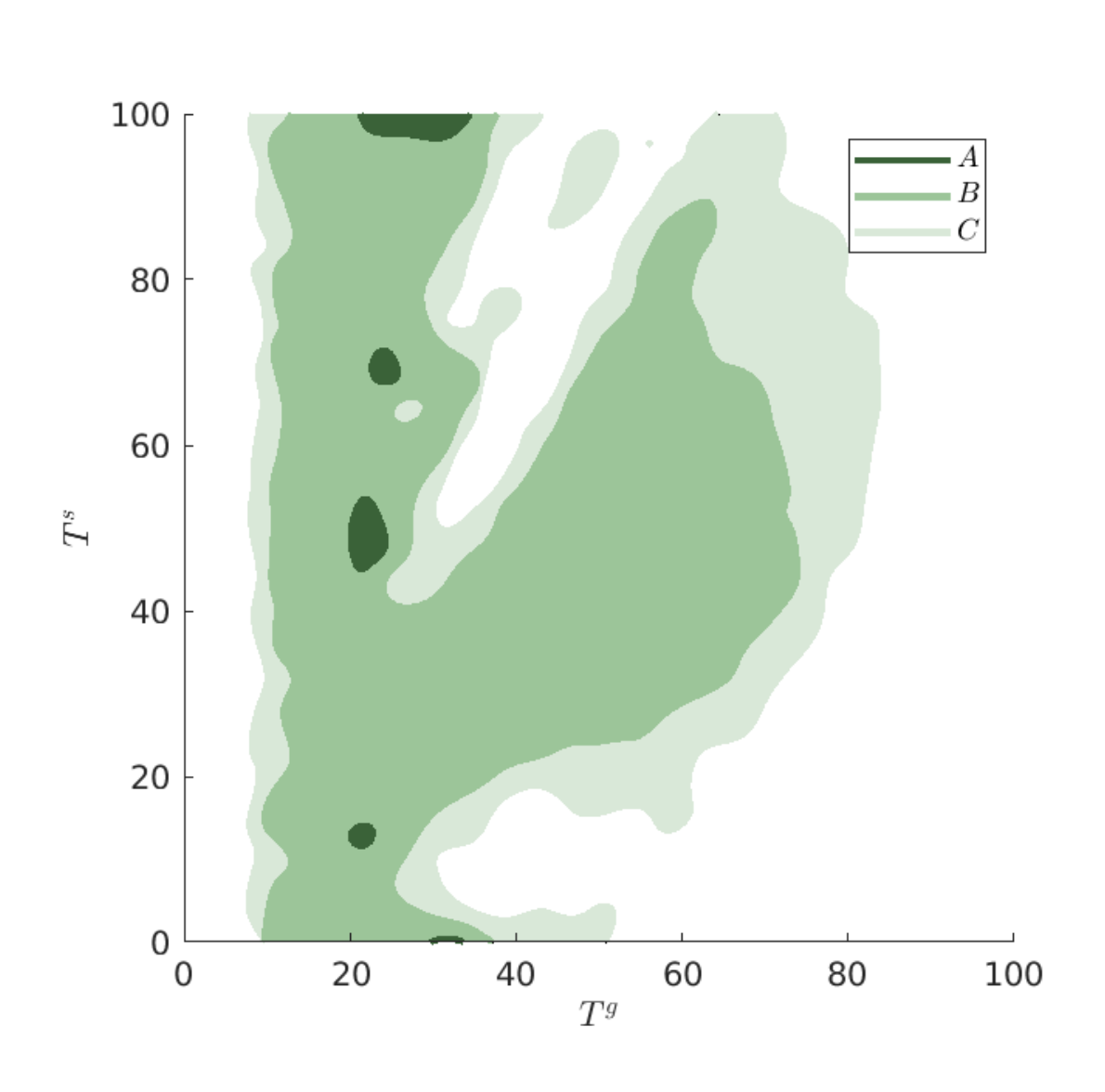}

\subcaption{$u(x)=0.2(x-60)_+ - 0.8(x-60)_-$}

\end{minipage}\hspace{0.5cm}\begin{minipage}{0.3\textwidth}
\centering

\includegraphics[scale=0.35,trim={4.5cm 8cm 5cm 8cm},clip]{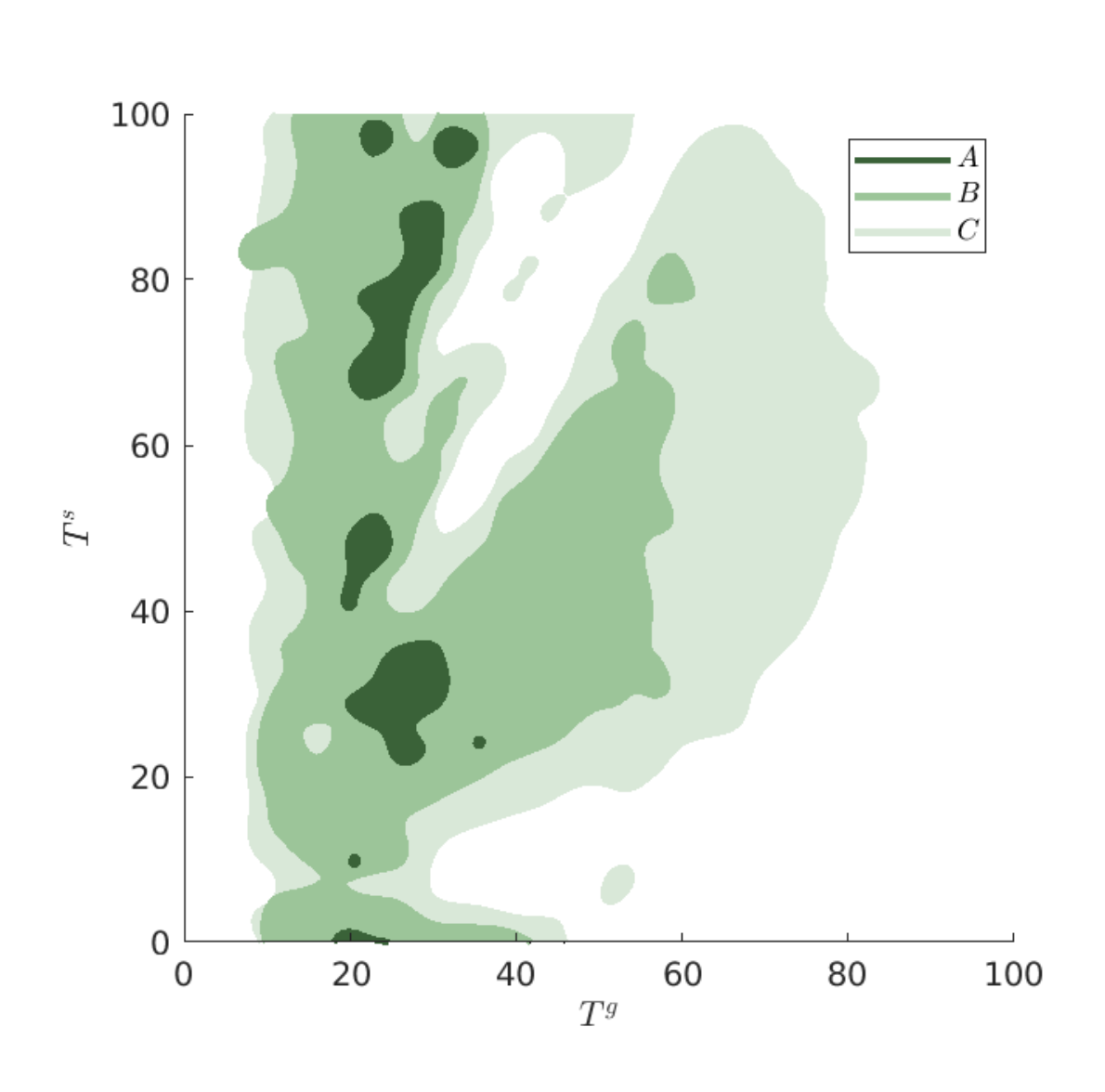}

\subcaption{$u(x)=\sqrt{x}$}

\end{minipage}

\begin{minipage}{0.3\textwidth}
\centering

\includegraphics[scale=0.35,trim={4.5cm 8cm 5cm 8cm},clip]{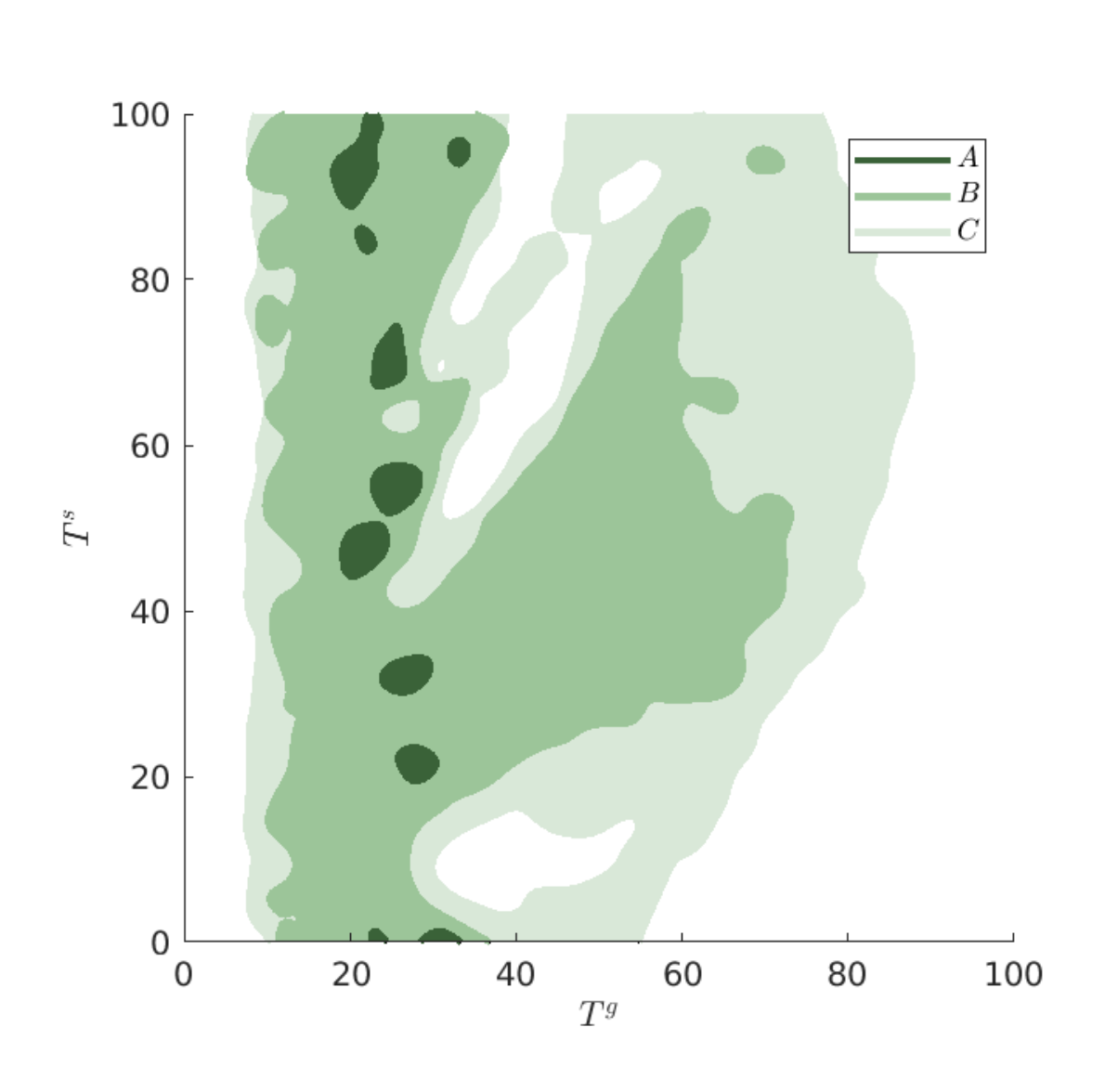}

\subcaption{$u(x)=-|x-2\cdot 60|^3\mathbbm{1}\{x\leq 2\cdot 60\}$}

\end{minipage}\hspace{0.5cm}\begin{minipage}{0.3\textwidth}
\centering

\includegraphics[scale=0.35,trim={4.5cm 8cm 5cm 8cm},clip]{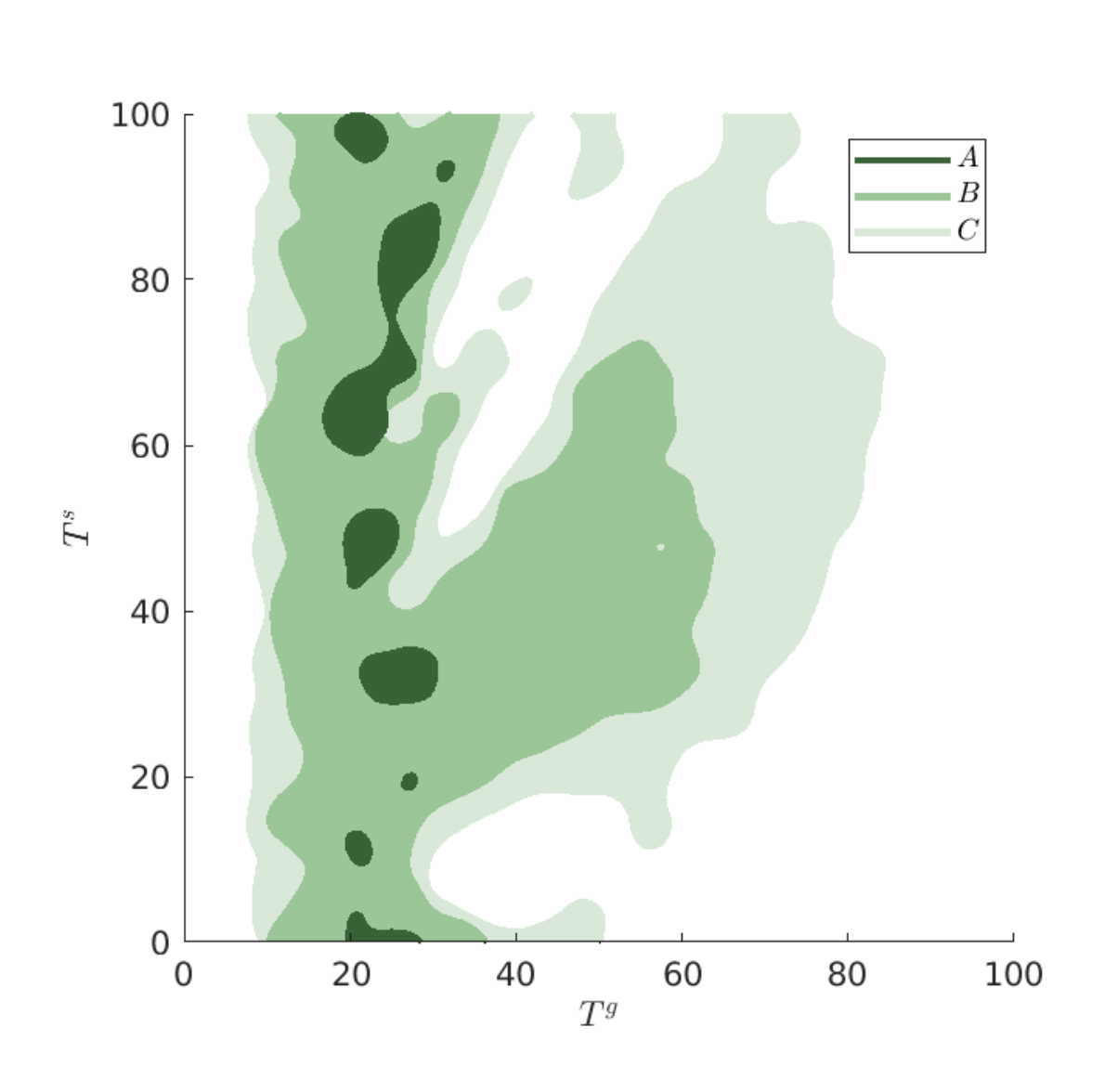}

\subcaption{$u(x)=-|x-2\cdot 60|^2\mathbbm{1}\{x\leq 2\cdot 60\}$}

\end{minipage}\hspace{0.5cm}\begin{minipage}{0.3\textwidth}
\centering

\includegraphics[scale=0.35,trim={4.5cm 8cm 5cm 8cm},clip]{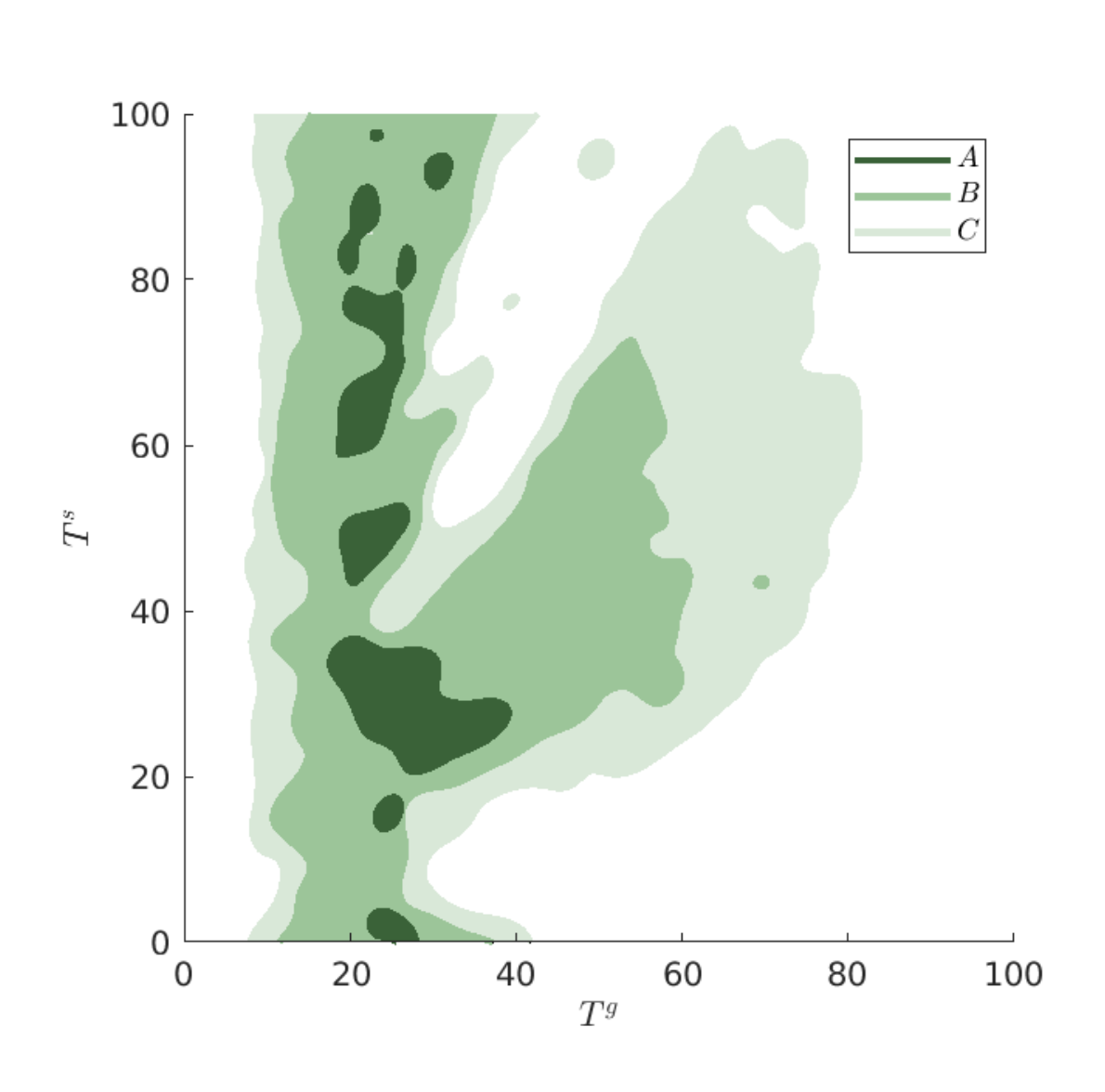}

\subcaption{$u(x)=-|x-2\cdot 60|^{3/2}\mathbbm{1}\{x\leq 2\cdot 60\}$}

\end{minipage}

\caption{Acceptable traffic lights configurations for $r=2.5$ and $\sigma_{(6,7,11)}=\sigma_{(24,23,19)}=0.01$. GPR is based on the Matérn kernel.}
\label{fig:network1AcceptableLightsMATERN}
\end{figure}

\paragraph{Comparison with the Cooperative Driving Benchmark Model.} The preceding results are based on the demand proportional flows as an interaction rule.  We investigate the benchmark model of cooperative driving as a theoretical alternative for selected design parameters.\footnote{The simulations are more complex because they have to solve a linear program for each node per time step. We use Matlab's built-in solver for these linear programs.}
Table \ref{tab:comp} compares the expected traffic flows for four selected design parameters. We find that the myopically optimized cooperative driving model yields moderately but consistently higher traffic flows in these examples.

\begin{table}[!htbp]
\centering 

\begin{threeparttable}
\caption[]{Comparison with the benchmark model for cooperative driving.} 
\label{tab:comp}
\begin{tabular}{rcccc} \toprule 
& \multicolumn{4}{c}{$(T^g,T^s)$}    \\ \midrule 
& $(20,75)$& $(20,10)$ & $(30,30)$ & $(30,20)$  \\     \toprule 
DPF &		 60.48 & 60.49 & 61.75 & 59.17 \\
CDBM &		 65.78 & 64.71 & 64.29 & 64.28\\
\bottomrule 
\end{tabular} 
\begin{tablenotes} 
\small 
\item Approximation of $\mathbb{E}(Q)$ for demand proportional flows (DPF) and the cooperative driving benchmark model (CDBM) based on $500$ independent simulations per traffic light configuration.
\end{tablenotes} 
\end{threeparttable}

\end{table}

\subsection{Highway Network}

Highways are an important part of the road network as they allow efficient travel over longer distances.  Vehicles can move at higher speeds, with the risk of accidents limited by the fact that the directions of travel are physically separated. Control mechanisms (such as different speed limits, optional driving on the shoulder, etc.) can be used to increase efficiency, depending on the traffic situation.

\subsubsection{Set-Up}

\paragraph{Network.} 

We consider a stylized highway network as shown in Figure \ref{fig:network2}. The network consists of 33 nodes $V=\{1,\dots,33\}$. Three roundabouts $\mathcal{Y}=\{1,12,23\}$ connect two types of highway modules. Those in the core $\mathcal{R}=\{7,\dots,11,18,\dots,22,29,\dots,33\}$ allow optional driving on the shoulder, while those in the periphery $\mathcal{G}=V\setminus\{\mathcal{Y}\cup\mathcal{R}\}$ do not. When the shoulder is open to traffic, the maximum density and maximum traffic flow are increased while the maximum speed is reduced.

All highway modules $\mathcal{G}\cup\mathcal{R}$ have the same length of $\unit[3]{km}$, while the roundabouts have a diameter of $\unit[200]{m}$. We assume a free-flow speed on the highways of $\unit[100]{km/h}$ and an average speed of $\unit[40]{km/h}$ on the roundabout. Highway modules $v\in\mathcal{R}$ allow for an optional driving on the hard shoulder: In this case, maximal flow and density are increased by $50$ \%, while the free-flow speed is decreased to $\unit[80]{km/h}$. We normalize traffic densities in terms of the roundabouts, i.e., we let $l_v=1$ for $v\in \mathcal{Y}$ and set $l_v=15$ for $v\in\mathcal{G}\cup\mathcal{R}$. Moreover, we specify symmetric and constant turning fractions, i.e., $f_{(x,u,v)\to w}(t+1)=1/3$ for $v\in\mathcal{Y}$ and  $w\in\cO(v),~w\neq u\in\cI(v),~u\neq x\in\cI(u)$ and $f_{(x,u,v)\to w}(t+1)=1$ for $v\in\mathcal{G}\cup \mathcal{R}$ in analogy. 

At $\unit[100]{km/h}$, a roundabout would be traversed in $\unit[7.2]{s}$, i.e., the interval $[t,t+1]$ corresponds to $t^\mathrm{real}=\unit[7.2]{s}$. Initially, the system is homogeneously filled with $5$\,\% of the maximum density.\footnote{To be precise, let $\rho^0=\sum_{v\in V}\sum_{u\in\mathcal{I}(v),~u\neq w\in\mathcal{O}(u)} \rho_{(u,v,w)}(0) /\sum_{v\in V}\rho_v^\mathrm{max}\in[0,1]$ denote the total initial density, as a fraction of the maximal density. For given $\rho^0=0.05$, we distribute the density homogeneously via
\begin{equation*}
\rho_{(u,v,w)}(0) =\frac{\rho_v^\mathrm{max}}{\#\{(u,v,w)\colon u\in\mathcal{I}(v),~u\neq w\in\mathcal{O}(u)\}}\cdot\rho^0.
\end{equation*}} We simulate $T=500$ time steps corresponding to $\unit[1]{h}$ of traffic. The remaining parameters of the modules are given in Table \ref{table:parameterNetwork2}. In this section, (CLOSED) means that the shoulder is closed to traffic, while (OPEN) denotes the configuration in which it is open.

\begin{table}[!htbp]
\centering 

\begin{threeparttable}
\caption[]{Parameter choice.} 
\label{table:parameterNetwork2}
	\begin{tabular}{lcccccc}\toprule 
		
&		$s^\mathrm{max}_v$ & $\rho^\mathrm{max}_v$& $a_v$ & $b_v$ & $c_v$ & $d_v$ \\ \toprule 
$v\in\mathcal{G}$&		$20$ & $100$ &  $1$ & $1$& $1$& $1$ \\
$v\in\mathcal{Y}$&		$5$ & $30$ &  $2/5$ & $1$& $1$& $1$   \\
$v\in\mathcal{R}$ (CLOSED)&		$20$ & $100$ &  $1$ & $1$& $1$& $1$   \\
$v\in\mathcal{R}$ (OPEN)&		$30$ & $150$ &  $4/5$ & $1$& $1$& $1$  \\
\bottomrule 
\end{tabular} 

\end{threeparttable} 

\end{table}

\paragraph{Random Sources and Sinks.} We focus mainly on traffic flowing from the bottom to the top of the network. For this purpose, we consider sources and sinks at the periphery, so that the corresponding traffic has to cross two traffic circles. These can be considered as bottlenecks in the network. We mix this traffic with additional traffic in the core of the network that uses only the highways, which have a high capacity; however, heavy traffic can cause congestion. Opening the shoulder with a lower speed limit can alleviate this problem. 

\begin{itemize}
\item \emph{Periphery.} The sources at the bottom of the network are associated with auxiliary flows $q_{(5,4,3)}^\mathrm{aux}(t+1)=q_{(14,15,16)}^\mathrm{aux}(t+1)$, where $q_{(5,4, 3)}^\mathrm{aux}(t+1)\sim \mathcal{N}(\xi_1,\psi_1^2 \cdot \xi_1^2)$. At the top, deterministic sinks are defined via auxiliary flows $q_{(27,26,25)}^\mathrm{aux}(t+1)=q_{(25,26,27)}^\mathrm{aux}(t+1)=-2\xi_1$.
\item \emph{Core.} The additional traffic on the highways is specified by sources and sinks associated with the following auxiliary flows:
\begin{alignat*}{3}
q_{(18,19,20)}^\mathrm{aux}(t+1)&\sim \mathcal{N}(\xi_2,\psi^2_2\cdot \xi_2^2),\quad &&q_{(20,21,22)}^\mathrm{aux}(t+1)&&=-q_{(18,19,20)}^\mathrm{aux}(t+1),\\
q_{(33,32,31)}^\mathrm{aux}(t+1)&\sim \mathcal{N}(\xi_2,\psi^2_2\cdot \xi_2^2),\quad &&q_{(31,30,29)}^\mathrm{aux}(t+1)&&=-q_{(33,32,31)}^\mathrm{aux}(t+1),\\
q_{(7,8,9)}^\mathrm{aux}(t+1)&\sim \mathcal{N}(\xi_2,\psi^2_2\cdot \xi_2^2),\quad &&q_{(9,10,11)}^\mathrm{aux}(t+1)&&=-q_{(7,8,9)}^\mathrm{aux}(t+1),\\
q_{(11,10,9)}^\mathrm{aux}(t+1)&\sim \mathcal{N}(\xi_2,\psi^2_2\cdot \xi_2^2),\quad &&q_{(9,8,7)}^\mathrm{aux}(t+1)&&=-q_{(11,10,9)}^\mathrm{aux}(t+1).
\end{alignat*}
\end{itemize}

All random variables are assumed to be independent. As in Case Study 1, $q^\mathrm{net}_{(u,v,w)}(t+1)$ is equal to $q^\mathrm{aux}_{(u,v,w)}(t+1)$, if this leads to $0\leq \rho_{(u,v,w)}(t+1) \leq (\rho_v^{\mathrm{max}})/2$; otherwise, the absolute value of $q^\mathrm{aux}_{(u,v,w)}(t+1)$ is reduced, such that one of these boundaries is attained. The quantity $q^\mathrm{aux}_{(u,v,w)}(t+1)$ should be interpreted as the flow of vehicles that attempt to enter the network in the considered time period. In our simulations, we set the coefficient of variation to be $\psi_1=\psi_2=0.1$ and vary $\xi_1$ and $\xi_2$.

\subsubsection{Acceptable Configurations and Design}

We study the impact of varying the design parameters $\xi_1$ and $\xi_2$, which control the volume and fluctuation of traffic originating from the periphery and the core, respectively. We compare two highway configurations: driving on shoulder prohibited (CLOSED) vs. driving on shoulder allowed (OPEN).  We focus on the mean performance $\mathbb{E}(\cdot)$ and consider two different performance measures.

The set $\mathcal{N}$ contains all travel directions $(u,v,w)$ that are sources or sinks. Letting
\begin{align*}
Q^a&=\frac{\frac{1}{T}\sum_{t=0}^{T-1}\sum_{(u,v,w)\in\mathcal{N}} \left( q_{(u,v,w)}^\mathrm{net}(t+1)\right)_-}{\frac{1}{T}\sum_{t=0}^{T-1}\sum_{v\in V}\sum_{(u,v,w)\in\mathcal{N}} \left( q_{(u,v,w)}^\mathrm{aux}(t+1)\right)_+},
\end{align*}
$Q^a$ can be regarded as measure of the actual throughput: The sum of the flows that are actually removed from the network is divided by the sum of the available flows that attempt to enter the network.

Another performance measure is
\begin{align*}
Q^b&=\frac{1}{T}\sum_{t=0}^{T-1}\left( \frac{q_{(12,18,19)}^\mathrm{out}(t+1)}{\rho_{(12,18,19)}(t)} + \frac{q_{(1,33,32)}^\mathrm{out}(t+1)}{\rho_{(1,33,32)}(t)} \right).
\end{align*}
that measures average velocity on $(12,18,19)$ and $(1,33,32)$ by considering the fraction of flow that actually moves divided by the available density. We compute acceptable designs\footnote{For the level set estimation, we apply the following computational budget: We consider $8$ iterations of our algorithm, where $n_\mathrm{initial}=100$ points are sampled uniformly in the initial phase and $n_\mathrm{loop}=50$ are sampled in the following $7$ iterations according to the acquisition function. We define target noises as $\{5\%,~10\%,~8\%,~6\%,~5\%,~4\%,~3\%,~2\%\}\cdot 0.1$. We consider at least $n_\mathrm{min}=20$ independent simulations and set $n_\mathrm{max}=500,~150,~200,~300,~400,~650,~1200,~3000$.} based on the two performance measures $Q^a$ and $Q^b$ in the two regions $\mathbb{D}^a=[1,61]\times[1,61]$ and $\mathbb{D}^b=[1,31]\times[1,31]$.\footnote{We assume $\xi_1,\xi_2\geq 1>0$ to exclude simulations with almost no traffic that might lead to small values in the denominators of the performance measures.}

\subsubsection{Results}

We evaluate the two highway configurations (CLOSED) and (OPEN) based on the two performance measures $Q^a$ and $Q^b$. To better compare the driving configurations, we also investigate the differences for (CLOSED) and (OPEN). The results are presented in Figure \ref{fig:resultsCaseStudyB}.

\begin{figure}[!htbp]

\begin{minipage}{0.5\textwidth}
\centering

\includegraphics[scale=0.35,trim={3.5cm 7cm 4cm 7cm},clip]{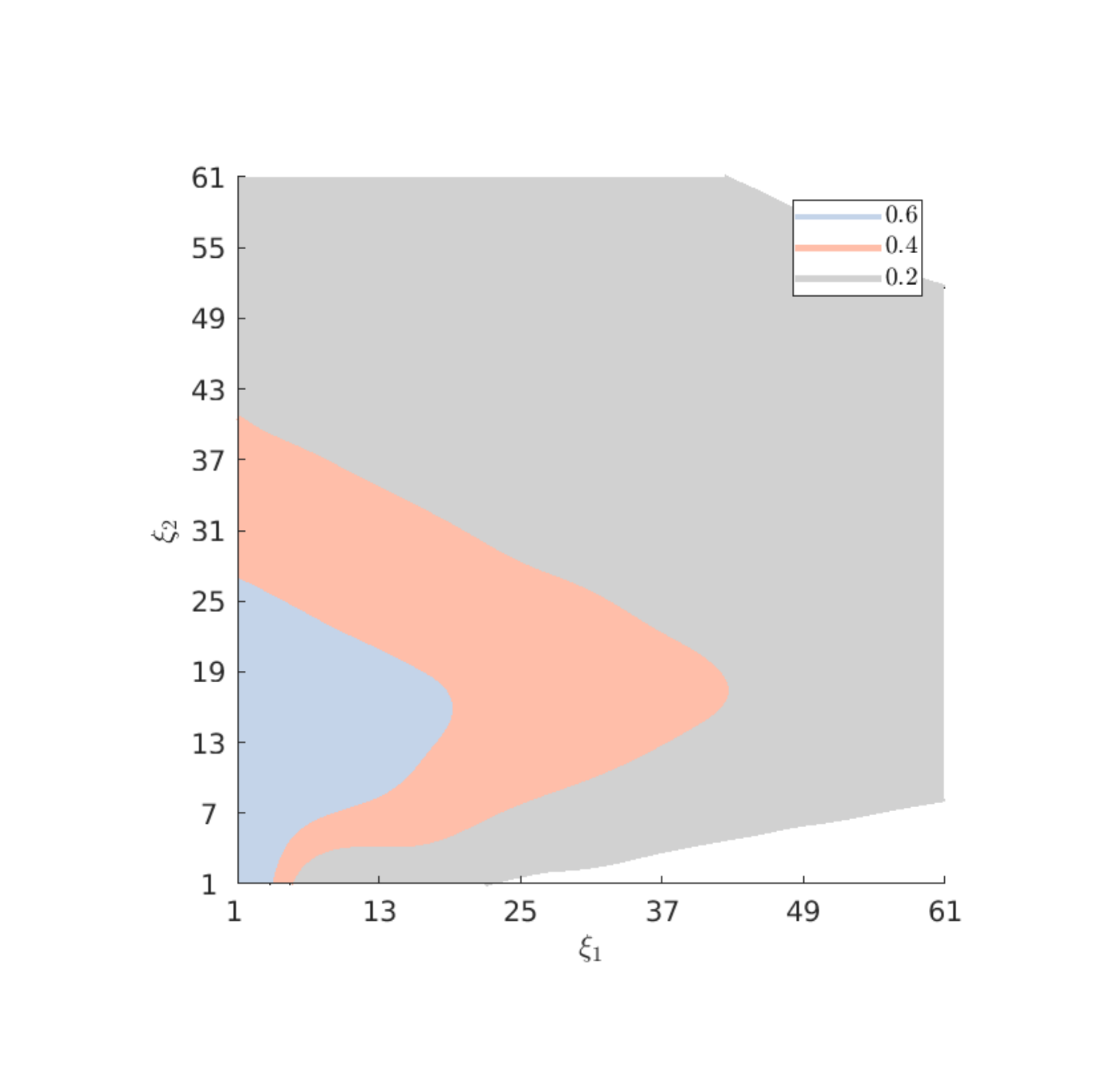}

\subcaption{$Q^a$: (CLOSED)}\label{fig:QA-closed}

\end{minipage}\begin{minipage}{0.5\textwidth}
\centering

\includegraphics[scale=0.35,trim={3.5cm 7cm 4cm 7cm},clip]{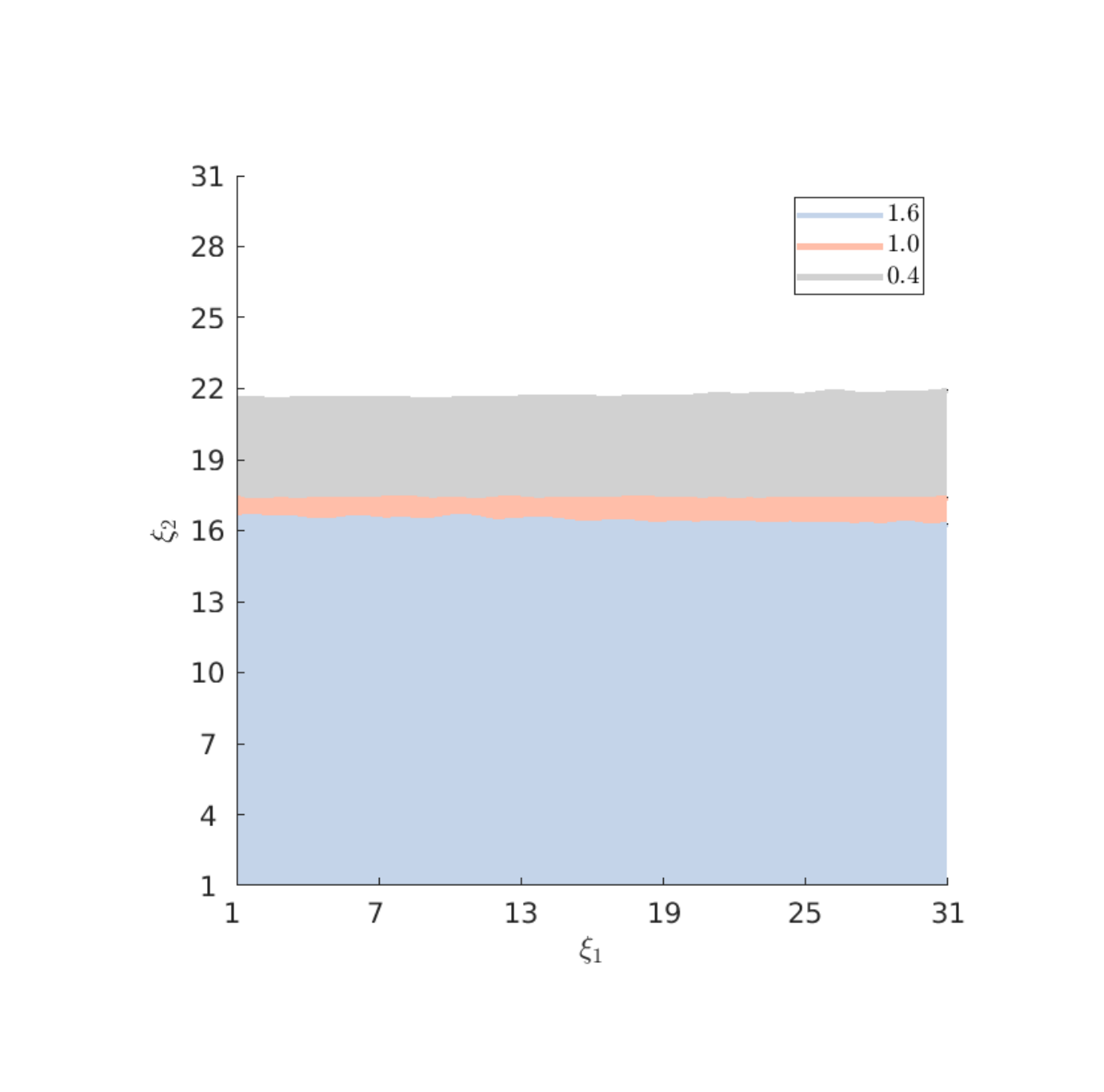}

\subcaption{$Q^b$: (CLOSED)}\label{fig:QB-closed}

\end{minipage}

\begin{minipage}{0.5\textwidth}
\centering

\includegraphics[scale=0.35,trim={3.5cm 7cm 4cm 7cm},clip]{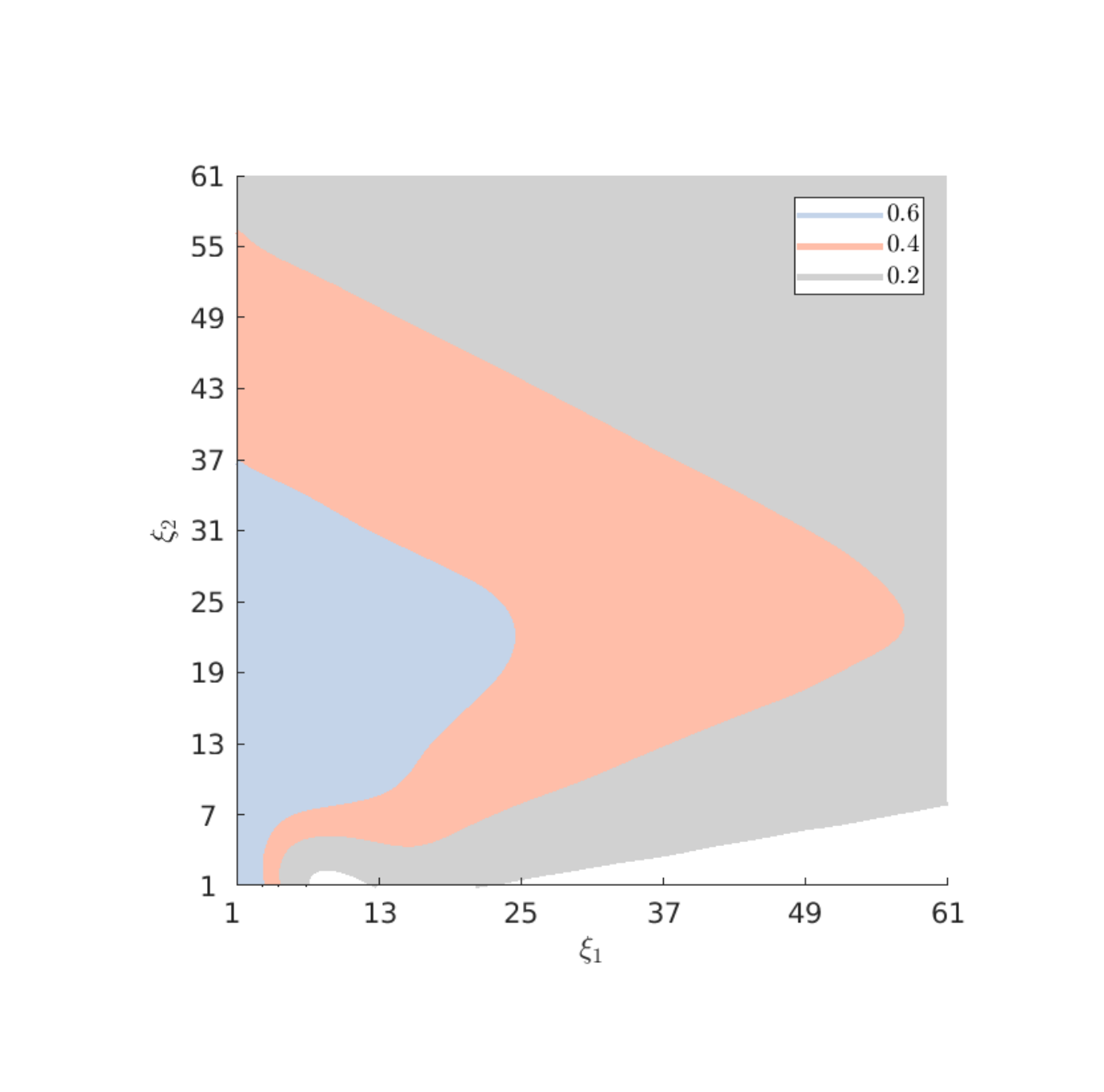}

\subcaption{$Q^a$: (OPEN)}\label{fig:QA-open}

\end{minipage}\begin{minipage}{0.5\textwidth}
\centering

\includegraphics[scale=0.35,trim={3.5cm 7cm 4cm 7cm},clip]{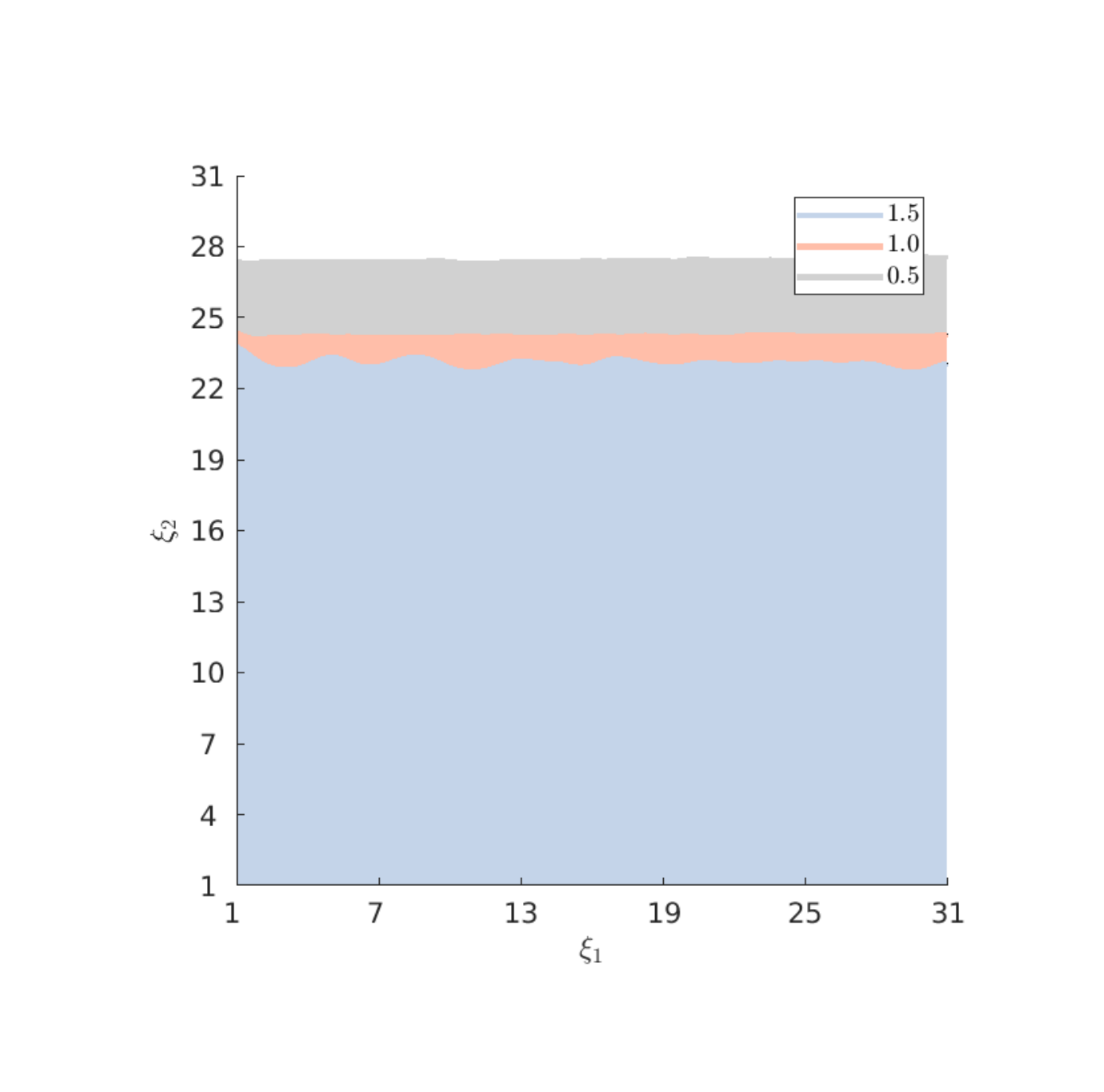}

\subcaption{$Q^b$: (OPEN)}\label{fig:QB-open}

\end{minipage}

\begin{minipage}{0.5\textwidth}
\centering

\includegraphics[scale=0.35,trim={3.5cm 7cm 4cm 7cm},clip]{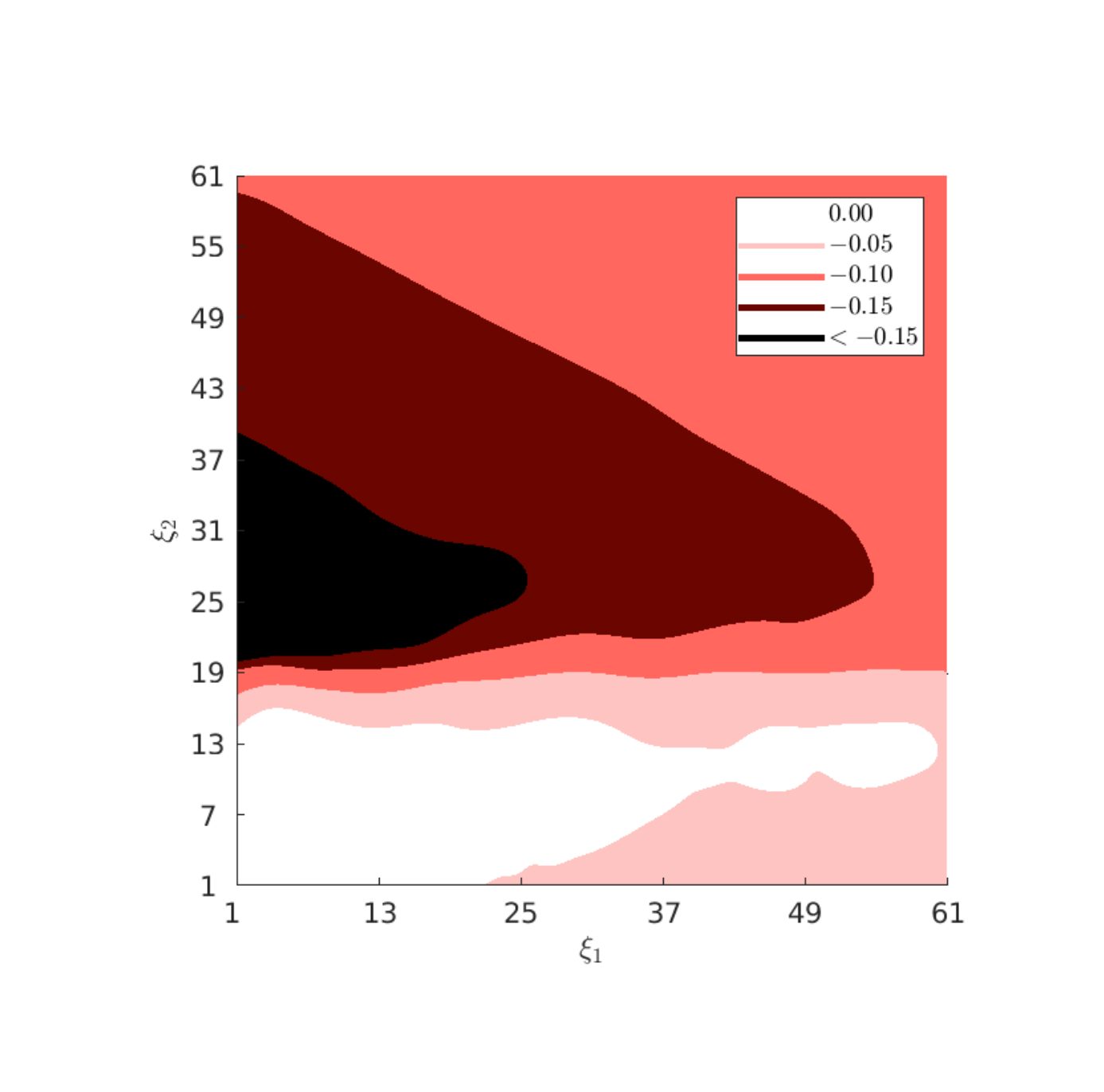}

\subcaption{$Q^a$: Difference}\label{fig:QA-diff}

\end{minipage}\begin{minipage}{0.5\textwidth}
\centering

\includegraphics[scale=0.35,trim={3.5cm 7cm 4cm 7cm},clip]{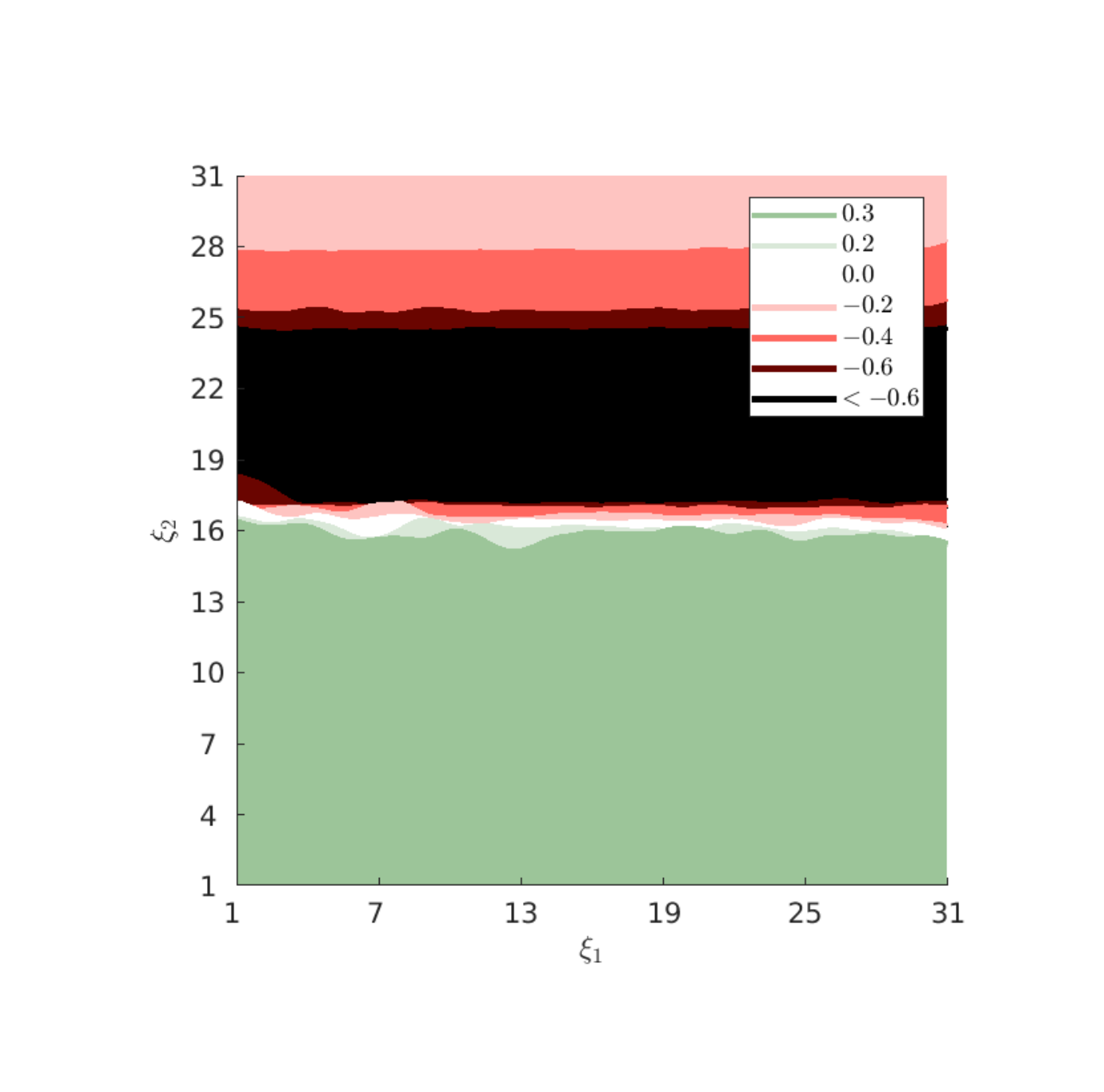}

\subcaption{$Q^b$: Difference}\label{fig:QB-diff}

\end{minipage}

\caption{Acceptable designs and configurations.}
\label{fig:resultsCaseStudyB}
\end{figure}

Figure \ref{fig:QA-closed}, Figure \ref{fig:QA-open}, and Figure \ref{fig:QA-diff} show the set of acceptable designs based on $Q^a$. We can clearly distinguish the effects of $\xi_1$ (traffic generated in the periphery) and $\xi_2$ (traffic generated in the core). 
If we increase $\xi_1$ (traffic from the periphery) and keep $\xi_2$ constant, performance deteriorates because the former traffic needs to pass through roundabouts, which are bottlenecks. Increasing $\xi_2$ while $\xi_1$ is fixed initially increases the overall performance of the system. This is because the highway has enough capacity for core traffic; increasing $\xi_2$ increases the proportion of traffic that performs well. The statistic $Q^a$ measures overall performance, and its expected value is therefore increasing. If $\xi_2$ becomes even larger, congestion will occur on the highway, again reducing performance. Opening the shoulder to traffic is advantageous when traffic density is higher, as can be seen in Figure \ref{fig:QA-diff}. However, the advantages are less pronounced when the system as a whole is too congested.

Figure \ref{fig:QB-closed}, Figure \ref{fig:QB-open}, and Figure \ref{fig:QB-diff} show the set of acceptable designs based on $Q^b$. $Q^b$ essentially measures the speed of traffic originating from the periphery that has just entered the core area after passing through a traffic circle. It can be clearly observed that the acceptable designs do not depend on $\xi_1$: The roundabouts serve as bottlenecks that control flow into the core area so that no additional congestion is caused by these traffic participants and therefore no reduction in speed. An increase in $\xi_2$, in contrast, leads to a decrease in speed; this is due to congestion at nodes 19, 20 and 32, 31. In low density regimes, traffic flows with constant free-flow speed. Once a critical density is reached, the speed decreases relatively quickly. Depending on $\xi_2$, there is a clear region where it is beneficial to open the shoulder to traffic (see Figure \ref{fig:QB-open}). This effect is again less pronounced when the system is too congested.

\section{Conclusion}\label{sec:concl}

In this work, we introduced a rigorous framework for stochastic cell transmission models for general traffic networks. The performance of traffic systems was evaluated based on preference functionals. The numerical implementation combined simulation, Gaussian process regression, and a stochastic exploration procedure. The approach was illustrated in two case studies that served as proofs of concept.

Future research should address the following tasks: a) Our flexible framework can be applied to many traffic systems. This requires careful calibration and validation at both the traffic cell level and the traffic system level. b) These models can then be used to answer specific questions in traffic planning. c) As shown in a simple example in the appendix, the setting can be extended to multiple interacting populations.  
This requires a closer look at model extensions. d) The algorithm combines stochastic search and Gaussian process regression. The latter could be replaced by other techniques, e.g., Bayesian neural networks (cf.~\textcite{Goan2020}), and the performance of different techniques should be compared. e) The normative criteria in this paper were based on expected utility. Other preference functionals might be appropriate in the face of uncertainty, for example. Their implementation requires adapted estimation procedures. 

\printbibliography

\newpage 

\appendix

\section{Further Examples of Cells}\label{sec:a-ex}

\subsection{Roundabout}

\paragraph{Unidirectional Roundabout.}

Consider an unidirectional roundabout $\#$ with four entries/exits enumerated counterclockwise and identified with $\mathbb{Z}_4=  \cI(\#) = \cO(\#) $ as shown in Figure \ref{fig:roundabout}. In right-hand traffic, vehicles travel counterclockwise through the roundabout. For simplicity, we assume that the roundabout is completely symmetric.

\begin{figure}[!htbp]
\begin{center}
\includegraphics[scale=0.8]{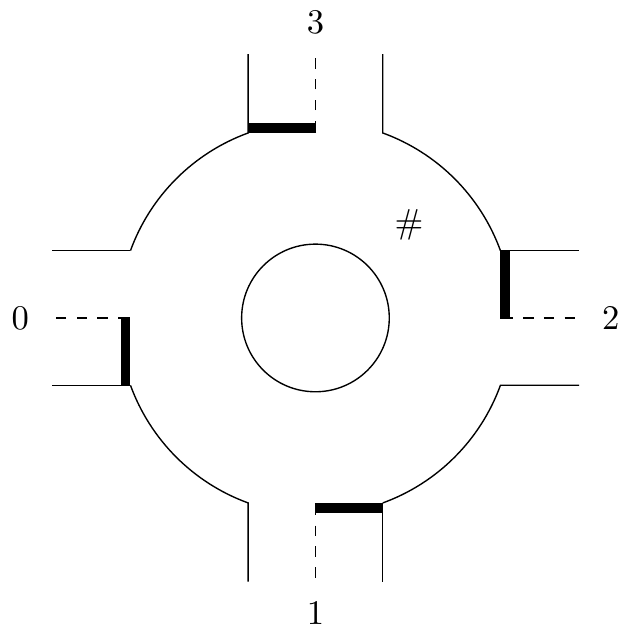}
\end{center}
\caption{Symmetric roundabout.}
\label{fig:roundabout}
\end{figure}

One possibility is to assume that sending functions are as for highways, but receiving functions capture that different paths overlap in the roundabout. This leads to sending function of the following form:\small
\begin{align*}
S_{(u,\#,w)}\left( \left(\rho_{(u',\#,w')} \right)_{u'\in\cI(\#),w'\in\cO(\#)}\right)&=\min\left(s^\mathrm{max}_\#,a\rho_{(u,\#,w)} \right),\quad w\neq u
\end{align*}\normalsize
The receiving functions have a similar shape as for bidirectional linear interfaces, but the counterdensity of vehicles traveling in opposite direction is replaced by the densities on overlapping paths. These densities must be adjusted by a factor corresponding to the length of the overlap in the roundabout. This is due to the fact that the densities are normalized for each node and are proportional to the number of vehicles on each path within each node. We assume that for each path the vehicles are uniformly distributed over the segments of the path. Under these assumptions we obtain the following receiving functions:\small
\begin{align*}
R_{(u,\#,u+1)}\left( \left(\rho_{(u',\#,w')} \right)_{u'\in\cI(\#),w'\in\cO(\#)}\right)&=\max\bigg(b\bigg(\frac{\rho^\mathrm{max}_\#}{4}-c\rho_{(u,\#,u+1)} -d\bigg(\frac{1}{2}\rho_{(u,\#,u+2)} +\frac{1}{3}\rho_{(u,\#,u+3)} \\
&+\frac{1}{3}\rho_{(u+2,\#,u+1)} +\frac{1}{2}\rho_{(u+3,\#,u+1)} +\frac{1}{3}\rho_{(u+3,\#,u+2)} \bigg) \bigg),0\bigg),\\
R_{(u,\#,u+2)}\left( \left(\rho_{(u',\#,w')} \right)_{u'\in\cI(\#),w'\in\cO(\#)}\right)&=\max\bigg(b\bigg(\frac{2\rho^\mathrm{max}_\#}{4}-c\rho_{(u,\#,u+2)} -d\bigg( \rho_{(u,\#,u+1)} +\frac{2}{3}\rho_{(u,\#,u+3)} \\
&+\rho_{(u+1,\#,u+2)} +\frac{1}{2}\rho_{(u+1,\#,u+3)} +\frac{1}{3}\rho_{(u+1,\#,u)} \\
&+\frac{1}{3}\rho_{(u+2,\#,u+1)} +\frac{1}{2}\rho_{(u+3,\#,u+1)} +\frac{2}{3}\rho_{(u+3,\#,u+2)}\bigg) \bigg),0\bigg),\\
R_{(u,\#,u+3)}\left( \left(\rho_{(u',\#,w')} \right)_{u'\in\cI(\#),w'\in\cO(\#)}\right)&=\max\bigg(b\bigg(\frac{3\rho^\mathrm{max}_\#}{4}-c \rho_{(u,\#,u+3)} -d\bigg(\rho_{(u,\#,u+1)} + \rho_{(u,\#,u+2)}  \\
&+\rho_{(u+1,\#,u+2)} +\rho_{(u+1,\#,u+3)} +\frac{2}{3}\rho_{(u+1,\#,u)} \\
&+\rho_{(u+2,\#,u+3)} +\frac{1}{2} \rho_{(u+2,\#,u)} +\frac{1}{3}\rho_{(u+2,\#,u+1)} \\
&+\frac{1}{2}\rho_{(u+3,\#,u+1)} +\frac{2}{3}\rho_{(u+3,\#,u+2)} \bigg)\bigg),0\bigg)
\end{align*}\normalsize

\paragraph{Bidirectional Roundabout.}

The roundabout model discussed above can be easily extended to bidirectional traffic flows of pedestrians in a bidirectional traffic area that has the form of a roundabout. The main difference is that the densities of traffic on overlapping paths for participants moving in the same direction and in the opposite direction must be considered in the receiving functions. In addition, pedestrians are assumed to choose the shortest path in the traffic circle. If two paths have the same length, half of the pedestrians will use the first path and the other half will use the second path. These assumptions lead to the following formalization:\small
\begin{align*}
S_{(u,\#,w)}\left( \left(\rho_{(u',\#,w')} \right)_{u'\in\cI(\#),w'\in\cO(\#)}\right)&=\min\left(s^\mathrm{max}_\#,a\rho_{(u,\#,w)} \right),\quad w\neq u,\\
R_{(u,\#,u+1)}\left( \left(\rho_{(u',\#,w')} \right)_{u'\in\cI(\#),w'\in\cO(\#)}\right)&=\max\bigg(b\bigg(\frac{\rho^\mathrm{max}_\#}{4}-c\rho_{(u,\#,u+1)} -d\bigg(\frac{1}{4}\rho_{(u,\#,u+2)} +\rho_{(u+1,\#,u)} \\
&+\frac{1}{4}\rho_{(u+1,\#,u+3)} +\frac{1}{4}\rho_{(u+2,\#,u)} +\frac{1}{4}\rho_{(u+3,\#,u+1)} \bigg)\bigg),0\bigg),\\
R_{(u,\#,u+2)}\left( \left(\rho_{(u',\#,w')} \right)_{u'\in\cI(\#),w'\in\cO(\#)}\right)&=\max\bigg(b\bigg(\rho^\mathrm{max}_\#-c\rho_{(u,\#,u+2)}-d \sum_{\substack{u'\in\cI(\#)\setminus\{u\},\\w'\in\cO(\#)\setminus\{u+2\}}} \rho_{(u',\#,w')} \bigg),0\bigg),\\
R_{(u,\#,u+3)}\left( \left(\rho_{(u',\#,w')} \right)_{u'\in\cI(\#),w'\in\cO(\#)}\right)&=\max\bigg(b\bigg(\frac{\rho^\mathrm{max}_\#}{4} -c\rho_{(u,\#,u+3)} -d\bigg(\frac{1}{4}\rho_{(u,\#,u+2)} \\
&+\frac{1}{4}\rho_{(u+1,\#,u+3)}+\frac{1}{4}\rho_{(u+2,\#,u)} +\rho_{(u+3,\#,u)} +\frac{1}{4}\rho_{(u+3,\#,u+1)} \bigg)\bigg),0\bigg)
\end{align*}\normalsize

\paragraph{Roundabout with Vehicles and Pedestrians.}

The conceptual framework we develop can be generalized to multiple populations. In this paper, for simplicity, we focus only on explaining generalized cell transmission models for traffic participants of one type. However, in the current example, we describe how an extension to more than one population is feasible.

Again, we focus on a node $\#$ with $\cI(\#)= \cO(\#)= \bbz_4$ with vehicles ($k=1$) moving as in the unidirectional roundabout. The cell transmission model can be implemented on different time scales. Here we assume that each time step corresponds to a relatively short real time span. Pedestrians ($k=2$) move in both directions, but -- according to their lower speed -- only up to the next exit. Pedestrians have priority in the roundabout. The dynamics of the pedestrians is independent of the movement of the vehicles. Vehicles move as in the unidirectional roundabaout, but can be blocked by pedestrians who have priority. This canonically leads to the following formalization, where we introduce an additional subscript for the type ($k=1,2$): \footnotesize
\begin{align*}
S_{(u,\#,w),1}&\left( \left(\rho_{(u',\#,w'),k'}  \right)_{u'\in\cI(\#),w'\in\cO(\#),k'=1,2}\right) =\quad\quad\quad\quad\quad\quad\quad\quad\quad\quad\quad\quad\quad\quad\quad\quad\quad\quad\quad\quad\quad\quad\quad\quad\quad
\\ & \quad\min\left(s^\mathrm{max}_{\#,1},a_{1}\rho_{(u,\#,w),1}\right)\cdot
\mathbbm{1}\{\rho_{(w-1,\#,w),2}+\rho_{(w,\#,w-1),2}=0\} , \; w\neq u \quad \text{ (pedestrians may block exit)},\\
S_{(u,\#,w),2} & \left( \left(\rho_{(u',\#,w'),k'} \right)_{u'\in\cI(\#),w'\in\cO(\#),k'=1,2}\right)=\\
& \min\bigg\{s^\mathrm{max}_{\#,2},a_{2}\rho_{(u,\#,w),2}\bigg\},\; w\in\{u-1,u+1\} \quad \text{ (pedestrians move independently of vehicles)},
\\
R_{(u,\#,u+1),1} & \left( \left(\rho_{(u',\#,w'),k'} \right)_{u'\in\cI(\#),w'\in\cO(\#),k'=1,2}\right)= \\
&\mathbbm{1}\{\rho_{(u,\#,u+1),2}+\rho_{(u+1,\#,u),2}=0\}\cdot\max\bigg(b_1\bigg(\frac{\rho^\mathrm{max}_{\#,1}}{4}-c_1\rho_{(u,\#,u+1),1} -d_1\bigg(\frac{1}{2}\rho_{(u,\#,u+2),1} +\frac{1}{3}\rho_{(u,\#,u+3),1} \\
&+\frac{1}{3}\rho_{(u+2,\#,u+1),1} +\frac{1}{2}\rho_{(u+3,\#,u+1),1} +\frac{1}{3}\rho_{(u+3,\#,u+2),1} \bigg) \bigg),0\bigg) \quad  \quad  \text{(pedestrians may block entrance)},\\
R_{(u,\#,u+2),1} & \left( \left(\rho_{(u',\#,w'),k'} \right)_{u'\in\cI(\#),w'\in\cO(\#),k'=1,2}\right)= \\ 
& \mathbbm{1}\{\rho_{(u,\#,u+1),2}+\rho_{(u+1,\#,u),2}=0\}  \cdot \max\bigg(b_1\bigg(\frac{2\rho^\mathrm{max}_{\#,1}}{4}-c_1\rho_{(u,\#,u+2),1} -d_1\bigg( \rho_{(u,\#,u+1),1} +\frac{2}{3}\rho_{(u,\#,u+3),1} \\
&+\rho_{(u+1,\#,u+2),1} +\frac{1}{2}\rho_{(u+1,\#,u+3),1} +\frac{1}{3}\rho_{(u+1,\#,u),1} +\frac{1}{3}\rho_{(u+2,\#,u+1),1} \\&+\frac{1}{2}\rho_{(u+3,\#,u+1),1} +\frac{2}{3}\rho_{(u+3,\#,u+2),1}\bigg) \bigg),0\bigg)
 \quad  \text{(pedestrians may block entrance)},\\
R_{(u,\#,u+3),1}& \left( \left(\rho_{(u',\#,w'),k'} \right)_{u'\in\cI(\#),w'\in\cO(\#),k'=1,2}\right)=
\\& \mathbbm{1}\{\rho_{(u,\#,u+1),2}+\rho_{(u+1,\#,u),2}=0\}\cdot \max\bigg(b_1\bigg(\frac{3\rho^\mathrm{max}_{\#,1}}{4}-c_1 \rho_{(u,\#,u+3),1} -d_1\bigg(\rho_{(u,\#,u+1),1} + \rho_{(u,\#,u+2),1}  \\
&+\rho_{(u+1,\#,u+2),1} +\rho_{(u+1,\#,u+3),1} +\frac{2}{3}\rho_{(u+1,\#,u),1} +\rho_{(u+2,\#,u+3),1} +\frac{1}{2} \rho_{(u+2,\#,u),1} +\frac{1}{3}\rho_{(u+2,\#,u+1),1} \\
&+\frac{1}{2}\rho_{(u+3,\#,u+1),1} +\frac{2}{3}\rho_{(u+3,\#,u+2),1} \bigg)\bigg),0\bigg) \quad  \text{(pedestrians may block entrance)},\\
R_{(u,\#,w),2} & \left( \left(\rho_{(u',\#,w'),k'} \right)_{u'\in\cI(\#),w'\in\cO(\#),k'=1,2}\right)= 
\max\Bigg(b_{2}\bigg(\frac{\rho^\mathrm{max}_{\#,2}}{4}-c_2\rho_{(u,\#,w),2}-d_2\rho_{(w,\#,u),2}\bigg),0\Bigg),\\ & \quad \quad \quad \quad\quad \quad \quad\quad \quad\quad \quad \quad\quad  w\in\{u-1,u+1\}, \quad \text{(pedestrians move independently of vehicles)}
\end{align*}\normalsize
Here, for types $k=1,2$, $\rho^\mathrm{max}_{\#, k}>0$ is the maximum density, $s^\mathrm{max}_{\#,k}>0$ is the maximum flow, $0<a_k\leq 1$ is the free-flow speed, $0<b_k\leq 1$ is the congestion wave speed, and $c_k, d_k >0$ are interaction parameters.

\section{The Algorithm}

\subsection{A Bayesian Approach to Sampling}\label{sec:BayesianSampling}

In most applications, generating the samples of the complex system $Q_k$ is expensive.  A Bayesian approach exploits the previous GPR $(m^i,\sigma^i)$ in order to do variance reduction and, thus, reduce computational costs.  Specifically, we propose utilizing the GPR (with its normal distribution) as a prior for the estimation of $\hat{\mu}_k$.
With this structure, we consider Bayesian inference for normal mean conditional on the variance (see, e.g., \textcite[Section 5.2]{Hoff2009}) with a single sample taken at a time. Mathematically, we have the following structure:
\begin{itemize}
\item \emph{Prior distribution}: $\mu(k)\sim\mathcal{N}(m^i(k),\sigma^i(k)^2)$,
\item \emph{Sampling distribution}: $\hat{\mu}_k^n\mid \mu(k)\sim \mathcal{N}\left(\mu(k),\frac{(\hat{\sigma}_k^n)^2}{n}\right)$ as a central limit theorem heuristic by prior assumptions, and
\item \emph{Posterior distribution}: $\mu(k)\mid \hat{\mu}_k^n\sim \mathcal{N}\left(t_{\mathrm{post},n},s_{\mathrm{post},n}^2\right)$
where 
\begin{align*}
t_{\mathrm{post},n}=\frac{  \frac{1}{\left(\sigma^i(k)\right)^2}  m^i(k) + \frac{n}{(\hat{\sigma}_k^n)^2 } \hat{\mu}_k^n }{ \frac{1}{\left(\sigma^i(k)\right)^2} + \frac{n}{(\hat{\sigma}_k^n)^2 } } \quad\text{and}\quad s_{\mathrm{post},n}^2=\frac{1}{\frac{1}{\left(\sigma^i(k)\right)^2}+\frac{n}{(\hat{\sigma}_k^n)^2}}.
\end{align*}
\end{itemize}
Note that $s_{\mathrm{post},n}^2\leq \frac{(\hat{\sigma}_k^n)^2}{n}$, i.e., the variance (and corresponding sample size) is reduced compared to the purely frequentist view described in Section \ref{sec:leaning-alg}. At the same time, precision may be reduced as $t_\mathrm{post}$ in general may not be an unbiased estimator of $\mathbb{E}(u(Q_k))$. The updated stopping criterion in the Bayesian approach, i.e., so that the sample variance drops below $(\tau^i)^2$, is given by
\begin{equation*}
n=\min\left\{\min\left\{n_\mathrm{min}\leq \bar{n} \colon s_{\mathrm{post},\bar{n}}^2\leq (\tau^i)^2\right\},n_\mathrm{max}\right\}.
\end{equation*}

\subsection{Computing the Error Bounds}\label{sec:computeErrorBounds}
We wish to return to our discussion of the error bounds with some remarks on its computation with Monte Carlo estimation.
As shown in Section \ref{sec:learning-estimate-error}, the estimation error can be upper bounded by integrals of the form
\begin{equation*}
\mathrm{vol}\left\{k\in\mathbb{D}\colon m^i_+(k)\geq \gamma > m^i_-(k)\right\},
\end{equation*}
where $m^i_-, m^i_+\colon\mathbb{D}\to\mathbb{R}$ are constructed as either uniform or pointwise error bounds. 

The computation of these integrals is not trivial as the functions $m^i_-$ and $m^i_+$ are typically not analytically accessible. Yet, values at specific positions $k\in\mathbb{D}$ can be evaluated. This provides a natural setting for approximation via Monte Carlo simulation:
\begin{itemize}
\item For a fixed budget $n_\mathrm{eval}\in\mathbb{N}$, let $U_1,\dots,U_{n_\mathrm{eval}}\sim \mathrm{Unif}(\mathbb{D})$.
\item An approximation is given by
\begin{align*}
\mathrm{vol}\left\{k\in\mathbb{D}\colon m^i_+(k)\geq \gamma > m^i_-(k)\right\}&=\int_\mathbb{D} \mathbbm{1}\{m^i_+(k)\geq \gamma > m^i_-(k)\}\mathrm{d}k\\
&\approx \frac{\mathrm{vol}(\mathbb{D})}{n_\mathrm{eval}}\sum_{j=1}^{n_\mathrm{eval}} \mathbbm{1}\left\{m^i_+(U_j)\geq \gamma > m^i_-(U_j)\right\}
\end{align*}
\item In order to eliminate random fluctuations in the comparison of the error bound for different iterations $i$, we fix a particular sequence\footnote{This is a natural application for Quasi-Monte Carlo methods in order to decrease the approximation error; here, we use the \emph{Sobol sequence}. We refer to \textcite{Glasserman2003} for an overview on Quasi-Monte Carlo methods.} of samples $\hat{U}_1,\dots,\hat{U}_{n_\mathrm{eval}}$.
\end{itemize}
The details of the evaluation procedure with the Monte Carlo approximation are given in Algorithm \ref{alg:PEval}.
\begin{algorithm}[!htp]
\caption{Evaluation Procedure for the Approximation Error.}
\label{alg:PEval}
\begin{algorithmic}
\STATE{ \textbf{Input:} $n_\mathrm{eval}\in\mathbb{N}$, $m^i_+,m^i_-\colon\mathbb{D}\to\mathbb{R}$, samples $\hat{U}_1,\dots,\hat{U}_{n_\mathrm{eval}}\in\mathbb{D}$.}
\STATE{Compute
\begin{equation*}
\hat{e}^i=\frac{\mathrm{vol}(\mathbb{D})}{n_\mathrm{eval}}\sum_{j=1}^{n_\mathrm{eval}} \mathbbm{1}\left\{m^i_+(\hat{U}_j)\geq \gamma > m^i_-(\hat{U}_j)\right\}.
\end{equation*}}
\STATE{\textbf{Output:} $\hat{e}^i$.}
\end{algorithmic}
\end{algorithm}

\subsection{Robustification of the Pointwise Error Bound}\label{sec:interpret-pointwise}

The pointwise error bound can be extended \emph{locally} if we impose a Lipschitz assumption on $M$. This can be understood as a robustification of the pointwise bounds. We present the following statement in analogy to the uniform error bounds by \textcite{Lederer2019}.

\begin{proposition}[Local Credible Band]\label{thm:locCredBand}
Let $\delta\in(0,1)$ and $\varepsilon>0$. For fixed $k^*\in\mathbb{D}$ with $B_\varepsilon(k^*)\subseteq \mathbb{D}$, let $L=L(k^*)>0$ and assume that $|M(k^*)-M(k)|\leq L\|k^*-k\|$ $P$-a.s. for all $k\in B_\varepsilon(k^*)$. It holds
\begin{align*}
P\left(\forall~k\in B_\varepsilon(k^*)\colon |M(k)-m^i(k^*)|\leq \Phi^{-1}\left(1-\frac{\delta}{2}\right)\sigma^i(k^*)+L\|k-k^*\|\mid \hat{M}(\mathbb{D}^i)=\hat{\mu}(\mathbb{D}^i)\right)\geq 1-\delta
\end{align*}
\end{proposition}

\begin{proof}
Let $k^*\in\mathbb{D}$ be fixed. Applying the pointwise approximation error for $k^*$, we have
\begin{equation*}
P\left(|M(k^*)-m^i(k^*)|\leq \Phi^{-1}\left(1-\frac{\delta}{2}\right)\sigma^i(k^*)\mid \hat{M}(\mathbb{D}^i)=\hat{\mu}(\mathbb{D}^i)\right)\geq 1-\delta.
\end{equation*}
For all $k\in\mathbb{D}$, triangular inequality and Lipschitz assumption imply
\begin{align*}
|M(k)-m^i(k^*)|&\leq |M(k)-M(k^*)|+|M(k^*)-m^i(k^*)|\leq L\|k-k^*\|+|M(k^*)-m^i(k^*)|.
\end{align*}
We conclude
\begin{align*}
P\left(\forall~k\in B_\varepsilon(k^*)\colon |M(k)-m^i(k^*)|\leq \Phi^{-1}\left(1-\frac{\delta}{2}\right)\sigma^i(k^*)+L\|k-k^*\|\mid \hat{M}(\mathbb{D}^i)=\hat{\mu}(\mathbb{D}^i)\right)\geq 1-\delta
\end{align*}
\end{proof}

The local credible band gives rise to a local sandwich principle; we can upper bound the approximation error $\hat{\mathcal{D}}^i$ locally by intersecting it with $B_\varepsilon(k^*):=\{k\in\mathbb{D}\colon \|k-k^*\|<\varepsilon\}$. 

\begin{corollary}[Local Credible Band for the Acceptable Design and Error Bound]\label{cor:error-local}
In the setting of Proposition \ref{thm:locCredBand}, let 
$m^i_{\pm,k^*}(k):=m^i(k^*)\pm \Phi^{-1}\left(1-\frac{\delta}{2}\right)\sigma^i(k^*)\pm L\|k-k^*\|$ and define the estimators $\hat{\mathcal{D}}^i=\{k\in\mathbb{D}\colon m^i(k)\geq \gamma\}$, $\hat{\mathcal{D}}^i_{-,k^*}=\{k\in\mathbb{D}\colon m^i_{-,k^*}(k)\geq \gamma\}$, and $\hat{\mathcal{D}}^i_{+,k^*}=\{k\in\mathbb{D}\colon m^i_{+,k^*}(k)\geq \gamma\}$. Let $\mathfrak{D}=\{k\in\mathbb{D}\colon M(k)\geq \gamma\}$ be the corresponding prior for $\mathcal{D}$. Then, for all $k^*\in\mathbb{D}$, it holds
\begin{align*}
P(\hat{\mathcal{D}}^i_{-,k^*}\cap B_\varepsilon(k^*)\subseteq \mathfrak{D}\cap B_\varepsilon(k^*)\subseteq \hat{\mathcal{D}}^i_{+,k^*}\cap B_\varepsilon(k^*)\mid \hat{M}(\mathbb{D}^i)=\hat{\mu}(\mathbb{D}^i))\geq 1-\delta
\end{align*}
and
\begin{align*}
P\bigg(d_N(\mathfrak{D}\cap B_\varepsilon(k^*), \hat{\mathcal{D}}^i\cap B_\varepsilon(k^*))&\leq d_N\left((\hat{\mathcal{D}}^i_{+,k^*}\cap B_\varepsilon(k^*),~ \hat{\mathcal{D}}^i_{-,k^*})\cap B_\varepsilon(k^*)\right)\\
&\mid \hat{M}(\mathbb{D}^i)=\hat{\mu}(\mathbb{D}^i)\bigg)\geq 1-\delta.
\end{align*}
\end{corollary}

\begin{proof}
Clear.
\end{proof}

The preceding statement tells us that widening the pointwise lower and upper approximations \emph{locally} allows to probabilistically bound the approximation error of the acceptable design \emph{locally}. In analogy to the uniform error bounds by \textcite{Lederer2019}, this requires an additional Lipschitz assumption with a (local) Lipschitz constant $L$ which, in practice, is typically unknown.

\subsection{Proofs}

\subsubsection{Proof of Lemma \ref{thm:errorBound}}\label{proof:thm:errorBound}

We compute\footnotesize
\begin{align*}
\mathrm{vol}\left(\hat{\mathcal{D}}^i\Delta \mathcal{D}\right)&=\mathrm{vol}\left(\left(\hat{\mathcal{D}}^i\setminus \mathcal{D}\right)\cup \left(\mathcal{D}\setminus \hat{\mathcal{D}}^i\right)\right) =\mathrm{vol}\left(\hat{\mathcal{D}}^i\setminus \mathcal{D}\right)+\mathrm{vol}\left(\mathcal{D}\setminus \hat{\mathcal{D}}^i\right)\\
&\leq \mathrm{vol}\left(\hat{\mathcal{D}}^i\setminus \hat{\mathcal{D}}^i_-\right)+\mathrm{vol}\left(\hat{\mathcal{D}}^i_+\setminus \hat{\mathcal{D}}^i\right) =\mathrm{vol}\left(\hat{\mathcal{D}}^i\setminus \hat{\mathcal{D}}^i_-\right)+\mathrm{vol}\left(\hat{\mathcal{D}}^i_+\setminus \hat{\mathcal{D}}^i_-\right)-\mathrm{vol}\left(\hat{\mathcal{D}}^i\setminus \hat{\mathcal{D}}^i_-\right)\\
&=\mathrm{vol}\left(\hat{\mathcal{D}}^i_+ \setminus \hat{\mathcal{D}}^i_-\right).
\end{align*}\normalsize
Depending on the type of inclusion, also the inequality is guaranteed $P$-a.s. or with probability greater than $1-\delta$.

\subsubsection{Proof of Corollary \ref{cor:error}}\label{proof:cor:error}

It is clear that, for all $k\in\mathbb{D}$, $m^i_-(k)\leq m^i(k)\leq m^i_+(k)$. This implies the inclusion of the corresponding set estimators. More precisely, we have
\begin{align*}
\hat{\mathcal{D}}^i_-&=\{k\in\mathbb{D}\colon m^i_-(k)\geq \gamma\}\subseteq \{k\in\mathbb{D}\colon m^i(k)\geq \gamma\}=\hat{\mathcal{D}}^i\subseteq \{k\in\mathbb{D}\colon m^i_+(k)\geq \gamma\}=\hat{\mathcal{D}}^i_+.
\end{align*}
Correspondingly, $P\left(\forall~k\in\mathbb{D}\colon m^i_-(k)\leq M(k)\leq m^i_+(k)\right)\geq 1-\delta$ implies
\begin{equation*}
P(\hat{\mathcal{D}}^i_-\subseteq \mathfrak{D}\subseteq \hat{\mathcal{D}}^i_+)\geq 1-\delta.
\end{equation*}
Thus, Theorem \ref{thm:errorBound} yields the claimed error bound $P(\mathrm{vol}(\mathfrak{D}\Delta \hat{\mathcal{D}}^i)\leq \mathrm{vol}(\hat{\mathcal{D}}^i_+\Delta \hat{\mathcal{D}}^i_-))\geq 1-\delta$ and, due to the inclusion $\hat{\mathcal{D}}^i_-\subseteq \hat{\mathcal{D}}^i_+$, it follows
\begin{align*}
\mathrm{vol}(\hat{\mathcal{D}}^i_+\Delta \hat{\mathcal{D}}^i_-)&=\mathrm{vol}(\hat{\mathcal{D}}^i_+\setminus \hat{\mathcal{D}}^i_-)\\
&=\mathrm{vol}\left(\{k\in\mathbb{D}\colon m^i_+(k)\geq \gamma\}\setminus \{k\in\mathbb{D}\colon m^i_-(k)< \gamma\}\right)\\
&=\mathrm{vol}\left\{k\in\mathbb{D}\colon m^i_+(k)\geq \gamma > m^i_-(k)\right\}.
\end{align*}

\subsubsection{Proof of Lemma \ref{lem:pointwiseApprox}}\label{proof:lem:pointwiseApprox}

The GPR based on observed data $\hat{\mu}(\mathbb{D}^i)$ yields
\begin{equation*}
\forall~k\in\mathbb{D}\colon M(k)\mid \hat{M}(\mathbb{D}^i)=\hat{\mu}(\mathbb{D}^i)\sim\mathcal{N}\left(m^i(k),(\sigma^i(k))^2\right).
\end{equation*}
The pointwise approximation error directly follows from standard confidence intervals for the mean of the normal distribution $\mathcal{N}\left(m^i(k),(\sigma^i(k))^2\right)$.

\subsection{Algorithms}

Algorithm \ref{alg:PEstInitial} includes a pre-processing of the data: Subtracting the sample mean resembles the prior assumption $m\equiv 0$ (see also \textcite{Schulz2018}); the additional standardization of the data by its sample standard deviation serves to circumvent numerical issues. In our implementation, we use Matlab's built-in optimization routine. In case of numerical issues, we restart with a random initial point or reduce the number of points considered in the log likelihood. 

\begin{algorithm}[!htbp]
\caption{Pre-Processing and Estimation of Hyperparameters.}
\label{alg:PEstInitial}
\begin{algorithmic}
\STATE{\textbf{Input:}
\begin{itemize}
\item Noisy data $(k,\hat{\mu}_k)_{k\in\mathbb{D}^0}$ such that $\hat{\mu}_k=\mu(k)+\varepsilon_k$ with $\varepsilon_k\sim \mathcal{N}(0,\tau_k^2)$ independent,
\item prior mean $m\equiv 0$,
\item prior covariance function $c\colon\mathbb{D}\times\mathbb{D}\to[0,\infty)$ depending on hyperparameters $\sigma_c,l>0$.
\end{itemize}}
\STATE{\textbf{Pre-Process Data:} Let $\bar{\mu}^0=1/|\mathbb{D}^0| \sum_{k\in\mathbb{D}^0}\hat{\mu}_k$, $\bar{\varsigma}^0=\sqrt{1/(|\mathbb{D}^0|-1)\sum_{k\in\mathbb{D}^0}(\hat{\mu}_k-\bar{\mu}^0)^2}$  and define $\hat{\nu}_k=(\hat{\mu}_k-\bar{\mu}^0)/\bar{\varsigma}^0$, $k\in\mathbb{D}^0$.}
\STATE{\textbf{Model Selection:} Determine $(\hat{\sigma}_c,\hat{l})$ by maximizing the (log) marginal likelihood
\begin{multline*}
\ell(\hat{\nu}(\mathbb{D}^0);\sigma_c,l)=-\frac{1}{2}\hat{\nu}(\mathbb{D}^0)^\top \left( \Sigma(\mathbb{D}^0,\mathbb{D}^0) + \operatorname{diag}\left(\tau_1^2,\dots,\tau_{|\mathbb{D}^0|}^2\right)\right)^{-1}\hat{\nu}(\mathbb{D}^0)\\
-\frac{1}{2}\det\left(\Sigma(\mathbb{D}^0,\mathbb{D}^0) + \operatorname{diag}\left(\tau_1^2,\dots,\tau_{|\mathbb{D}^0|}^2\right)\right)-\frac{|\mathbb{D}^0|}{2}\log(2\pi).
\end{multline*}}
\STATE{\textbf{Output:} $\bar{\mu}^0$, $\bar{\varsigma}^0$, $\hat{\sigma}_c$, $\hat{l}$.}
\end{algorithmic}
\end{algorithm}

\begin{algorithm}[!htbp]
\caption{Estimation Procedure with Gaussian Process Regression.}
\label{alg:PEstGPR}
\begin{algorithmic}
\STATE{\textbf{Input:}
\begin{itemize}
\item Noisy data $(k,\hat{\mu}_k)_{k\in\mathbb{D}^i}$ such that $\hat{\mu}_k=\mu(k)+\varepsilon_k$ with $\varepsilon_k\sim \mathcal{N}(0,\tau_k^2)$ independent,
\item $\bar{\mu}^0$, $\bar{\varsigma}^0$, $\hat{\sigma}_c$, $\hat{l}$ from Algorithm \ref{alg:PEstInitial}
\end{itemize}}
\STATE{\textbf{Transformation:} Define $\hat{\nu}_k=(\hat{\mu}_k-\bar{\mu}^0)/\bar{\varsigma}^0$, $k\in\mathbb{D}^i$.}
\STATE{\textbf{Bayesian Update:} Based on $(\hat{\sigma}_c,\hat{l})$ and the data $(k,\hat{\nu}_k)_{k\in\mathbb{D}^i}$, compute $m_\nu^i\colon\mathbb{D}\to\mathbb{R}$ and $\sigma^i_\nu\colon\mathbb{D}\to[0,\infty)$ given by $\sigma_\nu^i(k)=\sqrt{c_\nu(k,k)}$ according to Theorem \ref{thm:GPR}.}
\STATE{\textbf{Retransformation:} Define $m^i\colon\mathbb{D}\to\mathbb{R}$, $m^i(k):=m_\nu^i(k)\bar{\varsigma}^0+\bar{\mu}^0$ and $\sigma^i\colon\mathbb{D}\to[0,\infty)$, $\sigma^i(k):=\sigma_\nu^i(k)\bar{\varsigma}^0$.}
\STATE{\textbf{Output:} $\hat{\mathcal{D}}^i=\{k\in\mathbb{D}\colon m^i(k)\geq \gamma\}$, $m^i$, $\sigma^i$.}
\end{algorithmic}
\end{algorithm}

\section{Supplement to the Case Studies}

\subsection{Companion to Section~\ref{sec:casestudyI}: Traffic Simulation}\label{app:TrafficSimulation}

The implementations of our traffic models adhere to the structure of the following pseudo-code.
\begin{algorithm}[!htp]
\caption{Basic Traffic Simulation}
\label{alg:TrafficSim}
\begin{algorithmic}
\STATE{\textbf{Input:}
\begin{itemize}
\item Adjacency matrix: $A^E\in\{0,1\}^{|V|\times|V|}$ for a set of enumerated nodes $V=\{1,\dots,|V|\}$.
\item Initial traffic configuration: $\rho_{(u,v,w)}(0)\geq 0 $ for all $v\in V$, $u\in \cI(v)$, $u\neq w\in\cO(v)$
\item Parameters (including terminal time: $T\in\mathbb{N}$)
\end{itemize}}
\FOR{$t=0,\dots,T-1$}
\STATE{\textbf{Phase 1: Compute sending and receiving constraints.}}
\FOR{$v\in V$}
\FOR{$u\in\cI(v)$ and $u\neq w\in\cO(v)$}
\STATE{Compute sending function $S_{(u, v, w)}  \left((\rho_{(u', v, w'), k'} (t))_{u'\in \cI (v), w'\in \cO(v) } \right)$ and receiving function $R_{(u,v,w)} \left((\rho_{(u', v, w')} (t))_{u'\in \cI (v), w'\in \cO(v) } \right)$.}
\ENDFOR
\ENDFOR
\STATE{\textbf{Phase 2: Compute outflows.}}
\FOR{$u\in V$}
\FOR{$x\in\cI(u)$ and $x\neq v\in\cO(u)$}
\STATE{Compute $q_{(x,u,v)}^\mathrm{out}(t+1)$.}
\ENDFOR
\ENDFOR
\STATE{\textbf{Phase 3: Compute inflows.}}
\FOR{$v\in V$}
\FOR{$u\in\cI(v)$ and $u\neq w\in\cO(v)$}
\STATE{Compute $q_{(u,v,w)}^\mathrm{in}(t+1)=\sum_{x\in \cI(u)} f_{(x,u, v) \to w}(t+1) \cdot q_{(x,u, v)}^\mathrm{out}(t+1)$.}
\ENDFOR
\ENDFOR
\STATE{\textbf{Phase 4: Compute source/sink flows.}}
\FOR{$v\in V$}
\FOR{$u\in\cI(v)$ and $u\neq w\in\cO(v)$}
\STATE{Compute $q_{(u,v,w)}^\mathrm{net}(t+1)$.}
\ENDFOR
\ENDFOR
\STATE{\textbf{Phase 5: Update densities.}}
\FOR{$v\in V$}
\FOR{$u\in\cI(v)$ and $u\neq w\in\cO(v)$}
\STATE{Compute $\rho_{(u,v,w)}(t+1)=\rho_{(u,v,w)}(t)+q^\mathrm{in}_{(u,v,w)}(t+1)-q^\mathrm{out}_{(u,v,w)}(t+1)+q^\mathrm{net}_{(u,v,w)}(t+1)$.}
\ENDFOR
\ENDFOR
\ENDFOR
\end{algorithmic}
\end{algorithm}

\subsection{Urban Network}

\subsubsection{Companion to Section~\ref{sec:casestudyI}: Traffic Light Implementation}\label{sec:trafficLightImplementation}

Let $v\in\mathcal{R}=\{14,16\}$. For any $u\in \cI(v)$ and $u\neq w\in\cO(v)$, let $LA_{(u,v,w)}\in[0,1]$ model the traffic light adjustment for traffic users with traveling direction $(u,v,w)$ which is based on the respective traffic light signal $LS_{(u,v,w)}\in\{0,1\}$. 

In the following, we identify $\{13,20,15,9\}$ and $\{15,21,17,10\}$ with $\mathbb{Z}_4$ and set
\begin{align*}
S_{(u,v,u+1)}\left( \left(\rho_{(u',v,w')} \right)_{u'\in\cI(v),w'\in\cO(v)},LA_{(u,v,u+1)}\right)&=\min\left\{s^\mathrm{max}_v,a_vLA_{(u,v,u+1)}\rho_{(u,v,u+1)} \right\},\\
S_{(u,v,u+2)}\left( \left(\rho_{(u',v,w')} \right)_{u'\in\cI(v),w'\in\cO(v)},LA_{(u,v,u+2)}\right)&=\min\left\{s^\mathrm{max}_v,a_vLA_{(u,v,u+2)}\rho_{(u,v,u+2)} \right\},\\
S_{(u,v,u+3)}\left( \left(\rho_{(u',v,w')} \right)_{u'\in\cI(v),w'\in\cO(v)},LA_{(u,v,u+3)}\right)&=
\min\Bigg\{s^\mathrm{max}_v,a_vLA_{(u,v,u+3)}\rho_{(u,v,u+3)}\cdot\\
&\exp\left( -\zeta_v\left(\rho_{(u+2,v,u)} +\rho_{(u+2,v,u+3)}\right) \right)\Bigg\},\\
R_{(u,v,w)}\left( \left(\rho_{(u',v,w')}\right)_{u'\in\cI(v),w'\in\cO(v)}\right)&=\max\left( b_v \left(\frac{\rho^\mathrm{max}_v}{4}-\sum_{w'\in\cO(v)} \rho_{(u,v,w')} \right),0\right).
\end{align*}

We implement the signal policy as follows. Let $T^g\in\mathbb{N}$ be the duration of the green phase and $T^s\in\mathbb{N}$ the shift between the green times of the two traffic lights. Let $\mathcal{I}_{14}=\{13,15\}$, $\mathcal{J}_{14}=\{9,20\}$ and $\mathcal{I}_{16}=\{15,17\}$, $\mathcal{J}_{16}=\{10,21\}$. We set
\begin{equation*}
LS_{(u,14,w)}(t)=\begin{cases}
1,& t\mod 2T^g\in\{0,1,\dots,T^g-1\},~u\in \mathcal{I}_{14},w\neq u,\\
0,& t\mod 2T^g\in\{0,1,\dots,T^g-1\},~u\in \mathcal{J}_{14},w\neq u,\\
0,& t\mod 2T^g\in\{T^g,T^g+1,\dots,2T^g-1\},~u\in \mathcal{I}_{14},w\neq u,\\
1,& t\mod 2T^g\in\{T^g,T^g+1,\dots,2T^g-1\},~u\in \mathcal{J}_{14},w\neq u.
\end{cases}
\end{equation*}
and
\begin{equation*}
LS_{(u,16,w)}(t)=\begin{cases}
1,& t+T^s\mod 2T^g\in\{0,1,\dots,T^g-1\},~u\in \mathcal{I}_{16},w\neq u,\\
0,& t+T^s\mod 2T^g\in\{0,1,\dots,T^g-1\},~u\in \mathcal{J}_{16},w\neq u,\\
0,& t+T^s\mod 2T^g\in\{T^g,T^g+1,\dots,2T^g-1\},~u\in \mathcal{I}_{16},w\neq u,\\
1,& t+T^s\mod 2T^g\in\{T^g,T^g+1,\dots,2T^g-1\},~u\in \mathcal{J}_{16},w\neq u.
\end{cases}
\end{equation*}
We assume that vehicles accelerate comfortably with $a^\mathrm{real}=\unit[1.5]{m/s^2}$ when a traffic lights switches from red to green. We introduce $t^\mathrm{safe}=2$ to model the acceleration delay caused by safety and reaction time and set
\begin{equation*}
LA_{(u,v,w)}(t)=\max\left\{0,\min\left\{1, \left(t^\mathrm{switch}_{(u,v,w)}(t)-t^\mathrm{safe}\right)\cdot t^\mathrm{real}\cdot \frac{\mathrm{a^\mathrm{real}}}{v^\mathrm{real}}\right\}\right\}\cdot LS_{(u,v,w)}(t)
\end{equation*}
with the time intervals since the last switch being
\begin{equation*}
t^\mathrm{switch}_{(u,v,w)}(t)=\inf\{s\in\mathbb{N}\colon LS_{(u,v,w)}(t)\neq LS_{(u,v,w)}(t-s)\}.
\end{equation*}

\subsubsection{Companion to Section~\ref{sec:casestudyI}: Net Flows}\label{sec:highwayNetworkDetailsA}

Net flows are modelled as follows: For $(u,v,w)\in \{(6,7,11), (24,23,19)\}$, we define autoregressive models of order 1:\footnotesize
\begin{align*}
q^\mathrm{ar}_{(u,v,w)}(t+1)= q^\mathrm{ar}_{(u,v,w)}(t)+\varepsilon_{(u,v,w)}(t+1)
\end{align*}\normalsize
where $\varepsilon_{(u,v,w)}(t+1)\sim\mathcal{N}\left(0,\sigma_{(u,v,w)}\right)$, $\sigma_{(u,v,w)}^2\geq 0$, and innovations are stochastically independent across time. The initial value is $q^\mathrm{ar}_{(u,v,w)}(0)=0$. The dependence of $\varepsilon_{(6,7,11)}(t+1)$ and $\varepsilon_{(24,23,19)}(t+1)$ is governed by a Frank copula, parametrized by $r\in\mathbb{R}\setminus\{0\}$. In order to respect non-negativity constraints and maximal densities, we set \footnotesize
\begin{align*}
q^\mathrm{net}_{(u,v,w)}(t+1)=\min\bigg(& \max\big(q^\mathrm{ar}_{(u,v,w)}(t+1),~ q^\mathrm{out}_{(u,v,w)}(t+1)-q^\mathrm{in}_{(u,v,w)}(t+1)-\rho_{(u,v,w)}(t)\big),\\
& \rho^\mathrm{max}_v +q^\mathrm{out}_{(u,v,w)}(t+1)-q^\mathrm{in}_{(u,v,w)}(t+1)-\rho_{(u,v,w)}(t) \bigg).
\end{align*}\normalsize

\subsubsection{Companion to Section~\ref{sec:casestudyI}:  Dependence Parameter and Noise}\label{app:caseStudy1Comp}

To study the effects of the random environment, we set $T^g=10$ and $T^s=0$ and run simulations\footnote{GRP is based on the Matérn kernel. We also compared this to the squared exponential kernel which leads to almost the same results.} in the three-dimensional subset  $\left\{\left(r,\sigma_{(6,7,11)},\sigma_{(24,23,19)},10,0\right)\in\mathbb{D}\right\}$. Figure \ref{fig:network1AcceptableDependence} shows acceptable designs in terms of dependence structure and noise (where we set $\sigma_{(6,7,11)}=\sigma_{(24,23,19)}$) for the considered utility functions. The impact of the dependence parameter $r$ is small. In Figure \ref{fig:network1AcceptableNoise}, we set $r=0$, corresponding to independent noise at the sources. The system performance decreases with increasing noise present in the system.

\begin{figure}[!htbp]

\begin{minipage}{\textwidth}
\centering

\includegraphics[scale=0.4,trim={4.5cm 8cm 5cm 8cm},clip]{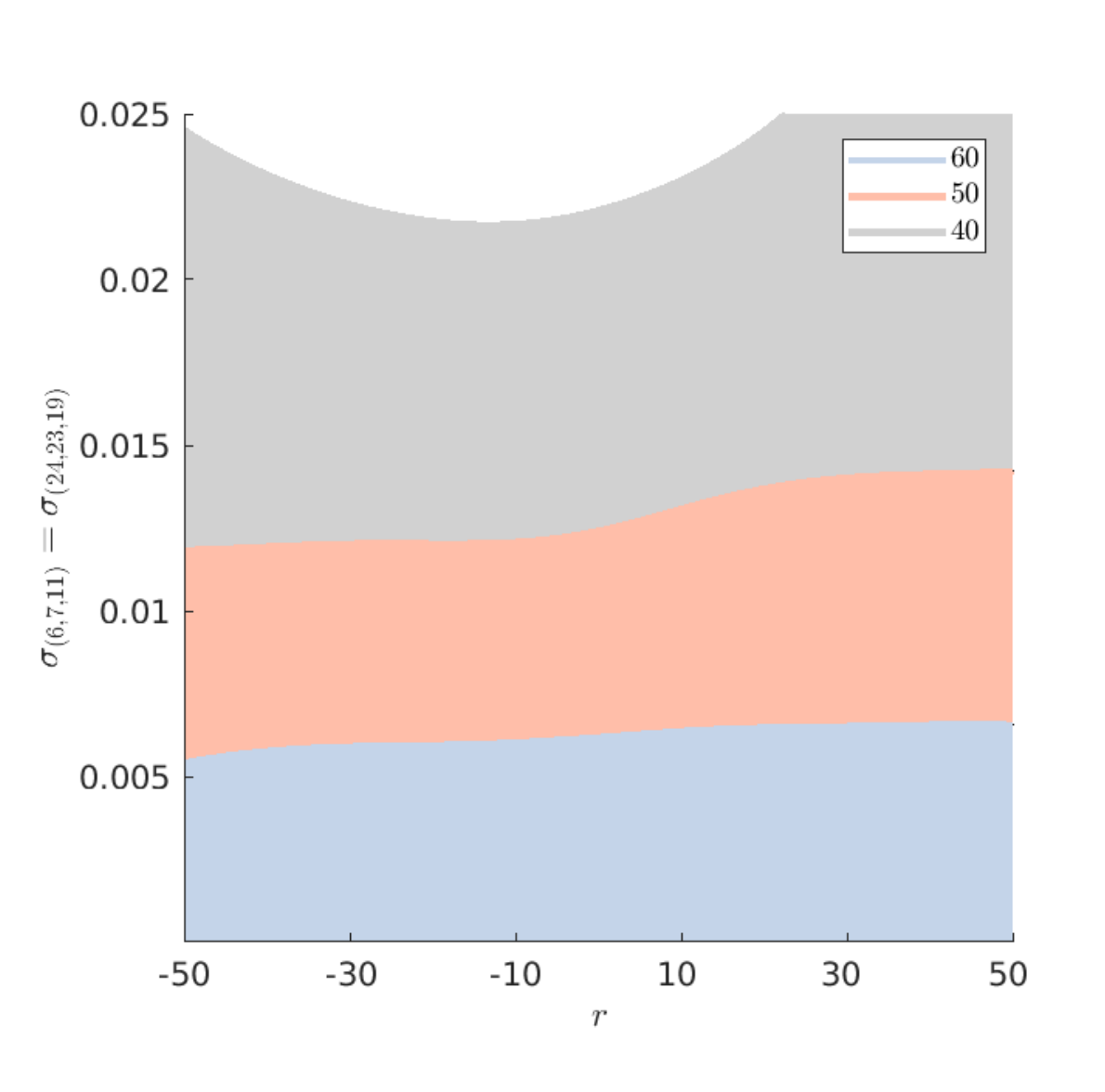}

\subcaption{$u(x)=x$}

\end{minipage}

\begin{minipage}{0.3\textwidth}
\centering

\includegraphics[scale=0.35,trim={4.5cm 8cm 5cm 8cm},clip]{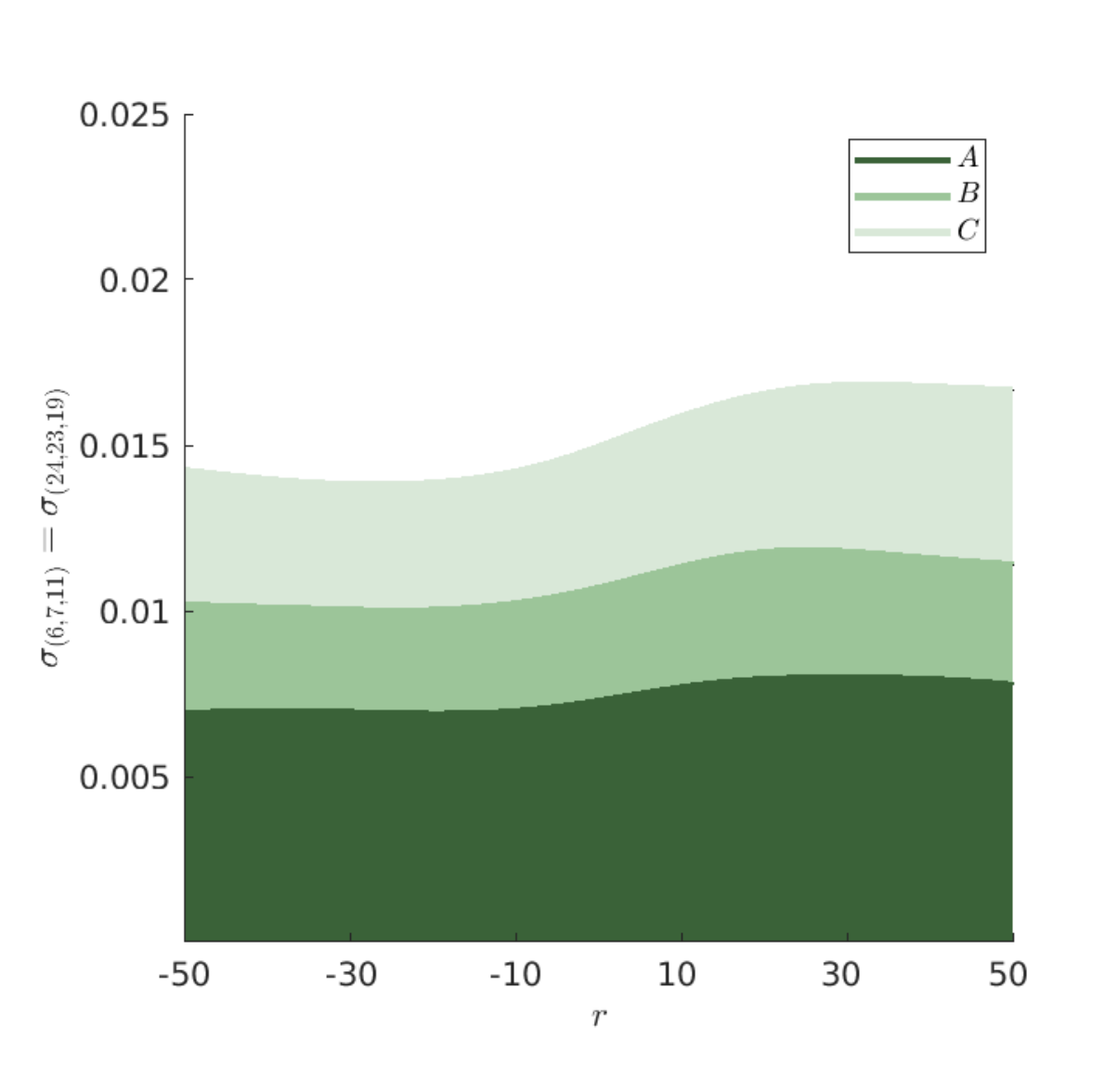}

\subcaption{$u(x)=0.1(x-60)_+ - 0.9(x-60)_-$}

\end{minipage}\hspace{0.5cm}\begin{minipage}{0.3\textwidth}
\centering

\includegraphics[scale=0.35,trim={4.5cm 8cm 5cm 8cm},clip]{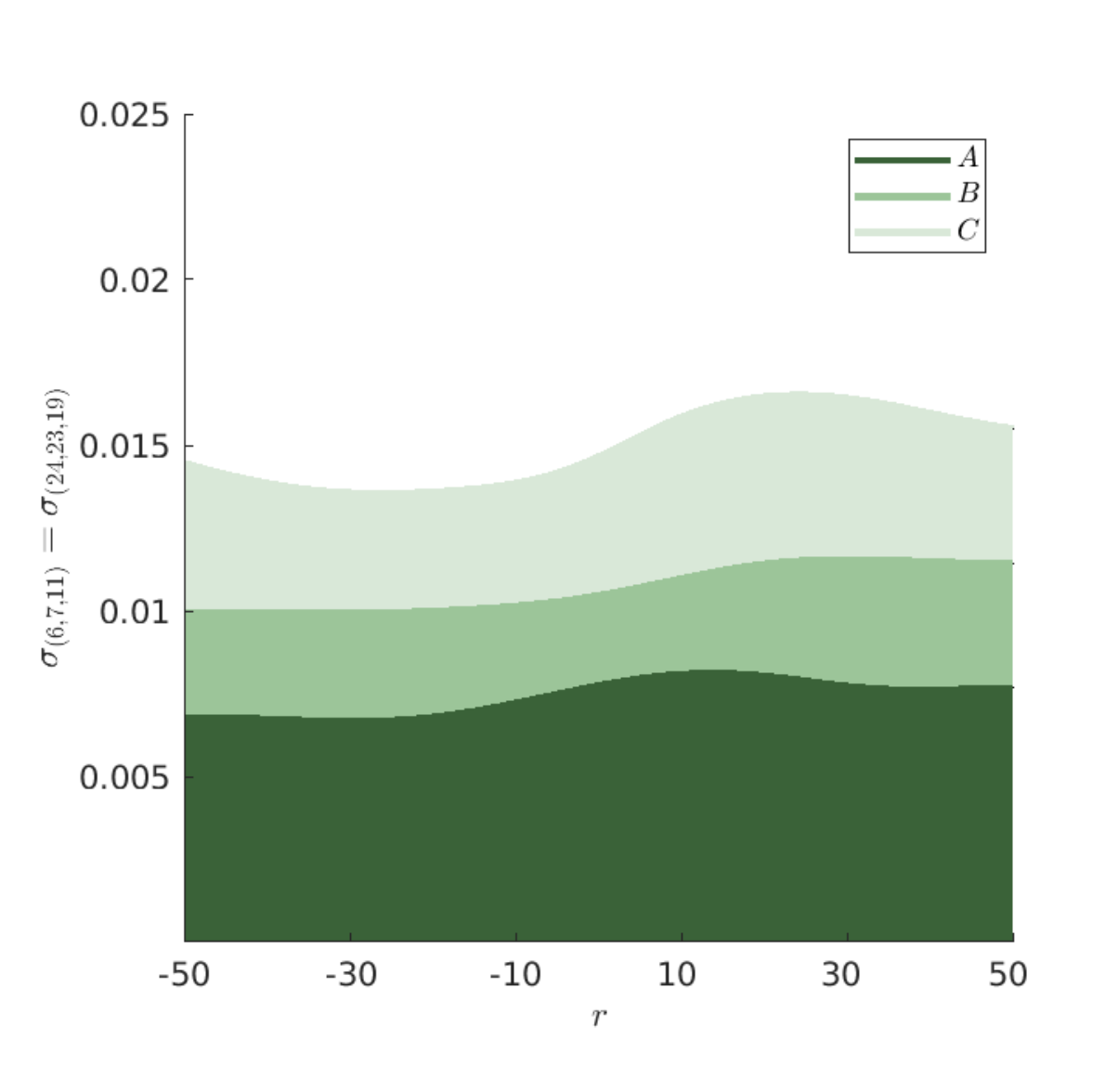}

\subcaption{$u(x)=0.2(x-60)_+ - 0.8(x-60)_-$}

\end{minipage}\hspace{0.5cm}\begin{minipage}{0.3\textwidth}
\centering

\includegraphics[scale=0.35,trim={4.5cm 8cm 5cm 8cm},clip]{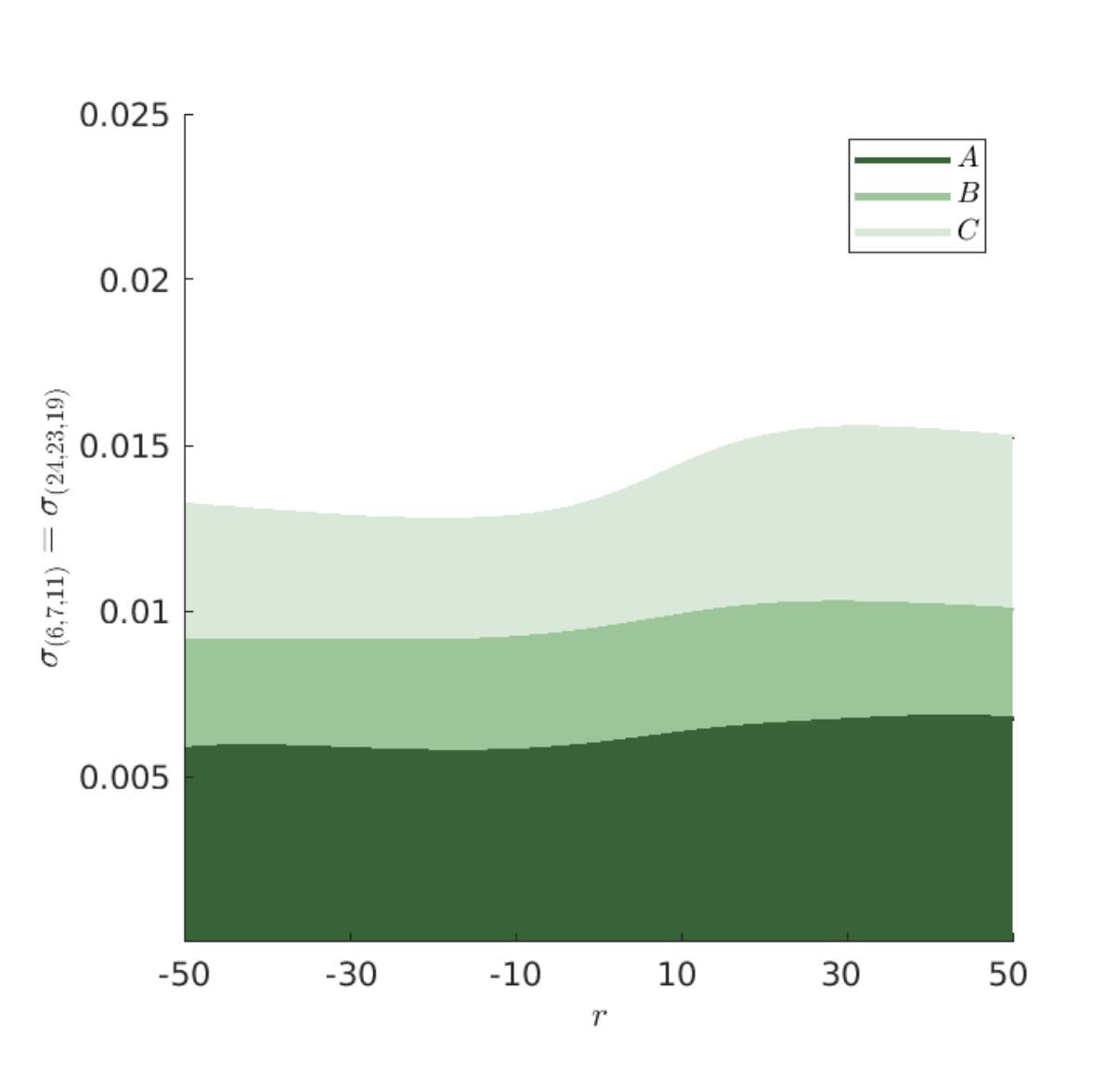}

\subcaption{$u(x)=\sqrt{x}$}

\end{minipage}

\begin{minipage}{0.3\textwidth}
\centering

\includegraphics[scale=0.35,trim={4.5cm 8cm 5cm 8cm},clip]{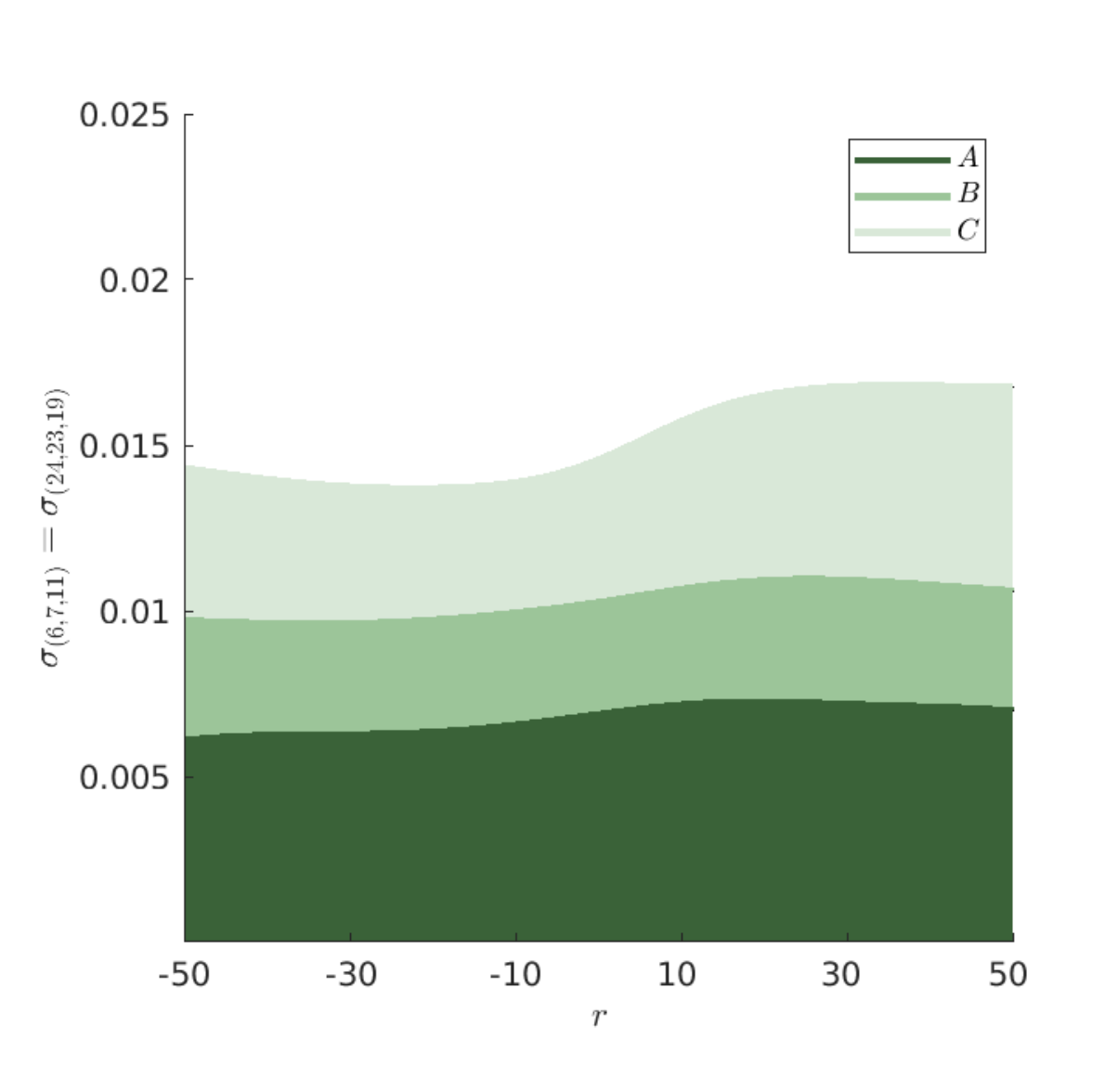}

\subcaption{$u(x)=-|x-2\cdot 60|^3\mathbbm{1}\{x\leq 2\cdot 60\}$}

\end{minipage}\hspace{0.5cm}\begin{minipage}{0.3\textwidth}
\centering

\includegraphics[scale=0.35,trim={4.5cm 8cm 5cm 8cm},clip]{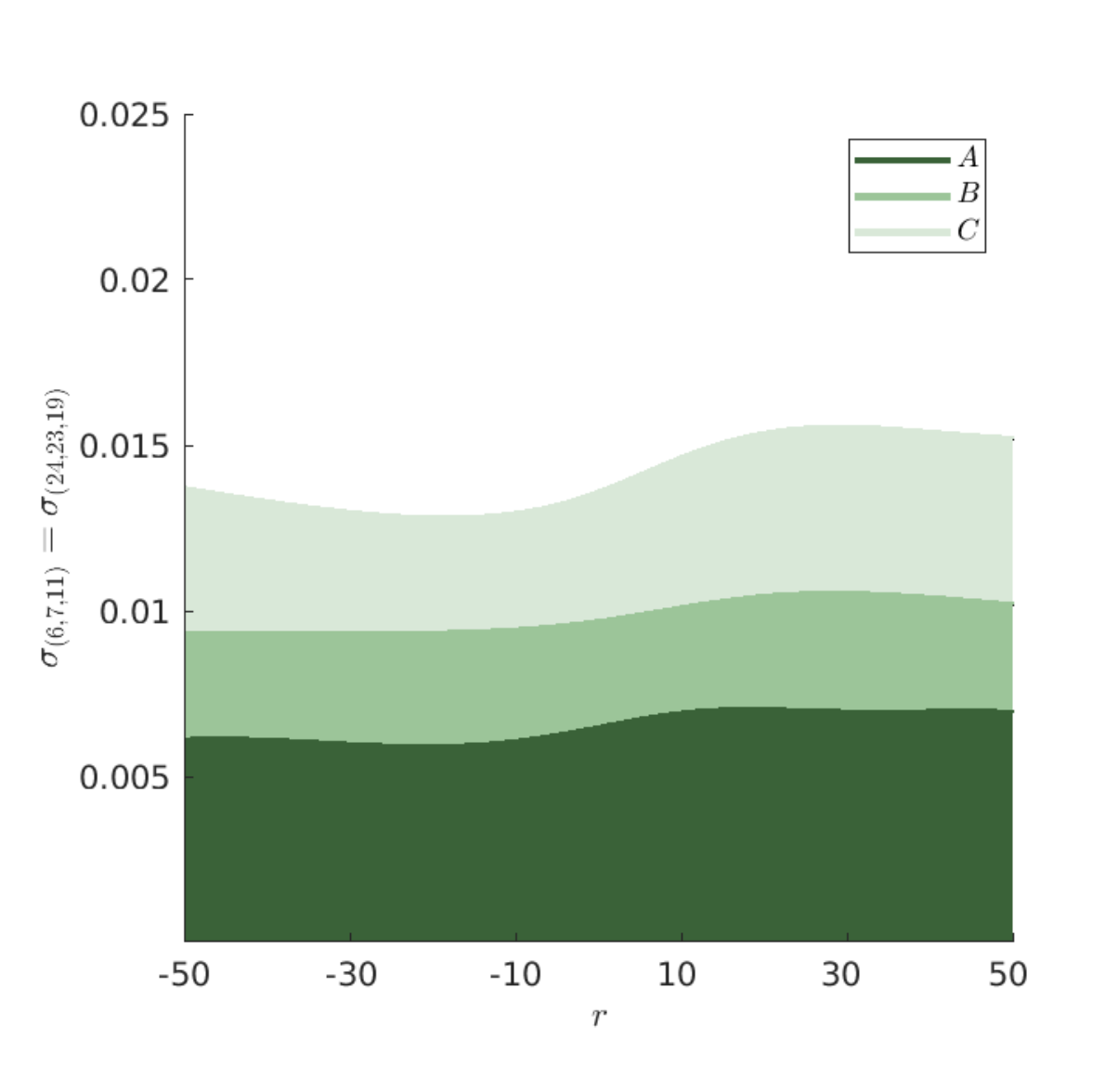}

\subcaption{$u(x)=-|x-2\cdot 60|^2\mathbbm{1}\{x\leq 2\cdot 60\}$}

\end{minipage}\hspace{0.5cm}\begin{minipage}{0.3\textwidth}
\centering

\includegraphics[scale=0.35,trim={4.5cm 8cm 5cm 8cm},clip]{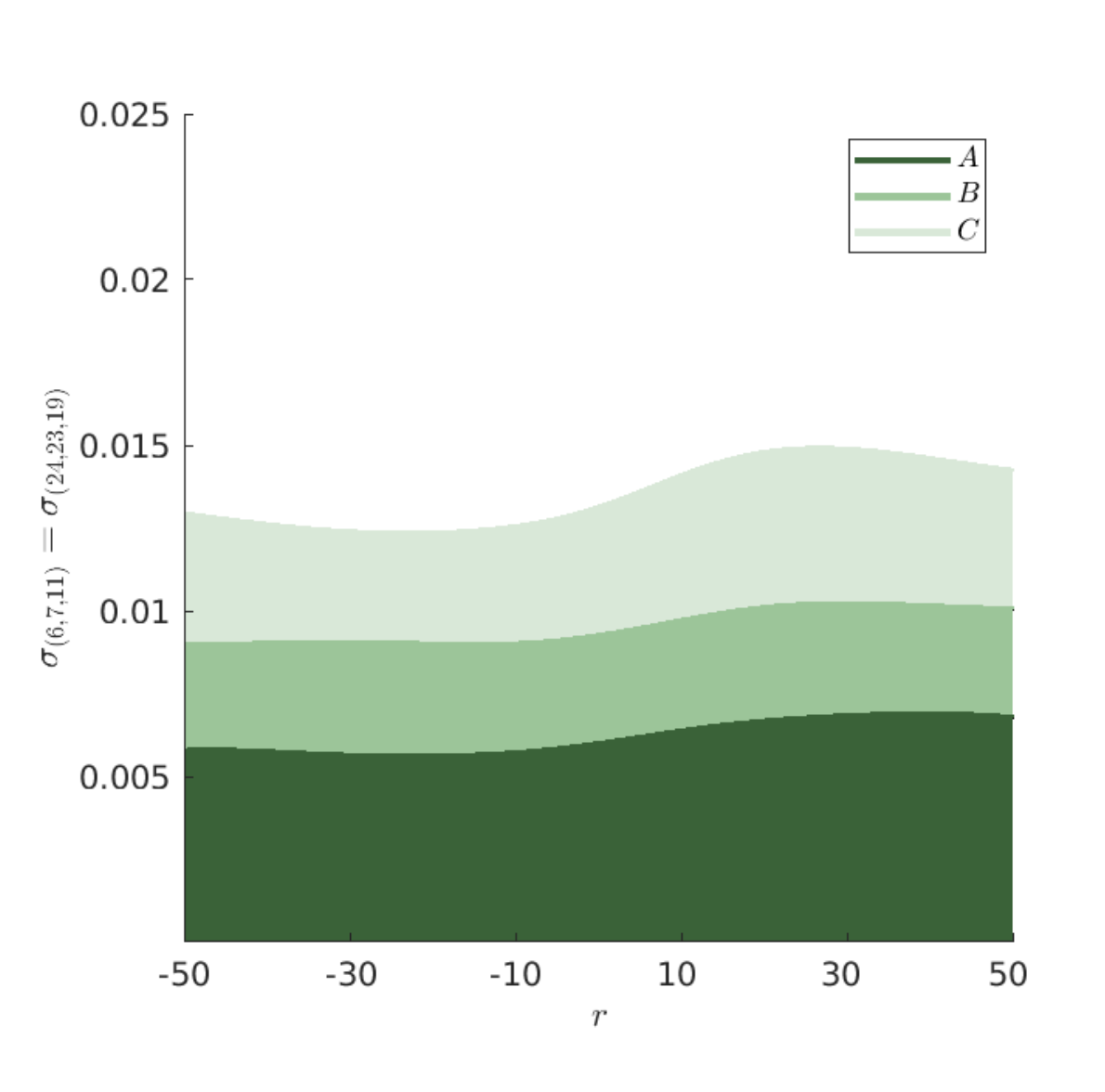}

\subcaption{$u(x)=-|x-2\cdot 60|^{3/2}\mathbbm{1}\{x\leq 2\cdot 60\}$}

\end{minipage}

\caption{Acceptable dependence and noise. GRP is based on the Matérn kernel.}
\label{fig:network1AcceptableDependence}
\end{figure}

\begin{figure}[!htbp]

\begin{minipage}{\textwidth}
\centering

\includegraphics[scale=0.4,trim={4.5cm 8cm 5cm 8cm},clip]{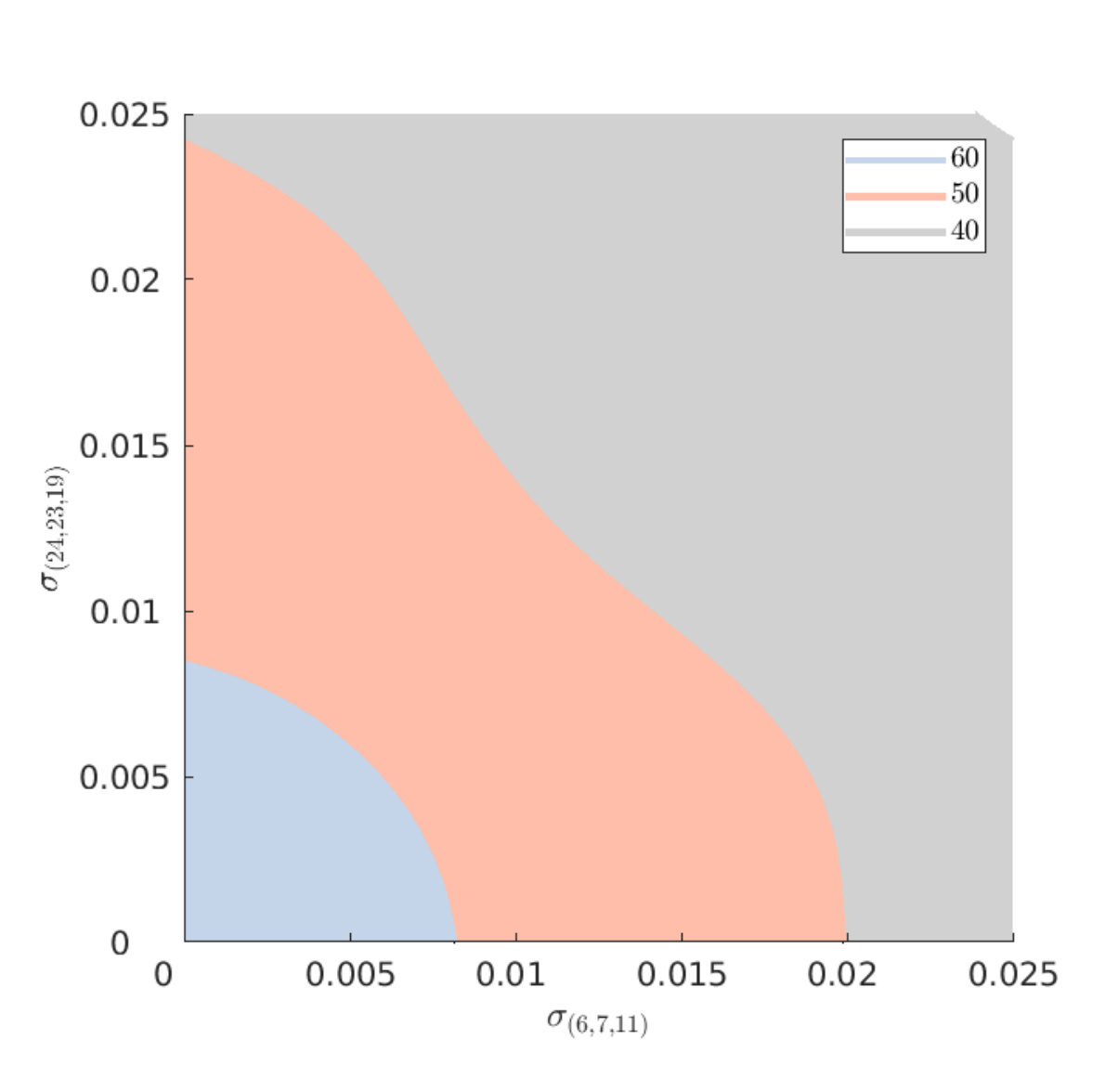}

\subcaption{$u(x)=x$}

\end{minipage}

\begin{minipage}{0.3\textwidth}
\centering

\includegraphics[scale=0.35,trim={4.5cm 8cm 5cm 8cm},clip]{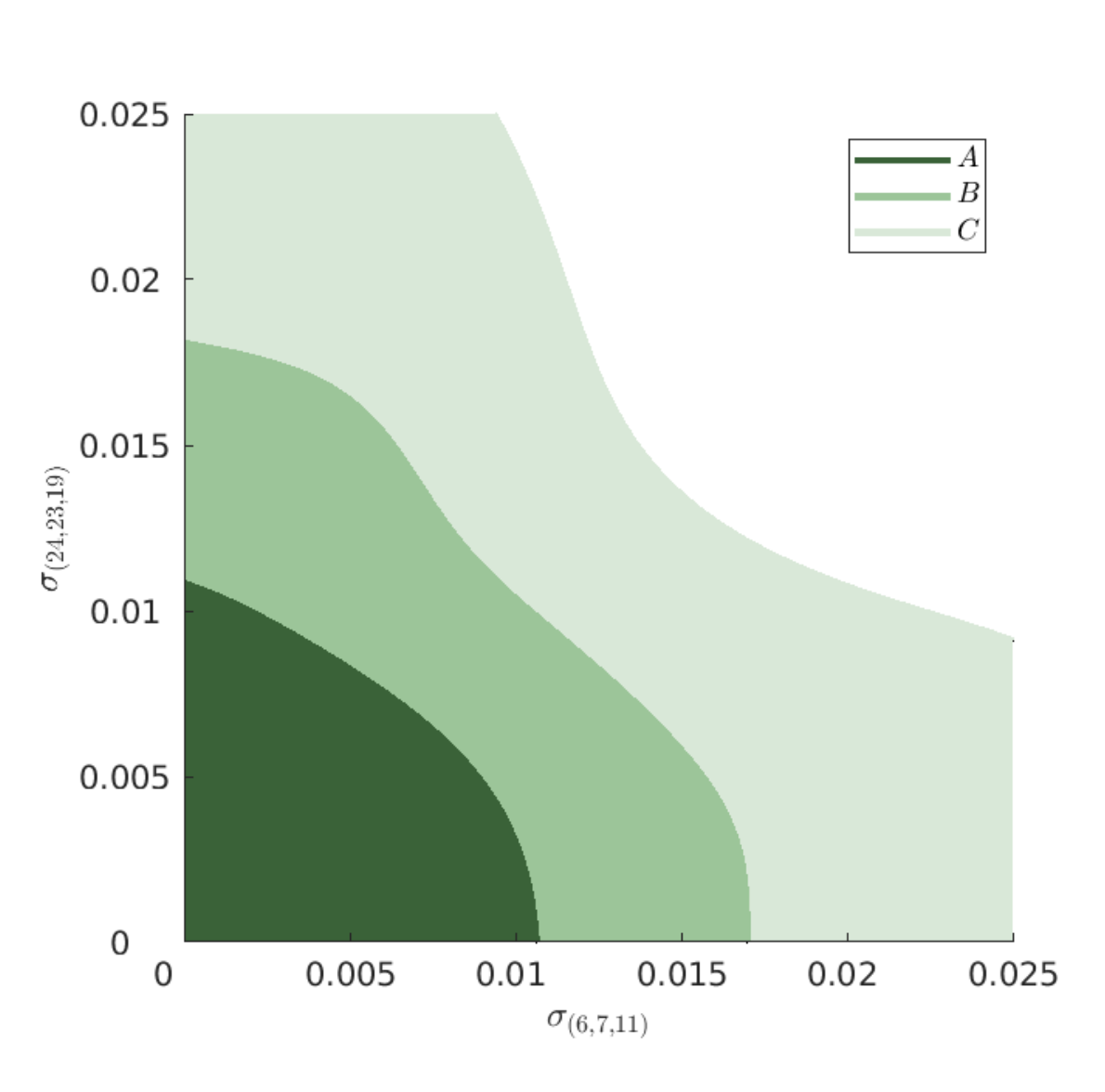}

\subcaption{$u(x)=0.1(x-60)_+ - 0.9(x-60)_-$}

\end{minipage}\hspace{0.5cm}\begin{minipage}{0.3\textwidth}
\centering

\includegraphics[scale=0.35,trim={4.5cm 8cm 5cm 8cm},clip]{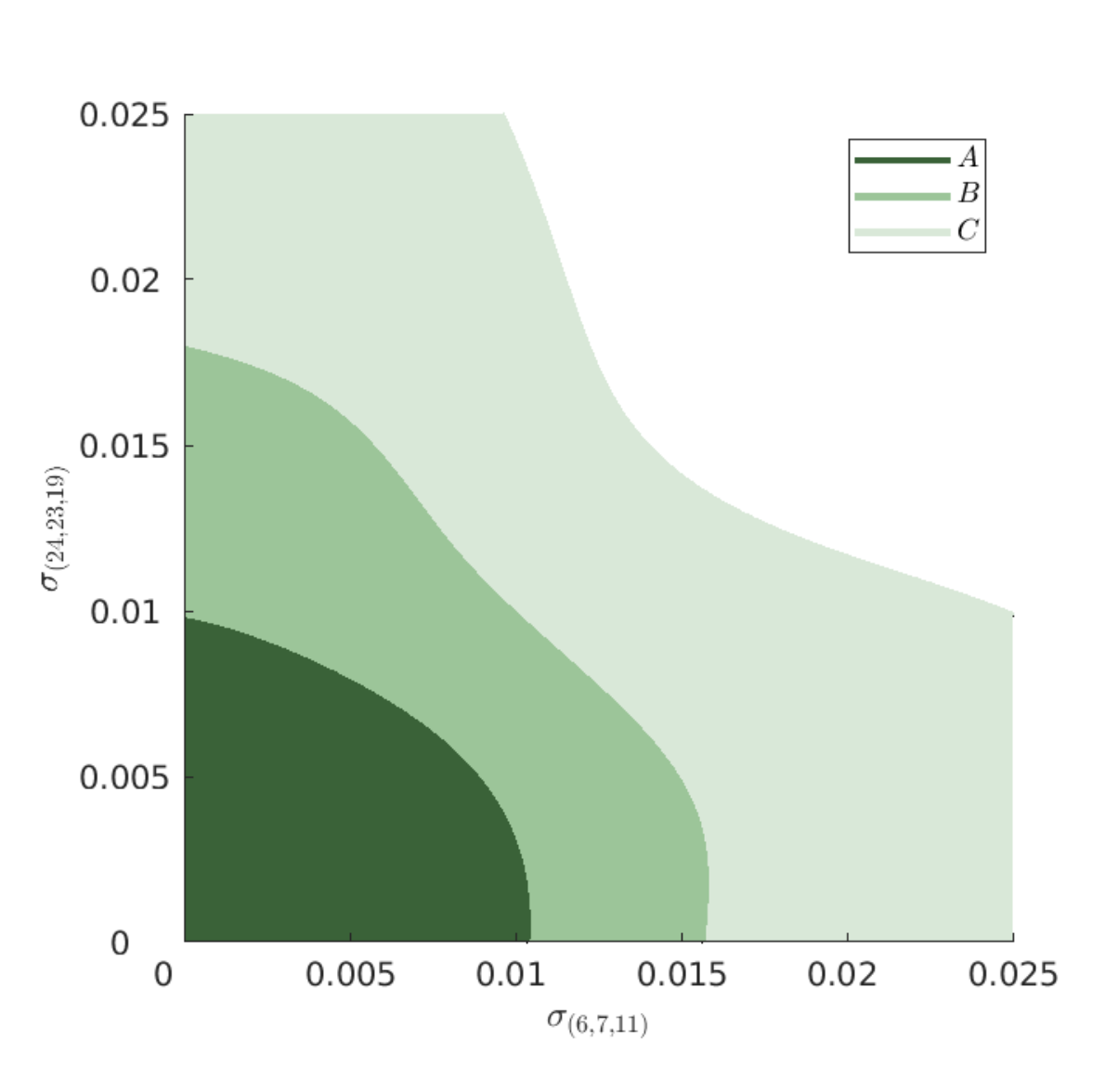}

\subcaption{$u(x)=0.2(x-60)_+ - 0.8(x-60)_-$}

\end{minipage}\hspace{0.5cm}\begin{minipage}{0.3\textwidth}
\centering

\includegraphics[scale=0.35,trim={4.5cm 8cm 5cm 8cm},clip]{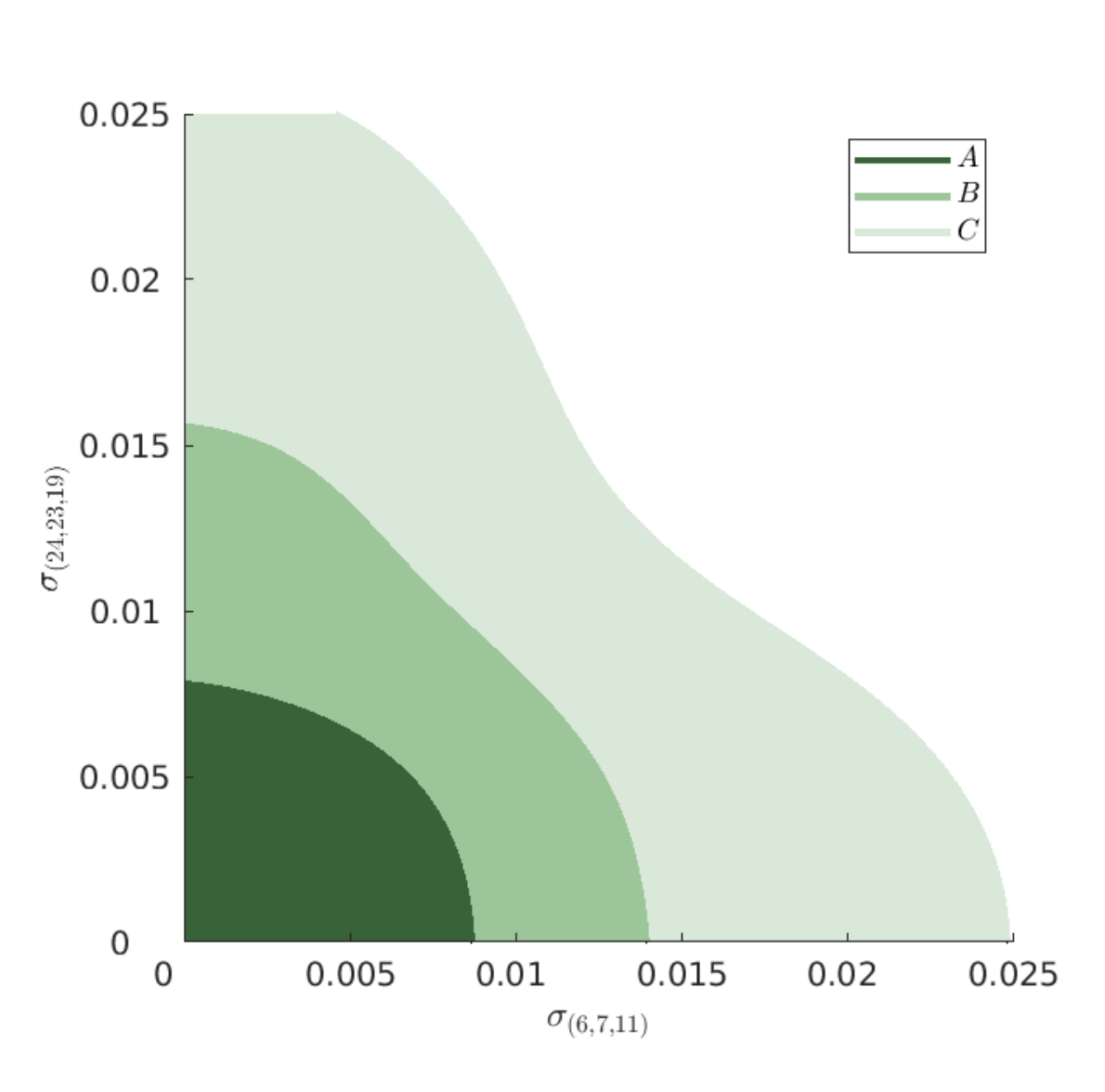}

\subcaption{$u(x)=\sqrt{x}$}

\end{minipage}

\begin{minipage}{0.3\textwidth}
\centering

\includegraphics[scale=0.35,trim={4.5cm 8cm 5cm 8cm},clip]{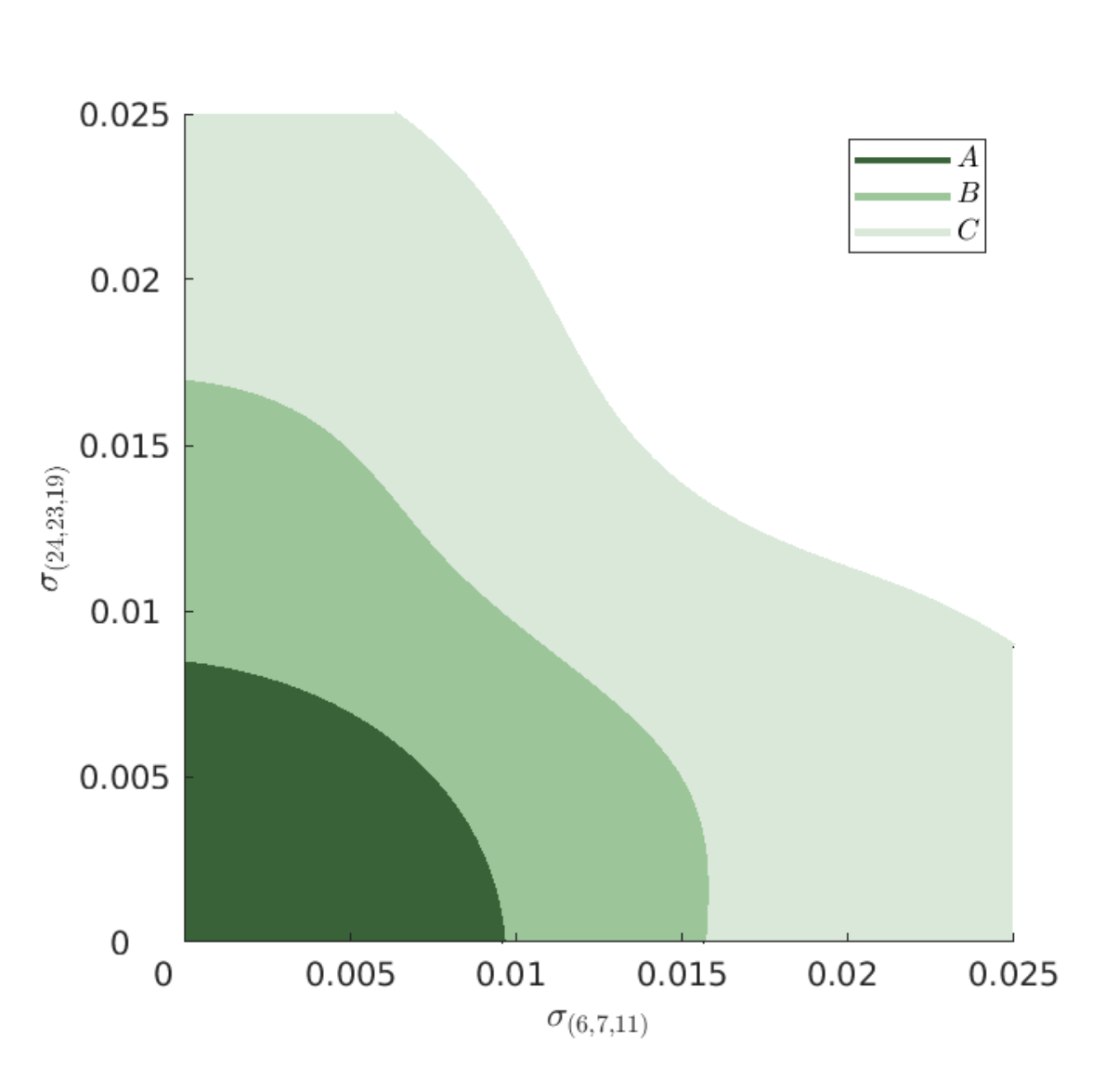}

\subcaption{$u(x)=-|x-2\cdot 60|^3\mathbbm{1}\{x\leq 2\cdot 60\}$}

\end{minipage}\hspace{0.5cm}\begin{minipage}{0.3\textwidth}
\centering

\includegraphics[scale=0.35,trim={4.5cm 8cm 5cm 8cm},clip]{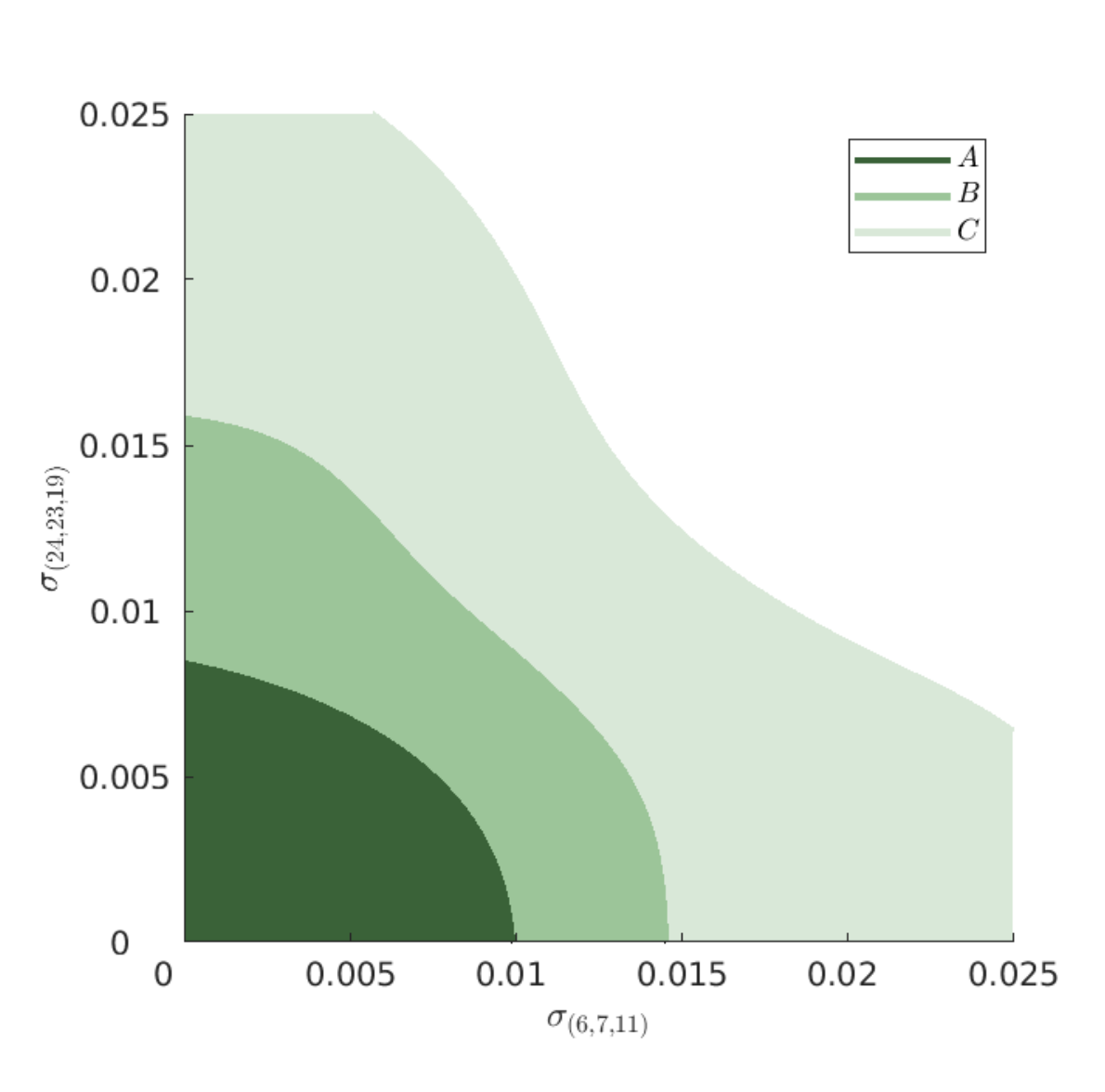}

\subcaption{$u(x)=-|x-2\cdot 60|^2\mathbbm{1}\{x\leq 2\cdot 60\}$}

\end{minipage}\hspace{0.5cm}\begin{minipage}{0.3\textwidth}
\centering

\includegraphics[scale=0.35,trim={4.5cm 8cm 5cm 8cm},clip]{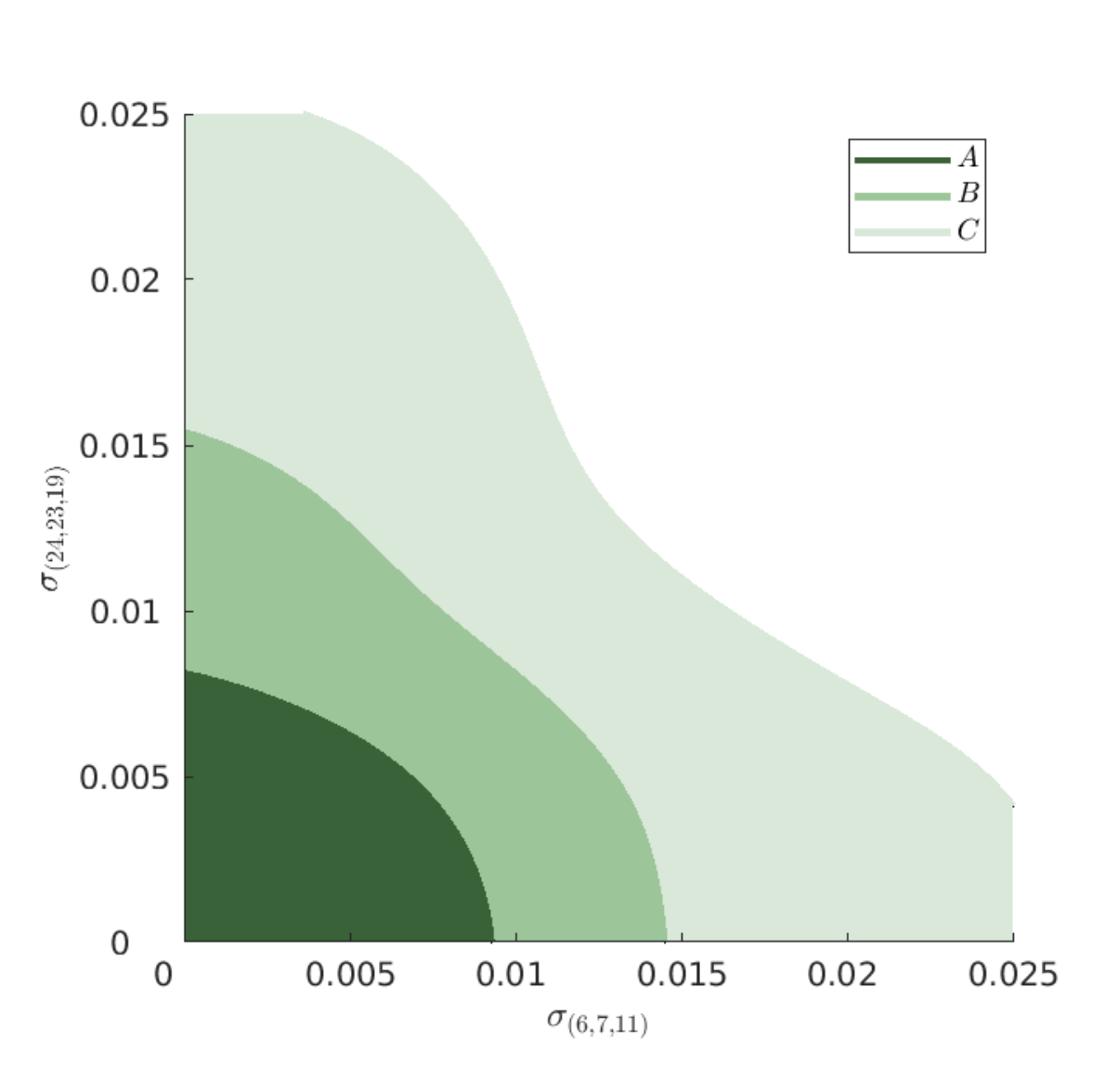}

\subcaption{$u(x)=-|x-2\cdot 60|^{3/2}\mathbbm{1}\{x\leq 2\cdot 60\}$}

\end{minipage}

\caption{Acceptable noise for $r=0$. GPR is based on the Matérn kernel.}
\label{fig:network1AcceptableNoise}
\end{figure}

\subsubsection{Companion to Section~\ref{sec:casestudyI}: Comparison of Squared Exponential and Matérn Kernel} 

Figure~\ref{fig:compKernels} compares the squared exponential and Matérn kernel in the situation of Figure~\ref{fig:network1AcceptableLightsMATERN}. The Matérn kernel can more flexibly adapt to functions that require higher curvature at some points. The surface plots have smaller amplitudes of the fluctuations than the squared exponential kernel.

\begin{figure}[!htbp]

\begin{minipage}[t]{0.5\textwidth}
\centering

\includegraphics[scale=0.35,trim={3.5cm 7cm 4cm 7cm},clip]{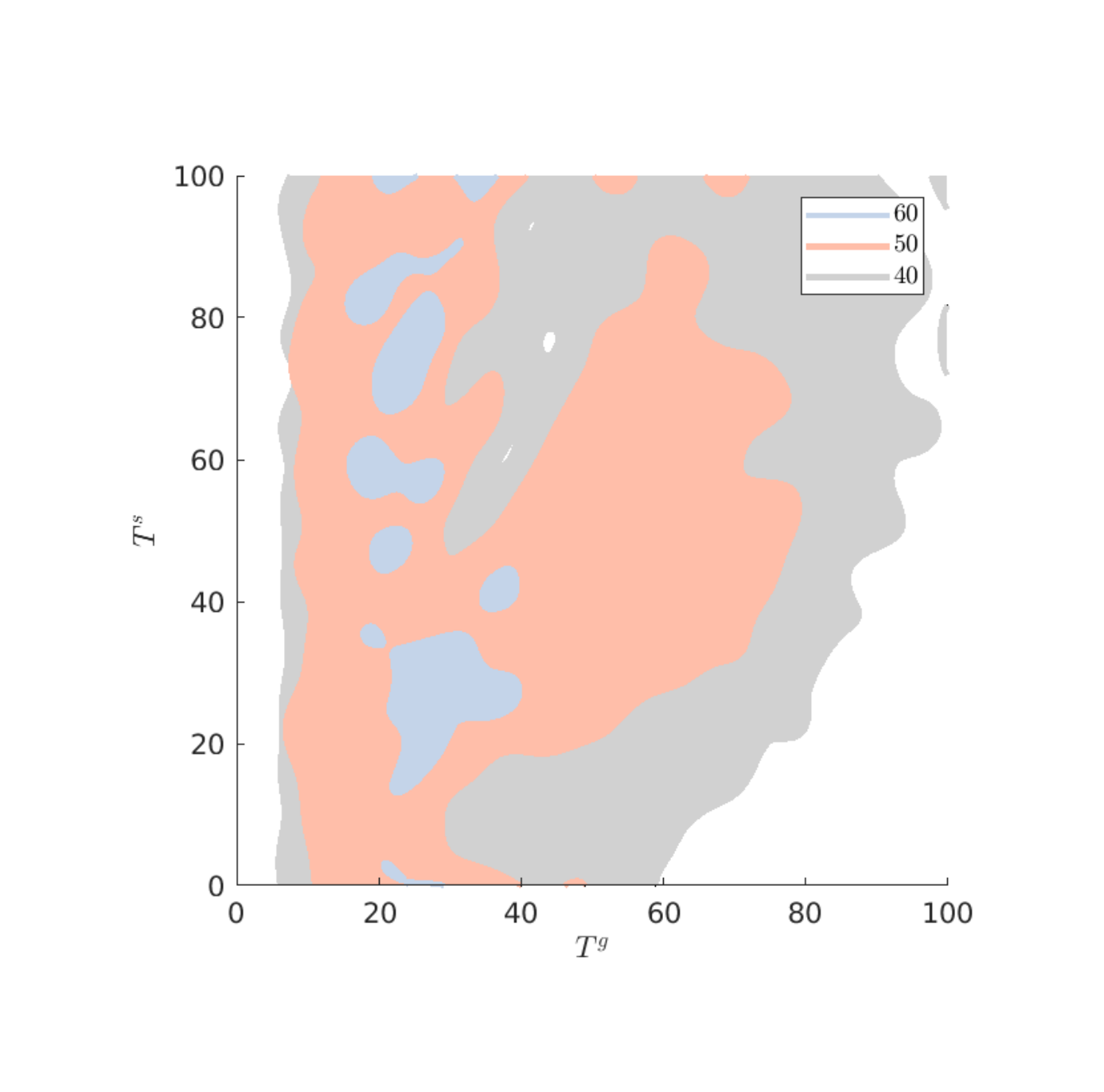}

\subcaption{Level sets: Squared exponential kernel}

\end{minipage}\begin{minipage}[t]{0.5\textwidth}
\centering

\includegraphics[scale=0.35,trim={3.5cm 7cm 3.5cm 7cm},clip]{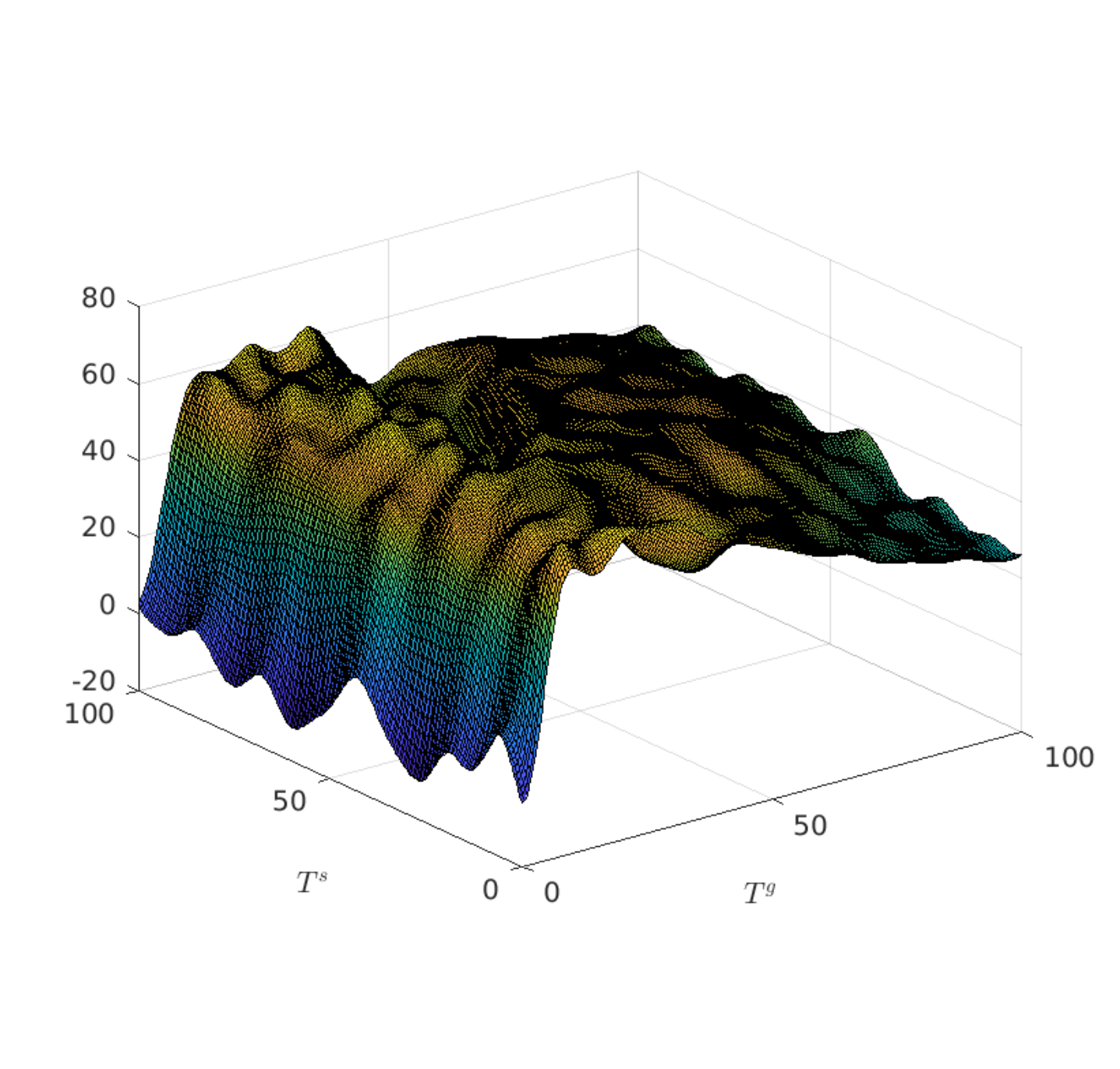}

\subcaption{Surface plot: Squared exponential kernel ($\gamma=60$)}

\end{minipage}

\begin{minipage}[t]{0.5\textwidth}
\centering

\includegraphics[scale=0.35,trim={3.5cm 7cm 4cm 7cm},clip]{Figures/CaseStudy1/acceptableLightsEMATERN1_NEW.pdf}

\subcaption{Level sets: Matérn kernel}

\end{minipage}\begin{minipage}[t]{0.5\textwidth}
\centering

\includegraphics[scale=0.35,trim={3.5cm 7cm 3.5cm 7cm},clip]{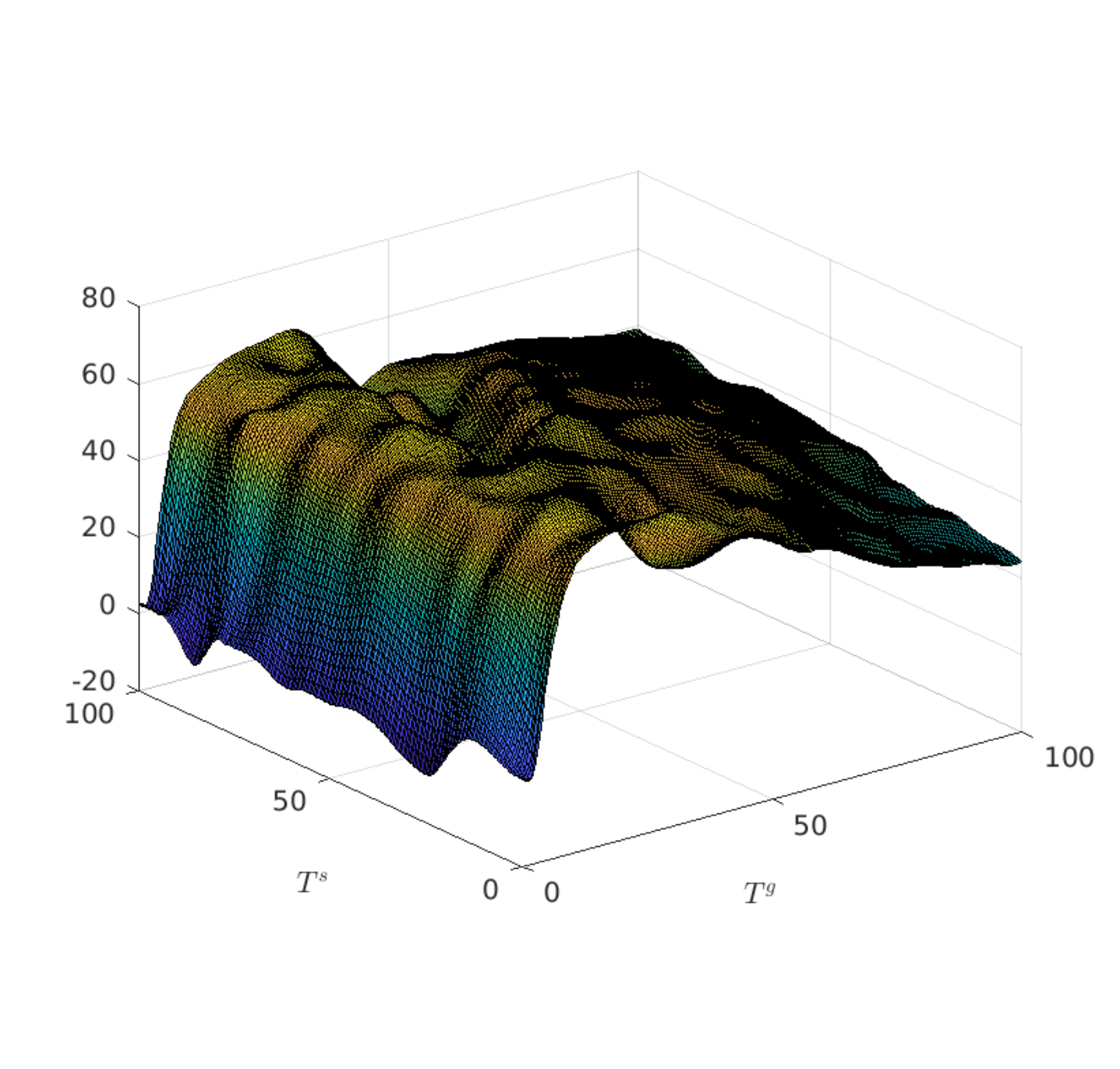}

\subcaption{Surface plot: Matérn kernel ($\gamma=60$)}

\end{minipage}

\caption{Comparison of kernels.}
\label{fig:compKernels}
\end{figure}

\end{document}